\documentclass[11pt]{article}

\usepackage[margin=1.1in]{geometry}

\usepackage[english]{babel}
\usepackage{xspace}
\usepackage{algorithm,algorithmicx}
\usepackage{tcolorbox}
\usepackage{tikz}

\usepackage{framed}
\usepackage{algpseudocode}
\usepackage{amsthm,nccmath}
\usepackage{amsmath,bm}
\usepackage[normalem]{ulem}
\usepackage{bigints}
\usepackage{amssymb}
\usepackage{caption}
\usepackage{graphicx}
\usepackage{cite}
\usepackage{enumerate}
\usepackage{multirow}
\usepackage{hyperref}
\usepackage{booktabs}
\usepackage{hyperref}
\usepackage{url}
\usepackage{color, colortbl}
\usepackage{graphicx}
\usepackage{amsmath,amssymb, bm}
\usepackage{algorithm}
\usepackage{algpseudocode}
\usepackage{xcolor}
\usepackage{flushend}
\usepackage{subfigure}
\usepackage[none]{hyphenat}
\usepackage{enumitem}
\usepackage[bottom]{footmisc}
\usepackage{hyperref}
\hypersetup{
	colorlinks   = true, 
	urlcolor     = blue, 
	linkcolor    = blue, 
	citecolor    = blue   
}

\newcommand{\xio}{\bar{\xi}}

\newcommand{\Ebb}{\mathbb{E}}

\newcommand{\R}{\mathbb{R}}

\newcommand{\aname}{\textsf{FedDRO}}

\usepackage{amsmath,amsfonts,bm}

\def\eqref#1{equation~\ref{#1}}

\def\1{\bm{1}}

\DeclareMathAlphabet{\mathsfit}{\encodingdefault}{\sfdefault}{m}{sl}
\SetMathAlphabet{\mathsfit}{bold}{\encodingdefault}{\sfdefault}{bx}{n}

\usepackage{float}
\usepackage{mdframed,lipsum}
\usepackage{epstopdf}
\usepackage{footnote}
\makesavenoteenv{tabular}
\makesavenoteenv{table}
\usepackage{bbm}
\usepackage{amsthm}
\usepackage[makeroom]{cancel}
\usepackage{stmaryrd}
\usepackage{amsmath,bm}
\usepackage{amssymb}
\usepackage{mathtools, cuted}
\usepackage{kantlipsum,setspace}

\theoremstyle{plain}
\newtheorem{theorem}{Theorem}[section]
\newtheorem{lem}[theorem]{Lemma}
\newtheorem*{lem*}{Lemma}

\newtheorem{cor}{Corollary}
\newtheorem*{cor*}{Corollary}

\theoremstyle{definition}
\newtheorem{defn}{Definition}[section]
\newtheorem*{defn*}{Definition}
\newtheorem{assump}{Assumption}

\theoremstyle{remark}
\newtheorem{rem}{Remark}
\newtheorem*{rem*}{Remark}

\title{{\bf \aname: Federated Compositional Optimization for Distributionally  Robust Learning}}
\date{} 
\author{\large Prashant Khanduri$^{\dagger, \ast}$, Chengyin Li$^\dagger$, Rafi Ibn Sultan$^\dagger$, Yao Qiang$^\dagger$, \\
	\large Joerg Kliewer$^\ddagger$, and Dongxiao Zhu$^\dagger$  \\[.5cm]
	\small $^{\dagger}$Department  of Computer Science, \\
	\small Wayne State University, MI, USA\\
	\small $^{\ddagger}$Department of Electrical and Computer Engineering,\\
	\small New Jersey Institute of Technology, NJ, USA\\
\small $^\ast$Email:  \texttt{khanduri.prashant@wayne.edu}} 

\begin{document}

\maketitle

\begin{abstract}
Recently, compositional optimization (CO) has gained popularity because of its applications in distributionally robust optimization (DRO) and many other machine learning problems. Large-scale and distributed availability of data demands the development of efficient federated learning (FL) algorithms for solving CO problems. Developing FL algorithms for CO is particularly challenging because of the compositional nature of the objective. Moreover, current state-of-the-art methods to solve such problems rely on large batch gradients (depending on the solution accuracy) not feasible for most practical settings. To address these challenges, in this work, we propose efficient FedAvg-type algorithms for solving non-convex CO in the FL setting. We first establish that vanilla FedAvg is not suitable to solve distributed CO problems because of the data heterogeneity in the compositional objective at each client which leads to the amplification of bias in the local compositional gradient estimates. To this end, we propose a novel FL framework \aname~that utilizes the DRO problem structure to design a communication strategy that allows FedAvg to control the bias in the estimation of the compositional gradient. A key novelty of our work is to develop solution accuracy-independent algorithms that do not require large batch gradients (and function evaluations) for solving federated CO problems. We establish $\mathcal{O}(\epsilon^{-2})$ sample and $\mathcal{O}(\epsilon^{-3/2})$ communication complexity in the FL setting while achieving linear speedup with the number of clients. We corroborate our theoretical findings with empirical studies on large-scale DRO problems.
\end{abstract}

\section{Introduction}
\label{sec: Intro}
Compositional optimization (CO) problems deal with the minimization of the composition of functions. A standard CO problem takes the form 
\begin{align}
\label{Eq: Basic_CompositeOpt}
  \min_{x \in \R^d} f(g(x))~~\text{with}~~g(x) \coloneqq \Ebb_{\zeta \sim \mathcal{D}_g}[ g(x ; \zeta)],
\end{align}
where $x \in  \R^d$ is the optimization variable, $f: \R^{d_g} \to \R$ and $g : \R^d \to \R^{d_g}$ are smooth functions, and $\zeta \sim \mathcal{D}_g$ represents a stochastic sample of $g(\cdot)$  from distribution $\mathcal{D}_g$. CO finds applications in a broad range of machine learning applications, including but not limited to distributionally robust optimization (DRO) \cite{qi2022stochastic}, meta-learning \cite{finn2017model}, phase retrieval \cite{duchi2019solving}, portfolio optimization \cite{shapiro2021lectures}, and reinforcement learning \cite{wang2017stochastic}. 

In this work, we focus on a more challenging version of the CO problem (\ref{Eq: Basic_CompositeOpt}) that often arises in the DRO formulation \cite{haddadpour2022learning}. Specifically, the problems that jointly minimize the summation of a compositional and a non-compositional objective. 
DRO has recently garnered significant attention from the research community because of its capability of handling noisy labels \cite{chen2022learning}, training fair machine learning models \cite{qi2022stochastic}, imbalanced \cite{qi2020attentional} and adversarial data \cite{chen2018robust}. A standard approach to solve DRO is to utilize primal-dual algorithms \cite{nemirovski2009robust} that are inherently slow because of a large number of stochastic constraints. The CO formulation enables the development of faster (dual-free) primal-only DRO algorithms \cite{haddadpour2022learning}. The majority of existing works to solve CO problems consider a centralized setting wherein all the data samples are available on a single server. However, modern large-scale machine-learning applications are characterized by the distributed collection of data by multiple clients \cite{kairouz2021advances}. This necessitates the development of distributed algorithms to solve the DRO problem.

Federated learning (FL) is a distributed learning paradigm that allows clients to solve a joint problem in collaboration with a server while keeping the data of each client private \cite{Mcmahan_PMLR_2017}. The clients act as computing units where within each communication round, the clients perform multiple updates while the server orchestrates the parameter sharing among clients. 
  Numerous FL algorithms exist in the literature to tackle standard (non-compositional) optimization problems \cite{li2019convergence,li2020federated,praneeth2019scaffold, Sharma_Arxiv_2019,zhang2021fedpd, khanduri2021stem,karimireddy2020mime}. However, there is a lack of efficient implementations when it comes to distributed CO problems. The major challenges in developing FL algorithms for solving the CO problem are:  \\{$\bf{[C1]}$:} The compositional structure of the problem leads to {\em biased} stochastic gradient estimates and this bias is amplified during local updates, which makes the theoretical analysis of the gradient-based algorithms intractable \cite{chen2021solving}. \\
{$\bf{[C2]}$:} Typically, data distribution at each client is different, referred to as data heterogeneity. Heterogeneously distributed compositional objective results in {\em client drift} during local updates that lead to divergence of federated CO algorithms. This is in sharp contrast to the standard FedAvg for non-CO objectives where client drift can be controlled during the local updates \cite{praneeth2019scaffold}.  
 \\
{$\bf{[C3]}$:} A majority of algorithms for solving CO rely on accuracy-dependent large batch gradients where the batch size depends on the desired solution accuracy, which is not practical from an implementation point of view \cite{huang2021compositional,haddadpour2022learning,guo2022fedx}. 

\noindent
These challenges naturally lead to the following question:
\begin{mdframed}
\centering
{\em Can we develop FL algorithms that tackle $[\mathbf{C1}]-[\mathbf{C3}]$ to solve CO in a distributed setting?}
\end{mdframed}
In this work, we address the above question and develop a novel FL algorithm to solve typical versions of the CO problem that arise in DRO (Section \ref{sec: Problem}). The major contributions of our work are:
\begin{itemize}[leftmargin=*]
\item We for the first time present a negative result that establishes that the vanilla FedAvg (customized to CO) is {\bf incapable of solving} the CO problems as it leads to bias amplification during the local updates. This shows that additional communication/processing is required by FedAvg to mitigate the bias in the local gradient estimation. 
    \item We develop \aname, a novel FL algorithm for solving problems with both {\bf compositional and non compositional non-convex objectives} at the same time. To the best of our knowledge, such an algorithm has been absent from the open literature so far. Importantly, \aname~addresses the above-mentioned challenges by developing several key innovations in the algorithm design.
    \begin{itemize}[leftmargin=*]
        \item \aname~addresses ${\bf [C1]}$ by designing a {\bf communication strategy} that utilizes the specific problem structure resulting from the DRO formulation and allows us to control the gradient bias. Specifically, \aname~utilizes the fact that the compositional functions $g(\cdot)$ are often {\bf low-dimensional embeddings} in the DRO formulation (see Examples in Section \ref{Subsec: Examples}) and can be shared without incurring significant communication costs.  
\item To address $\bf{[C2]}$, we {\bf design the local updates} at each client so that the client drift is bounded. Our analysis captures the effect of data heterogeneity on the performance of \aname.
\item[--] To address $\bf{[C3]}$, we utilize a {\bf hybrid momentum-based estimator} to learn the compositional embedding and combine it with a stochastic gradient (SG) estimator to conduct the local updates. This construction allows us to circumvent the need to compute large accuracy-dependent batch sizes for computing the gradients and the compositional function evaluations.
    \end{itemize}
\item We establish the {\bf convergence} of \aname~and show that to achieve an $\epsilon$-stationary point, \aname~requires $\mathcal{O}(\epsilon^{-2})$ samples while achieving {\bf linear speed-up} with the number of clients, i.e., requiring $\mathcal{O}(K^{-1}\epsilon^{-2})$ samples per client. Moreover, \aname~requires sharing of $\mathcal{O}(\epsilon^{-3/2})$ high-dimensional parameters and $\mathcal{O}(K^{-1}\epsilon^{-2})$ low dimensional embeddings per client.   
\end{itemize}

\paragraph{Notations.} The expected value of a random variable (r.v) $X$ is denoted by $\Ebb[X]$. Conditioned on an event $\mathcal{F}$ the expectation of a r.v $X$ is denoted by $\Ebb[X | \mathcal{F}]$. We denote by $\R$ (resp. $\R^d$) the real line (resp. the $d$ dimensional Euclidean space). We denote by $[K]  \coloneqq \{ 1,\ldots K\}$. The notation $\| \cdot \|$ defines a standard $\ell_2$-norm. For a set $B$, $|B|$ denotes the cardinality of $B$. We use $\xi \sim \mathcal{D}_h$ and $\zeta \sim \mathcal{D}_g$ to denote the stochastic samples of functions $h(\cdot)$ and $g(\cdot)$ from distributions $\mathcal{D}_h$ and $\mathcal{D}_g$, respectively. A batch of samples from $h(\cdot)$ (resp. $g(\cdot)$) is denoted by $b_h$ (resp. $b_g$). Moreover, joint samples of $h(\cdot)$ and $g(\cdot)$ are denoted by $\bar{\xi} = \{ b_h, b_g\}$. We represent by $\bar{x}$ the empirical average of a sequence of vectors $\{ x_k\}_{k = 1}^K$.

\section{Problem}
\label{sec: Problem}
In this work, we focus on a general version of the CO problem defined in (\ref{Eq: Basic_CompositeOpt}). We consider the following problem that often arises in DRO (see Section \ref{Subsec: Examples}) in a distributed setting with $K$ clients 
\begin{align}
\label{Eq: Fl_Prob}
 \inf_{x \in \R^d}  \Big\{  \Phi(x)  \coloneqq        h(x) + f(g(x)) \Big\}~~\text{with $h(x)   \coloneqq  \frac{1}{K} \sum_{k = 1}^K  h_k(x)$ \& $g(x) \coloneqq \frac{1}{K}  \sum_{k = 1}^K   g_k(x)$},
\end{align}
where each client $k \in [K]$ has access to the local functions $h_k : \R^d \to \R$ and $g_k: \R^d \to \R^{d_g}$ while $f(\cdot)$ is same as (\ref{Eq: Basic_CompositeOpt}). The local functions $h_k(\cdot)$ and $g_k(\cdot)$ at 
each client $k \in [K]$ are:
$h_k(x) \! = \!\Ebb_{\xi_k \sim \mathcal{D}_{h_k}} [h_k(x ; \xi_{k})]$ and $g_k(x) \! = \! \Ebb_{\zeta_k \sim \mathcal{D}_{g_k}} [g_k(x ; \zeta_{k})]$
and where $\xi_{k} \sim \mathcal{D}_{h_k}$ (resp. $\zeta_{k} \sim \mathcal{D}_{g_k}$) represents a sample of $h_k(\cdot)$ (resp. $g_k(\cdot)$) from distribution $\mathcal{D}_{h_k}$ (resp. $\mathcal{D}_{g_k}$). Moreover, the data at each client is heterogeneous, i.e., $\mathcal{D}_{h_k} \neq \mathcal{D}_{h_\ell}$ and $\mathcal{D}_{g_k} \neq \mathcal{D}_{g_\ell}$ for $k \neq \ell$ and $k, \ell \in [K]$. 

In comparison to the basic CO in (\ref{Eq: Basic_CompositeOpt}), (\ref{Eq: Fl_Prob}) is significantly challenging, first, because of the presence of both compositional and non-compositional objectives and second, because of the distributed nature of the compositional function $g(\cdot)$. 

\begin{rem}[Comparison to \cite{huang2021compositional} and \cite{gao2022convergence}]
\label{Rem: Comparison_Huang}
   Note that formulation  (\ref{Eq: Fl_Prob}) is significantly different than the setting considered in \cite{huang2021compositional,gao2022convergence}. Specifically, our formulation considers a practical setting where the compositional functions are distributed across agents, i.e., the function is $g = {1}/{K}  \sum_{k = 1}^K \! g_k(x)$. In contrast,  \cite{huang2021compositional,gao2022convergence} consider a setting with objective $\frac{1}{k}\sum_{k = 1}^K f_k(g_k(\cdot))$, note here that the compositional function is local to each agent. This implies that algorithms developed in \cite{huang2021compositional,gao2022convergence} cannot solve problem (\ref{Eq: Fl_Prob}). Importantly, problem (\ref{Eq: Fl_Prob}) models realistic FL training settings while being more challenging compared to \cite{huang2021compositional,gao2022convergence} since in (\ref{Eq: Fl_Prob}) the data heterogeneity of the inner problem also plays a role in the convergence of the FL algorithm. Please see the discussion in Appendix \ref{App: Comparison} for more details. 
\end{rem}

\subsection{Examples: CO reformulation of DRO problems}
\label{Subsec: Examples}
In this section, we discuss different DRO formulations that can be efficiently solved using CO \cite{haddadpour2022learning}. DRO problem with a set of $m$ training samples denoted as $\{ \zeta_i\}_{i = 1}^m$ is  
\begin{align}
\label{Eq: GeneralDRO}
   \min_{x \in \R^d} \max_{ \mathbf{p} \in P_m }\sum_{i = 1}^m p_i \ell(x ; \zeta_i) - \lambda D_{\ast}(\mathbf{p}, \mathbf{1}/m)  
\end{align}
where $x \in \R^d$ is the model parameter, $P_m \coloneqq \{\mathbf{p} \in \R^m: \sum_{i = 1}^m p_i = 1, p_i \geq 0\}$ is $m$-dimensional simplex, $D_{\ast}(\mathbf{p}, \mathbf{1}/m)$ is a divergence metric that measures distance between $\mathbf{p}$ and uniform probability $\mathbf{1}/m \in \R^m$, and $\ell(x, \zeta_i)$ denotes the loss on sample $\zeta_i$, $\rho$ is a constraint parameter, and $\lambda$ is a hyperparameter. Next, we discuss two popular reformulations of (\ref{Eq: GeneralDRO}) in the form of CO problems. \vspace{2mm}\\     
\textbf{DRO with KL-Divergence.}
 Problem (\ref{Eq: GeneralDRO}) is referred to as a KL-regularized DRO when the distance metric $D_\ast (\mathbf{p}, \mathbf{1}/m)$ is the KL-Divergence, i.e., we have 
 $D_\ast (\mathbf{p}, \mathbf{1}/m) = D_{\text{KL}} (\mathbf{p}, \mathbf{1}/m)$ with 
 $D_{\text{KL}} (\mathbf{p}, \mathbf{1}/m) \coloneqq \sum_{i = 1}^m p_i \log (p_i m)$. For this case, an equivalent reformulation of (\ref{Eq: GeneralDRO}) is
 \begin{align}
 \label{eq: DRO_KL}
      \min_{x \in \R^d}   \log \Big( \frac{1}{m} \sum_{i = 1}^m \exp\Big( \frac{\ell(x; \zeta_i)}{\lambda} \Big) \Big),
 \end{align}
which is a CO with $g(x)  = 1/m \sum_{i = 1}^m \exp(\ell(x; \zeta_i)/ \lambda)$, $f(g(x)) = \log(g(x))$ and $h(x) = 0$.\vspace{2mm}\\ 
\textbf{DRO with \texorpdfstring{$\chi^2$} ~- Divergence.}
 Similar to KL-regularized DRO, (\ref{Eq: GeneralDRO}) is referred to as a $\chi^2$-regularized DRO when $D_\ast (\mathbf{p}, \mathbf{1}/m)$ is the $\chi^2$-Divergence, i.e., we have 
 $D_\ast (\mathbf{p}, \mathbf{1}/m) = D_{\chi^2} (\mathbf{p}, \mathbf{1}/m)$ with 
 $D_{\chi^2} (\mathbf{p}, \mathbf{1}/m) \coloneqq m/2\sum_{i = 1}^m (p_i - 1/m)^2$. For this case, an equivalent reformulation of (\ref{Eq: GeneralDRO}) is
\begin{align}
 \label{eq: DRO_Chi}
      \min_{x \in \R^d}   - \frac{1}{2 \lambda m} \sum_{i = 1}^m \big( \ell(x; \zeta_i) \big)^2 + \frac{1}{2 \lambda} \Big( \frac{1}{m} \sum_{i = 1}^m \ell(x; \zeta_i) \Big)^2
 \end{align}
which is again a CO with $g(x) = 1/m \sum_{i = 1}^m \ell(x ; \zeta_i )$, $f(g(x))= g(x)^2/2 \lambda$ and $h(x) = - \frac{1}{2 \lambda m} \sum_{i = 1}^m \big( \ell(x; \zeta_i) \big)^2$.

Note that both (\ref{eq: DRO_KL}) and (\ref{eq: DRO_Chi}) 
can be equivalently restated in the practical FL setting 
of (\ref{Eq: Fl_Prob}) if the overall samples are shared across multiple clients with each client having access to a subset of samples.

\paragraph{Related work.}
Please see Table \ref{tab:table1} for a comparison of current approaches to solve CO problems in distributed settings. For a detailed review of centralized and distributed non-convex CO and DRO problems, please see Appendix \ref{App: Related Work}. Here, we point out some drawbacks of the current approaches to solving federated CO problems:
\begin{itemize}[leftmargin=*]
    \item[--] None of the current works guarantee linear speedup with the number of clients \cite{huang2021compositional,haddadpour2022learning,tarzanagh2022fednest,gao2022convergence}.    
    \item[--] Utilize complicated multi-loop algorithms with momentum or VR-based updates \cite{tarzanagh2022fednest} that sometime require computation of large batch size gradients \cite{haddadpour2022learning} to guarantee convergence. Such algorithms are not preferred in practical implementations.
    \item[--] Consider a restricted setting where the compositional objective is not distributed among nodes \cite{huang2021compositional, gao2022convergence}. Importantly, the algorithms developed therein cannot solve the problem considered in our work (see Appendix \ref{App: Comparison}).
\end{itemize}
Our work addresses all these issues and develops, \aname, the first simple SGD-based FL algorithm to tackle CO problems with the distributed compositional objective. Please see Table \ref{tab:table1} for a comparison of the above works.

\begin{table}[t]
\scriptsize
  \begin{center}
    \caption{Comparison with the existing works. Here, CO-ND refers to CO with a non-distributed compositional part (see Remark \ref{Rem: Comparison_Huang}). CO + Non-CO refers to problems with both CO and Non-CO objectives. VR refers to variance reduction.  (I) and (O) refers to the inner and outer loop, respectively.\\
    $^\ast$Theoretical guarantees for GCIVR exist only for the finite sample setting with $m$ total network-wide samples. }
    \label{tab:table1}
\renewcommand{\arraystretch}{1}
    \resizebox{\textwidth}{!}{\begin{tabular}{|c|c|c|c|c|}
      \toprule 
    \bf ALGORITHM & \bf SETTING &\bf UPDATE & \bf BATCH-SIZES &\bf CONVERGENCE  \\ 
            \hline
             \hline
      {ComFedL \cite{huang2021compositional} } & 
 CO-ND   & 
      SGD & $\mathcal{O}(\epsilon^{-2})$ & $\mathcal{O}(\epsilon^{-4})$   \\ \hline  
      {Local-SCGDM \cite{gao2022convergence} } & 
       CO-ND &  Momentum SGD & $\mathcal{O}(1)$ &  $\mathcal{O}(\epsilon^{-2})$ \\ \hline      
      
     {FedNest \cite{tarzanagh2022fednest} } &  Bilevel & VR & $\mathcal{O}(1)$ & $\mathcal{O}(\epsilon^{-2})$ \\ \hline
      {GCIVR$^\ast$ \cite{haddadpour2022learning}} & 
      CO + Non-CO & VR &  $\sqrt{m}~ \text{(I)}, m ~\text{(O)}$
 &  $\mathcal{O}(\min \{\sqrt{m} \epsilon^{-1} ,\epsilon^{-1.5} \})$ \\ \hline
 \cellcolor{blue!10}   {\aname ~(Ours)} &
  \cellcolor{blue!10}   CO + Non-CO & \cellcolor{blue!10} SGD & \cellcolor{blue!10} $\mathcal{O}(1)$ & \cellcolor{blue!10} $\mathcal{O}(K^{-1} \epsilon^{-2})$\\ 
      \hline
      \hline  
      \end{tabular} }
  \end{center}
\end{table}
 
 \section{Preliminaries}
\label{Sec: Prelim}
In this section, we introduce the assumptions, definitions, and preliminary lemmas.  

\begin{defn}[Lipschitzness]
\label{Def: Lip}
    For all $x_1, x_2 \in \R^d$, a differentiable function $\Phi: \R^d \to \R$ is: Lipschitz smooth if $\| \nabla \Phi(x_1) - \nabla \Phi(x_2) \| \leq L_\Phi \|x_1 - x_2 \|$
   for some $L_\Phi > 0$; Lipschitz if $\| \Phi(x_1) -  \Phi(x_2) \| \leq B_\Phi \|x_1 - x_2 \|$ for some $B_\Phi > 0$ and; Mean-Squared Lipschitz if $\Ebb_{{\xi}}\|  \Phi(x_1; \xi) -  \Phi(x_2 ; \xi) \|^2 \leq B_\Phi^2 \|x_1 - x_2 \|^2$ for some $B_\Phi > 0$.
\end{defn} 
We make the following assumptions on the local and global functions in problem (\ref{Eq: Fl_Prob}). 
\begin{assump}[Lipschitzness]
\label{Ass: Lip}
The following holds 
\begin{enumerate}[leftmargin=*]
    \item The functions $f(\cdot)$, $h_k(\cdot)$, $g_k(\cdot)$ for all $k \in [K]$ are differentiable and Lipschitz-smooth with constants $L_f, L_h, L_g > 0$, respectively.
    \item The function $f(\cdot)$ is Lipschitz with constant $B_f > 0$  and $g_k(\cdot)$ is mean-squared Lipschitz for all $k \in [K]$ with constant $B_g > 0$.
\end{enumerate}
\end{assump}
Next, we introduce the variance and heterogeneity assumptions.

\begin{assump}[Unbiased Gradient and Bounded Variance]
\label{Ass: BoundedVar} The stochastic gradients and function evaluations of the local functions at each client are unbiased and have bounded variance, i.e.,
\begin{align*}
\begin{split}
   &  \mathbb{E}_{\xi_k} [\nabla h_k(x; \xi_k)]   = \nabla h_k(x) ,~
   \mathbb{E}_{\zeta_k } [\nabla g_k(x; \zeta_k)] = \nabla g_k(x), ~ 
  \mathbb{E}_{\zeta_k} [  g_k(x; \zeta_k)]  =  g_k(x), \\
  &   \qquad  \qquad  \qquad   \mathbb{E}_{\zeta_k } [\nabla g_k(x; \zeta_k) \nabla f(y)] = \nabla g_k(x) \nabla f(y)
  \end{split}
\end{align*}
\begin{align*}
\text{and} \qquad \qquad & \mathbb{E}_{\xi_k}\| \nabla h_k(x ; \xi_k) - \nabla h_k(x)\|^2   \leq \sigma^2_h   , \\
&   \mathbb{E}_{\zeta_k }\| \nabla g_k(x ; \zeta_k) - \nabla g_k(x)\|^2   \leq \sigma^2_g ,~
  \mathbb{E}_{\zeta_k }\|  g_k(x ; \zeta_k) -  g_k(x)\|^2  \leq \sigma^2_g, 
\end{align*}
for some $\sigma_h, \sigma_g > 0$ and for all $x \in \R^d$ and $k \in [K]$. 
 \end{assump}

\begin{assump}
[Bounded Heterogeneity]
\label{Ass: BoundedHetero}
The heterogeneity $h_k(\cdot)$ and $g_k(\cdot)$ is characterized as
\begin{align*}
  \sup_{x \in \R^d}  \| \nabla h_k(x) - \nabla h(x) \|^2  \leq \Delta_h^2 ~~\text{and}~~
   \sup_{x \in \R^d}     \| \nabla g_k(x) - \nabla g(x) \|^2 \leq \Delta_g^2,
\end{align*}
for some $\Delta_h, \Delta_g > 0$ for all $k \in [K]$.
\end{assump}
A few comments regarding the assumptions are in order. We note that the above assumptions are commonplace in the context of non-convex CO problems. Specifically, Assumption \ref{Ass: Lip} is required to establish Lipschitz smoothness of the $\Phi(\cdot)$ (see Lemma \ref{Lem: Lip}) and is standard in the analyses of CO problems \cite{wang2017stochastic,chen2021solving,khanduri2023proximal}. Assumption \ref{Ass: BoundedVar} captures the effect of stochasticity in the gradient and function evaluations of the CO problem while Assumption \ref{Ass: BoundedHetero} characterizes the data heterogeneity among clients. We note that these assumptions are standard and have been utilized in the past to establish the convergence of many FL non-CO algorithms \cite{Yu_Jin_Arxiv_2019linear,praneeth2019scaffold,khanduri2021stem,zhang2021fedpd, Woodworth_Minibatch_Arxiv_2020}.

\begin{lem}[Lipschitzness of $\Phi$]
\label{Lem: Lip}
Under Assumption \ref{Ass: Lip} the compositional function, $\Phi(\cdot)$, defined in (\ref{Eq: Fl_Prob}) is Lipschitz smooth with constant: $L_\Phi \coloneqq L_h + B_f L_g + B_g^2 L_f > 0$.
\end{lem}
Lemma \ref{Lem: Lip} establishes Lipschitz smoothness (Definition \ref{Def: Lip}) of the compositional function $\Phi(\cdot)$. In general, $\Phi(\cdot)$ is a non-convex function, and therefore, we cannot expect to globally solve (\ref{Eq: Fl_Prob}). We instead rely on finding approximate stationary points of $\Phi(\cdot)$ defined next.

\begin{defn}[$\epsilon$-stationary point]
\label{Def: StationaryPt}
A point $x$ generated by a stochastic algorithm is an $\epsilon$-stationary point of a differentiable function $\Phi(\cdot)$ if $\Ebb \| \nabla \Phi(x) \|^2 \leq \epsilon$, where the expectation is taken with respect to the stochasticity of the algorithm. 
\end{defn}

\begin{defn}[Sample and Communication Complexity]
\label{Def: Comp}
    The sample complexity is defined as the total number of (stochastic) gradient and function evaluations required to achieve an $\epsilon$-stationary solution. Similarly, communication complexity is defined as the total communication rounds between the clients and the server required to achieve an $\epsilon$-stationary solution. 
\end{defn}

\section{Federated non-convex CO algorithms}
\label{sec: Federated}
 
In this section, we first establish the incapability of vanilla FedAvg to solve CO problems in general. Then, we design a communication-efficient FL algorithm to solve the non-convex CO problem. 

\subsection{Candidate FedAvg algorithms}
\label{subsec: VanillaFedAvg}

\begin{algorithm}[t]
\caption{Vanilla FedAvg for non-convex CO} \label{Algo: FedAvg}
\begin{algorithmic}[1]
\State{\textbf{Input}: Parameters: $\{\eta^t\}_{t=0}^{T-1}$, $I$  }
\State{\textbf{Initialize}: $x^0_k = \bar{x}^0$, $y^0_k = \bar{y}^0$ }
\For{$t = 0$ to $T - 1$}
\For{$k = 1$ to $K$}
\State{ {\begin{minipage}{ 0.76\textwidth}  
$\texttt{Update:}    
\begin{cases}
\text{Compute $\nabla   \Phi_k(x_k^t)$ using (\ref{eq: Deterministic_Grad})}\\
 x^{t+1}_k =  x^t_k - \eta^t \nabla   \Phi_k(x_k^t) \\
  {y}_k^{t + 1} = g_k(x_k^{t+1}) 
	\end{cases}
$\end{minipage}}}  
\If{$t + 1~ \text{mod}~ I = 0$}
\State{
{\begin{minipage}
{0.72\textwidth}$   \texttt{[Case I] Share:}\begin{cases}
 x^{t+1}_k \! = \bar{x}^{t+1}
\end{cases}$\\
$\texttt{[Case II] Share}\!: \begin{cases}\!
 x^{t+1}_k \! = \bar{x}^{t+1}\\
 {y}_k^{t + 1} \! =\! g_k(\bar{x}^{t+1}) \\
 y^{t+1}_k  \!= \!\bar{y}^{t+1}
\end{cases}$
\end{minipage}}
}
\EndIf
\EndFor
\EndFor
 \end{algorithmic}
\end{algorithm}

In this section, we show that vanilla FedAvg is not suitable for solving federated CO problems of form (\ref{Eq: Fl_Prob}). To establish this, we consider a simple deterministic setting with $h(x) = 0$. For this setting, the local gradients of $\Phi(\cdot)$ are estimated as
\begin{align}
\label{eq: Deterministic_Grad}
    \nabla \Phi_k(x) = \nabla g_k(x_k) \nabla f(y_k),
\end{align}
where the sequence $y_k$ represents the local estimate of the inner function $g(x)$. To solve the above problem in a federated setup, we consider two candidate versions of FedAvg described in Case I and II of Algorithm \ref{Algo: FedAvg}. Similar to vanilla FedAvg, each agent performs multiple local updates within each communication round (see Step 5 of Algorithm \ref{Algo: FedAvg}). Moreover, since $g(x) \coloneqq 1/k \sum_{k = 1}^k g_k(x)$ with each agent $k \in [K]$ having access to only the local copy $g_k(\cdot)$, estimating $g(\cdot)$ locally within each communication round is not feasible. Therefore, each agent utilizes $y_k = g_k(x)$ as the local estimate of the inner function $g(\cdot)$. For communication, we consider two protocols. In the first setting, after $I$ local updates, in each communication round the agents share the locally updated parameters with the server and receive the aggregated parameter from the server (see Case I in Step 7). In the second setting, in addition to the locally updated parameters the agents also share their local function evaluations $y_k^t  = g_k(x_k^t)$ with the server and receive the aggregated embedding $\bar{y}^{t}$ from the server. This step is utilized to improve the local estimates of $g(\cdot)$ (see Case II in Step 7). The algorithm executes for a total of $\lfloor T/I \rfloor$ communication rounds.  

In the following, we show that Algorithm \ref{Algo: FedAvg} is not a good choice to solve the federated CO problem presented in (\ref{Eq: Fl_Prob}) even in the simple deterministic setting with $h(x) = 0$. 

\begin{theorem}[Vanilla FedAvg: Non-Convergence for CO]
    There exist functions $f(\cdot)$ and $g_k(\cdot)$ for $k \in [K]$ satisfying Assumptions \ref{Ass: Lip}, \ref{Ass: BoundedVar}, and \ref{Ass: BoundedHetero}, and an initialization strategy such that for a fixed number of local updates $I > 1$, and for any $0< \eta^t  < C_\eta$ for $t \in \{0, 1, \ldots, T - 1\}$ where $C_\eta > 0$ is a constant, the iterates generated by Algorithm \ref{Algo: FedAvg} under both Cases I and II do not converge to the stationary point of $\Phi(\cdot)$, where $\Phi(\cdot)$ is defined in (\ref{Eq: Fl_Prob}) with $h(x) = 0$.
    \label{Thm: FedAvg_No}
\end{theorem}
Theorem \ref{Thm: FedAvg_No} establishes that vanilla FedAvg is not suitable for solving federated CO problems. This naturally leads to the question of how can we modify FedAvg such that it can efficiently solve CO problems of the form (\ref{Eq: Fl_Prob})? Clearly, Theorem \ref{Thm: FedAvg_No} suggests that sharing $y_k$'s in each iteration is required to ensure convergence of FedAvg since sharing the iterates $y_k$'s only intermittently leads to non-convergence of FedAvg. To this end, we propose to modify the FedAvg algorithm as presented in Algorithm \ref{Algo: FedAvg} by sharing $y_k$ in each iteration $t \in \{ 0,1, \ldots, T - 1\}$. The next result shows that the modified FedAvg resolves the non-convergence issue of FedAvg for solving CO problems.     
\begin{theorem}[Modified FedAvg: Convergence for CO]
   Suppose we modify Algorithm \ref{Algo: FedAvg} such that $y_k^t = \bar{y}^t$ is updated at each iteration $t \in \{0, 1, \ldots, T - 1 \}$ instead of $[t+1~ \text{\em mod}~ I]$ iterations as in current version of Algorithm \ref{Algo: FedAvg}. Then if functions $f(\cdot)$ and $g_k(x)$ for $k \in [K]$ satisfy Assumptions \ref{Ass: Lip}, \ref{Ass: BoundedVar}, and \ref{Ass: BoundedHetero} such that for a fixed number of local updates $1 \leq I \leq \mathcal{O}(T^{1/4})$, there exists a choice of $\eta^t > 0$ for $t \in \{0, 1, \ldots, T - 1\}$ such that the iterates generated by (modified) Algorithm \ref{Algo: FedAvg} converge to the stationary point of $\Phi(\cdot)$, where $\Phi(\cdot)$ is defined in (\ref{Eq: Fl_Prob}) with $h(x) = 0$.
    \label{Thm: FedAvg_Yes}
\end{theorem}
Motivated by Theorem \ref{Thm: FedAvg_Yes}, we next develop a federated algorithm, \aname, to solve the problem (\ref{Eq: Fl_Prob}) in a general stochastic setting with $h(x) \neq 0$.

\subsection{Federated non-convex CO algorithm: \aname}
 In this section, we propose a novel distributed non-convex CO algorithm, \aname, for solving (\ref{Eq: Fl_Prob}). Note that as demonstrated in Section \ref{subsec: VanillaFedAvg} this problem is particularly challenging because of the compositional structure of the problem combined with the fact that the data is heterogeneous for each client. Motivated by Theorem \ref{Thm: FedAvg_Yes} above, in this work we develop a novel approach where we utilize the structure of the CO problem to develop efficient FL algorithms for solving (\ref{Eq: Fl_Prob}). Specifically, as also demonstrated in Section \ref{Subsec: Examples} we utilize the fact that the embedding $g(\cdot)$ is a low-dimensional (e.g., $d_g = 1$) mapping, especially for the DRO problems. This implies that sharing of $g(\cdot)$ will be relatively cheap in contrast to the high-dimensional model parameters of size $d$ which can be very large and take values in millions or even in billions for modern overparameterized neural networks \cite{vaswani2017attention}. Therefore, like FedAvg, we share the model parameters intermittently after multiple local updates while sharing the low-dimensional embedding of $g(\cdot)$ frequently to handle the compositional objective.
\begin{algorithm}[t]
\caption{Federated non-convex CO algorithm: \aname} \label{Algo: FL}
\begin{algorithmic}[1]
\State{\textbf{Input}: Parameters: $\{\beta^t\}_{t=0}^{T-1}$,  $\{\eta^t\}_{t=0}^{T-1}$, $I$  }
\State{\textbf{Initialize}: $x^{-1}_k = x^0_k = \bar{x}^0$, $y^0_k = \bar{y}^0$ }
\For{$t = 0$ to $T - 1$}
\For{$k = 1$ to $K$}
	\State{Sample $\xio^t_k = \{b^t_{g_k}, b^t_{h_k} \}$ uniformly randomly from $\mathcal{D}_{g_k}$ and $\mathcal{D}_{h_k}$  respectively}
\State{{\begin{minipage}{ 0.86\textwidth}  
$\texttt{Local Update and Sharing:}    
\begin{cases}
  \text{Compute}~{y}_k^{t} ~\text{using} ~(\ref{Eq: Update_Y_FL})~\text{and share with the server}  \\
  \text{Receive $\bar{y}^{t}$ from the server and update}~y_t^k = \bar{y}^{t}\\
   \text{Compute}~ \nabla \Phi_k(x^t_k ; \xio^t_k)~\text{using (\ref{Eq: SG_FL})}
   \\
   x^{t+1}_k =  x^t_k - \eta^t \nabla \Phi_k(x^t_k ; \xio^t_k)
	\end{cases}
$\end{minipage}}}  
\If{$t + 1~ \text{mod}~ I = 0$}
\State{
{\begin{minipage}
{0.825\textwidth}$   \texttt{Model Sharing}: \begin{cases}
 x^{t+1}_k = \bar{x}^{t+1}
\end{cases}$\end{minipage}}}
\EndIf
\EndFor
\EndFor
\State{{\bf Return:} $\bar{x}^{a(T)}$ where $a(T) \sim {\cal U}\{1,...,T\}$.} 
\end{algorithmic}
\end{algorithm}
Moreover, to solve the CO problems for DRO the developed algorithms generally utilize batch sizes (for gradient/function evaluation) that are dependent on the solution accuracy \cite{huang2021compositional,haddadpour2022learning}. However, this is not feasible in most practical settings. In addition, to control the bias and to circumvent the need to compute large batch gradients, we utilize a momentum-based estimator to learn the compositional function (see (\ref{Eq: Update_Y_FL})) \cite{chen2021solving}. This construction allows us to develop FedAvg-type algorithms for solving non-convex CO problems wherein the local updates resemble the standard SGD updates.

The detailed steps of \aname~are listed in Algorithm \ref{Algo: FL}. During the local updates each client $k \in [K]$ updates its
local model $x_k^t$ for all $t \in [T]$ using the local estimate of the stochastic gradients in Step 6. The local stochastic gradient estimates for each client $k \in [K]$ are denoted by $\nabla \Phi_k(x_k^t; \xio_k)$ and are evaluated using the chain rule of differentiation as
\begin{align}
\label{Eq: SG_FL}
  \nabla \Phi_k(x^t_k ; \xio^t_k)  = \frac{1}{|b^t_{h_k}|} \sum_{i \in b^t_{h_k}} \nabla h_k(x^t_k; \xi_{k,i}^{t})    +   \frac{1}{|b^t_{g_k}|} \sum_{j \in b^t_{g_k}} \nabla g_k(x^t_k;\zeta_{k,j}^{t} ) \nabla f(\bar{y}^t)
\end{align}
where $\xio_k^t = \{b_{h_k}^t, b_{g_k}^t \}$ represents the stochasticity of the gradient estimate and $b_{h_k}^t  = \{ \xi^t_{k,i}\}_{i  = 1}^{|b_{h_k}^t|}$ (resp. $b_{g_k}^t = \{ \zeta^t_{k,i}\}_{i  = 1}^{|b_{g_k}^t|}$) denotes the batch of stochastic samples of $h_k(\cdot)$ (resp. $g_k(\cdot)$) utilized to compute the stochastic gradient for each $k \in [K]$ and $t \in \{0,1, \ldots, T - 1\}$. The variable $\bar{y}^t$ is designed to estimate the inner function $ {1}/{K}\sum_{k = 1}^K g_k(x)$ in (\ref{Eq: Fl_Prob}). A standard approach to estimate $g_k(x)$ locally for each $k \in [K]$ is to utilize a large batch such that the gradient bias from the inner function estimate can be controlled \cite{guo2022fedx,huang2021compositional,haddadpour2022learning}. In contrast, we adopt a momentum-based estimate of $g_k(\cdot)$ at each client $k \in [K]$ that leads to a small bias asymptotically \cite{chen2021solving}. We note that the estimator utilizes a hybrid estimator that combines a SARAH \cite{Nguyen_ICML_2017_SARAH} and SGD \cite{Ghadimi_Siam_2013_SGD} estimate for the function values rather than the gradients \cite{Cutkosky_NIPS2019}. Specifically, 
individual $y_k^t$'s are estimated in Step 6 as  
\begin{align}
\label{Eq: Update_Y_FL}
     y^{t}_k  = (1 - \beta^t) \Big( y^{t-1}_k - \frac{1}{|b^t_{g_k}|} \sum_{i \in b_t^{g_k}}   g_k(x^{t-1}_k; \zeta_{k,i}^t)  \Big)    +   \frac{1}{|b^t_{g_k}|} \! \sum_{i \in b^t_{g_k}}    g_k(x_k^t; \zeta_{k,i}^t).
\end{align}
for all $k \in [K]$ and where $\beta^t \in (0,1)$ is the momentum parameter. Motivated by the discussion in Section \ref{subsec: VanillaFedAvg}, the parameters $y^{t}_k \in \R^{d_g}$ are shared with the server after the $y_k^t$ update, however, this sharing will not incur a significant communication cost since $y^{t}_k$'s are usually low dimensional embeddings (often a scalar with $d_g = 1$) as illustrated in Section \ref{Subsec: Examples} for DRO problems. The model parameters are then updated using the SG evaluated using (\ref{Eq: SG_FL}). Finally, after $I$ local updates the model potentially high-dimensional model parameters are aggregated at the server and broadcasted back to the clients after aggregation in Step 8. Next, we state the convergence guarantees.

 \section{Main result: Convergence of \aname}
 \label{sec: Convergence_FL}
In the next theorem, we first state the main result of the paper detailing the convergence of \aname.

 \begin{theorem}[Convergence of \aname]
 \label{Thm: FL}
For Algorithm \ref{Algo: FL}, choosing the step-size $\eta^t = \eta = \sqrt{{|b| K}/{T}}$ and the momentum parameter $\beta^t = 4 B_g^4 L_f^2 \cdot \eta^t$ for all $t \in \{0,1, \ldots, T - 1\}$. Moreover, with the selection of batch sizes $|b_{h_k}^t| = |b_{g_k}^t| = |b|$ for all $t \in \{ 0,1, \ldots, T - 1\}$ and $k \in [K]$, and for $T  \geq T_{\text{th}}$ where $T_{\text{th}}$ is defined in Appendix \ref{App: FL}, then under Assumptions \ref{Ass: Lip}, \ref{Ass: BoundedVar} and \ref{Ass: BoundedHetero} for $\bar{x}^{a(T)}$ chosen According to Algorithm \ref{Algo: FL}, we have 
    \begin{align*}
        \Ebb \big\|\nabla \Phi(\bar{x}^{a(T)}) \big\|^2 &  \leq 
 \underbrace{ \frac{2 \big[\Phi(\bar{x}^0) - \Phi(x^\ast) + \big\| \bar{y}^{0} - g(\bar{x}^0) \big\|^2 \big]}{\sqrt{|b| K T}}}_{\text{Initialization}} \\
 & \qquad + \mathcal{C} (|b|, K, T, I) \underbrace{\Big[ {C_{\sigma_h}}  \sigma_h^2  +   {C_{\sigma_g}}  \sigma_g^2 \Big]}_{\text{Variance}}   +  ~  \mathcal{C}(|b|, K, T, I)  
 \underbrace{\Big[ {C_{\Delta_h}}  \Delta_h^2   +   {C_{\Delta_g}} \Delta_g^2 \Big]}_{\text{Heterogeneity}} ,  
 \end{align*}
where $\mathcal{C}(|b|, K, T, I)  \coloneqq \max \bigg\{ 
\frac{|b|K (I - 1)^2}{T} , \frac{1}{\sqrt{|b| K T}} \bigg\}$ and constants $C_{\sigma_h}$, $C_{\sigma_g}$, $C_{\Delta_h}$, and $C_{\Delta_g}$ are defined in Appendix \ref{App: FL}
\end{theorem}
We note that the condition on $T \geq T_{\text{th}}$ is required for theoretical purposes. Specifically, it ensures that the step-size $\eta = \sqrt{|b|K/T}$ is upper-bounded. A similar requirement has also been posed in \cite{Yu_Jin_Arxiv_2019linear, Yu_Zhu_2018parallel, khanduri2021stem} in the past. Theorem \ref{Thm: FL} captures the effect of heterogeneity, stochastic variance, and the initialization on the performance of \aname. As can be seen from the expression in Theorem \ref{Thm: FL} the heterogeneity degrades the performance when the local updates, $I$, increase beyond a threshold, i.e., when the term $|b| K(I - 1)^2/T$ dominates $1/\sqrt{|b| K T}$. The next result characterizes the possible choices of $I$ that ensure the efficient convergence of \aname.
\begin{cor}[Local Updates]
\label{Cor: FL}
Under the setting of Theorem \ref{Thm: FL} and choosing the number of local updates, $I$, such that we have $I \leq \mathcal{O} (T^{1/4} / (|b| K)^{3/4})$, the iterate $\bar{x}^{a(T)}$ chosen according to Algorithm \ref{Algo: FL} satisfies
\begin{align*}
       &  \Ebb \big\|\nabla \Phi(\bar{x}^{a(T)}) \big\|^2  \!\!  \leq \underbrace{\frac{2 \big[\Phi(\bar{x}^0) - \Phi(x^\ast) + \big\| \bar{y}^{0} - g(\bar{x}^0) \big\|^2 \big]}{\sqrt{|b| K T}}}_{\text{Initialization}}    +   \underbrace{\frac{  {C_{\sigma_h}}  \sigma_h^2  +   {C_{\sigma_g}}  \sigma_g^2  }{\sqrt{|b| K T}}   }_{\text{Variance}}  +   \underbrace{ 
  \frac{ {C_{\Delta_h}}  \Delta_h^2   +   {C_{\Delta_g}} \Delta_g^2  }{\sqrt{|b| K T}}   }_{\text{Heterogeneity}} .   \end{align*}  
\end{cor}
\label{Cor: FL_Comm}
Corollary \ref{Cor: FL} states that there exists a choice of the number of local updates that guarantee that \aname~achieves the same convergence performance as a standard FedAvg \cite{praneeth2019scaffold, Woodworth_Minibatch_Arxiv_2020, Yu_Jin_Arxiv_2019linear,khanduri2021stem} for solving the non-CO problems. Next, we characterize the sample and communication complexities of \aname.
\begin{cor}[Sample and Communication Complexities]
Under the setting of Theorem \ref{Thm: FL} and choosing the number of local updates as $I =\mathcal{O} (T^{1/4} / (|b| K)^{3/4})$ the following holds 
\begin{enumerate}[leftmargin=*]
    \item[(i)]
    The  {\bf \em sample complexity} of \aname~is $\mathcal{O}(\epsilon^{-2})$. This implies that each client requires $\mathcal{O}(K^{-1}\epsilon^{-2})$ samples to reach an $\epsilon$-stationary point achieving linear speed-up. 
    \item[(ii)] The {\bf \em communication complexity} of \aname~is $O(\epsilon^{-3/2})$. 
\end{enumerate}
\end{cor}
The sample and communication complexities guaranteed by Corollary \ref{Cor: FL_Comm} match that of the standard FedAvg \cite{Yu_Zhu_2018parallel} for solving stochastic non-convex non-CO problems. 
We note that in addition to the $O(\epsilon^{-3/2})$ communication complexity that measures the sharing of high-dimensional parameters, \aname~also shares $\mathcal{O}(K^{-1}\epsilon^{-2})$ low-dimensional embeddings (usually scalar values as illustrated in Section \ref{Subsec: Examples}). Therefore, the total real values shared by each client during the execution of \aname~is $\mathcal{O}(\epsilon^{-3/2} d + K^{-1}\epsilon^{-2})$. Notice that for high-dimensional models like training (large) neural networks, we will usually have $d K \geq \mathcal{O}(\epsilon^{-0.5})$ meaning the total communication will be $\mathcal{O}(\epsilon^{-3/2} d)$ which is better than any Federated CO algorithm proposed in the literature \cite{huang2021compositional,gao2022convergence,guo2022fedx}. Importantly, to our knowledge this is the first work that ensures linear speed up in a federated CO setting, moreover, \aname~achieves this performance without relying on the computation of large batch sizes. 

\section{Experiments}
\label{Sec: Experiments}
In this section, we evaluate the performance of \aname~with both centralized and distributed baselines. We, a) establish the superior performance of \aname~in terms of training/testing accuracy, and b) evaluate the performance of \aname~with different numbers of local updates to capture the effect of data heterogeneity. To evaluate the performance of \aname, we focus on two tasks: classification with an imbalanced dataset and learning with fairness constraints. For the first task, we use CIFAR10-ST and CIFAIR-100-ST datasets \cite{qi2020simple} (unbalanced versions of CIFAR10 and CIFAR100 \cite{krizhevsky2009learning}) for image classification, and the performance is measured by training and testing accuracy achieved by different algorithms. For the second task, we use the Adult dataset \cite{Dua:2019} for enforcing equality of opportunity (on protected classes) on tabular data classification \cite{hardt2016equality}. For this setting, the performance is evaluated by training/testing accuracy, and the constraint violations, which are measured by the gap between the true positive rate of the overall data and the protected groups \cite{haddadpour2022learning}. Please see Appendix \ref{App: Add_Exp} for a detailed discussion of the classification problem, dataset description, experiment settings, and additional experimental evaluation.

\paragraph{Baseline methods.} For the CIFAR10-ST and CIFAR100-ST datasets we compare \aname~with popular centralized baselines for classification with imbalanced data. The baselines adopted for comparison are a popular DRO method, FastDRO \cite{levy2020large}, a primal-dual SGD approach to solve constrained problems with many constraints, PDSGD \cite{xu2020primal}, and a popular baseline minibatch SGD, MBSGD, customized for CO   \cite{Ghadimi_Siam_2013_SGD}. For the adult dataset, we use GCIVR \cite{haddadpour2022learning} as the baseline distributed model to compare with \aname, since like \aname~it is the only algorithm that can deal with compositional and non-compositional objectives at the same time. We also implement a simple parallel SGD as a baseline that ignores the fairness constraints, referred to as unconstrained in the experiments.  

\begin{figure}[t]
\centering
\includegraphics[width=\textwidth]{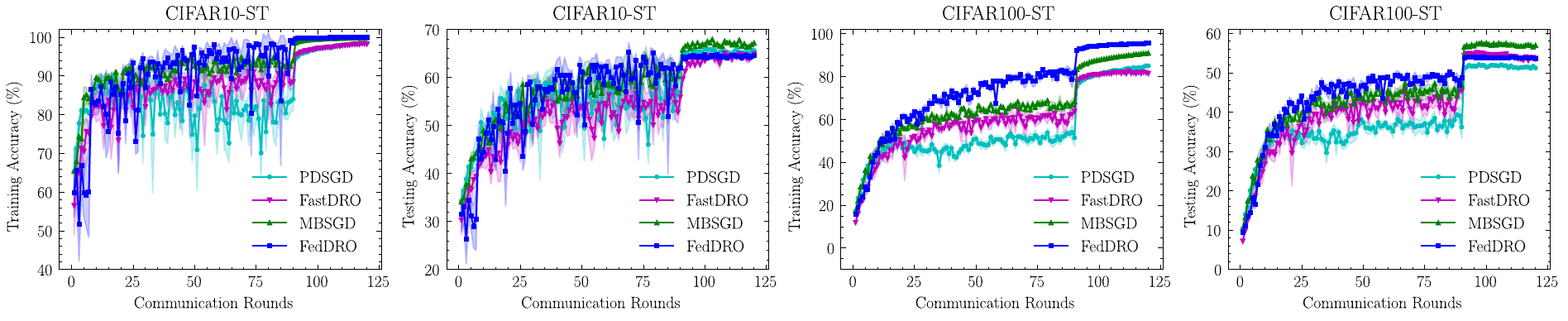}
\caption{Train and test accuracy vs communication rounds for CIFAR10-ST and CIFAR100-ST. }
\label{Fig: CIFAR}
\end{figure}
    
\begin{figure}[ht]
\centering
\includegraphics[width=0.7 \textwidth]{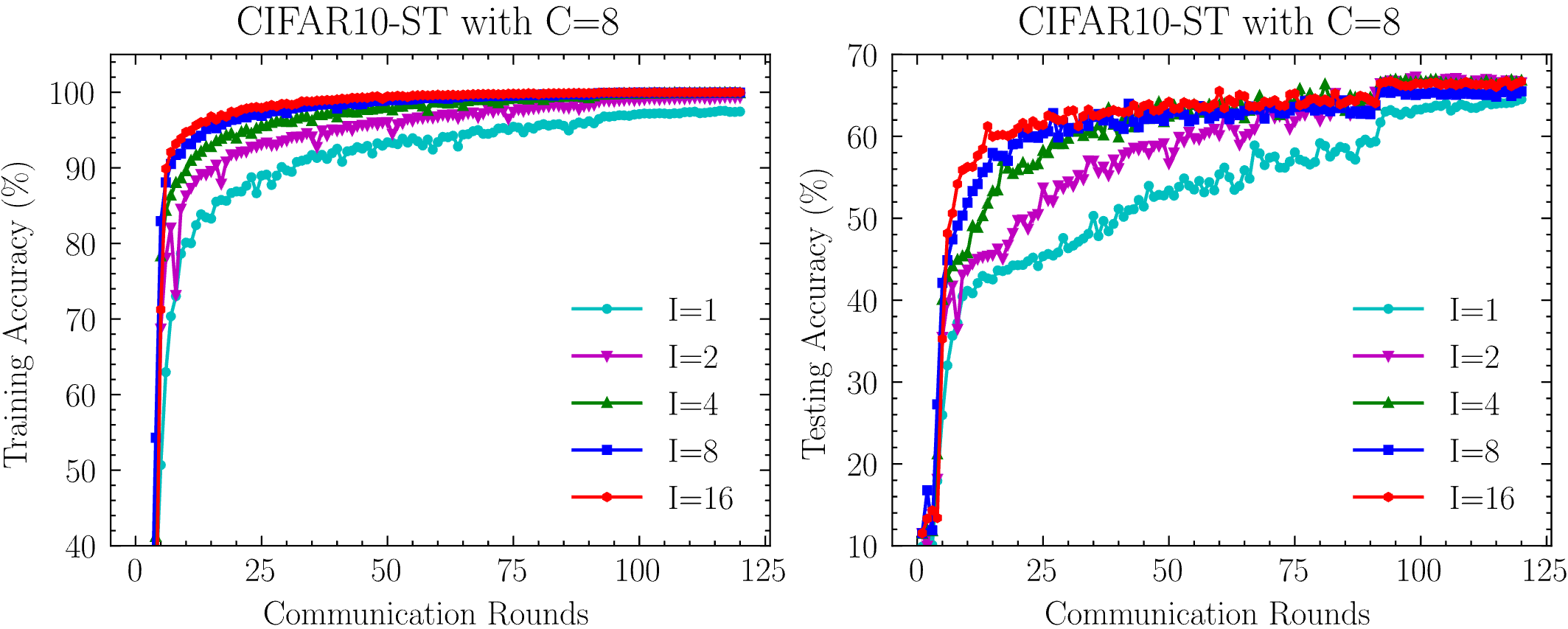} 
\caption{Train and test accuracy of \aname~on the CIFAR10-ST and CIFAR100-ST for different $I$.} 
\label{Fig: CIFAR_I}
\end{figure}

\paragraph{Implementation details.}  We use $8$ clients to model the distributed setting and split the (unbalanced) dataset equally for each client. We use ResNet20 for classification tasks on CIFAR10-ST and CIFAR100-ST datasets. For a fair comparison with centralized baselines, we choose $I = 1$ for \aname~and implement a parallel version of the centralized algorithms where the overall gradient computation is $K$ times larger for each algorithm. This is to make sure that the overall gradient computations in each step are uniform across all algorithms. Performance with different values of $I$ is evaluated separately. For each algorithm, we used a batch size of $16$ per client, and the learning rates were tuned from the set $\{0.001, 0.01, 0.05, 0.1\}$, the learning rate was dropped to $1/10^{\text{th}}$ after $90$ communication rounds. For fairness-constrained classification on the Adult dataset, we use a logistic regression model. For this experiment, we adopt the parameter settings suggested in \cite{haddadpour2022learning}, for \aname~we keep the same setting as in the earlier task. All results are averaged over $5$ independent runs.

\begin{figure}[t] 
\includegraphics[width=\textwidth]{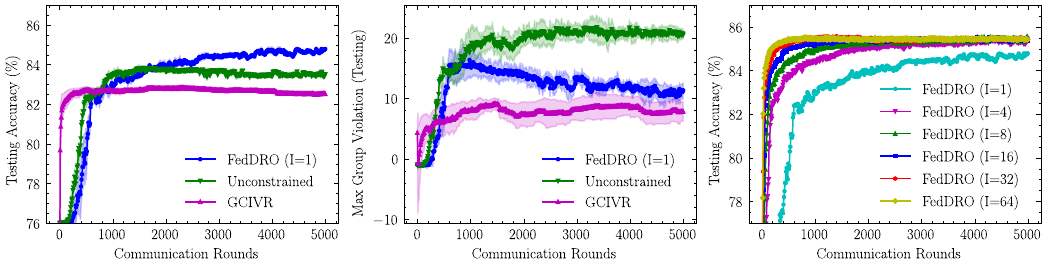}
\caption{Comparison of \aname, GCIVR, and the unconstrained baseline (first two figures). Performance of \aname~with different $I$ (rightmost figure).}
\label{Fig: Adult_Test}
\end{figure}

\paragraph{Discussion.} In Figure \ref{Fig: CIFAR}, we evaluate the performance of \aname~against the parallel implementations of the centralized baselines on unbalanced CIFAR datasets. Note that \aname~provides superior training and comparable test accuracy to the state-of-the-art methods. In Figure \ref{Fig: CIFAR_I}, we evaluate the performance of \aname~for a different number of local updates, $I$. Note that as $I$ increases the performance improves, however, beyond a certain, $I$, the performance doesn't improve capturing the effect of client drift because of data heterogeneity. Finally,  
in Figure \ref{Fig: Adult_Test} we assess the test performance of \aname~against the distributed baseline GCIVR on the Adult dataset. We observe that \aname~outperforms both GCIVR and unconstrained formulation in terms of accuracy and matches the constraint violation performance of GCIVR as communication rounds increase. Finally, for the rightmost image we evaluate the performance of \aname~with different values of $I$, we notice that increasing the value of $I$ leads to improved performance, however, beyond a certain threshold (approximately over 32), the performance saturates as a consequence of client drift.

\section{Conclusion and limitations} In this work, we first established that vanilla FedAvg algorithms are incapable of solving CO problems in the FL setting. To address this challenge, we showed that additional (low-dimensional) embeddings of the stochastic compositional objective are required to be shared to guarantee convergence of the SGD-based FL algorithms to solve CO of the form (\ref{Eq: Fl_Prob}). To this end, we proposed \aname, the first federated CO framework that achieves linear speedup with the number of clients without requiring the computation of large batch sizes. We conducted numerical experiments on various real data sets to show the superior performance of \aname~compared to state-of-the-art. An interesting future problem to be addressed includes limiting the privacy leakage of \aname~while sharing the low-dimensional embeddings.

\newpage

\bibliographystyle{IEEEtran}
\bibliography{abrv, References}

\begin{thebibliography}{10}
\providecommand{\url}[1]{#1}
\csname url@samestyle\endcsname
\providecommand{\newblock}{\relax}
\providecommand{\bibinfo}[2]{#2}
\providecommand{\BIBentrySTDinterwordspacing}{\spaceskip=0pt\relax}
\providecommand{\BIBentryALTinterwordstretchfactor}{4}
\providecommand{\BIBentryALTinterwordspacing}{\spaceskip=\fontdimen2\font plus
\BIBentryALTinterwordstretchfactor\fontdimen3\font minus \fontdimen4\font\relax}
\providecommand{\BIBforeignlanguage}[2]{{%
\expandafter\ifx\csname l@#1\endcsname\relax
\typeout{** WARNING: IEEEtran.bst: No hyphenation pattern has been}%
\typeout{** loaded for the language `#1'. Using the pattern for}%
\typeout{** the default language instead.}%
\else
\language=\csname l@#1\endcsname
\fi
#2}}
\providecommand{\BIBdecl}{\relax}
\BIBdecl

\bibitem{qi2022stochastic}
Q.~Qi, J.~Lyu, E.~W. Bai, T.~Yang \emph{et~al.}, ``Stochastic constrained {DRO} with a complexity independent of sample size,'' \emph{arXiv preprint arXiv:2210.05740}, 2022.

\bibitem{finn2017model}
C.~Finn, P.~Abbeel, and S.~Levine, ``Model-agnostic meta-learning for fast adaptation of deep networks,'' in \emph{International Conference on Machine Learning}.\hskip 1em plus 0.5em minus 0.4em\relax PMLR, 2017, pp. 1126--1135.

\bibitem{duchi2019solving}
J.~C. Duchi and F.~Ruan, ``Solving (most) of a set of quadratic equalities: Composite optimization for robust phase retrieval,'' \emph{Information and Inference: A Journal of the IMA}, vol.~8, no.~3, pp. 471--529, 2019.

\bibitem{shapiro2021lectures}
A.~Shapiro, D.~Dentcheva, and A.~Ruszczynski, \emph{Lectures on stochastic programming: Modeling and theory}.\hskip 1em plus 0.5em minus 0.4em\relax SIAM, 2021.

\bibitem{wang2017stochastic}
M.~Wang, E.~X. Fang, and H.~Liu, ``Stochastic compositional gradient descent: {A}lgorithms for minimizing compositions of expected-value functions,'' \emph{Mathematical Programming}, vol. 161, no.~1, pp. 419--449, 2017.

\bibitem{haddadpour2022learning}
F.~Haddadpour, M.~M. Kamani, M.~Mahdavi, and A.~Karbasi, ``Learning distributionally robust models at scale via composite optimization,'' \emph{arXiv preprint arXiv:2203.09607}, 2022.

\bibitem{chen2022learning}
M.~Chen, Y.~Zhao, B.~He, Z.~Han, B.~Wu, and J.~Yao, ``Learning with noisy labels over imbalanced subpopulations,'' \emph{arXiv preprint arXiv:2211.08722}, 2022.

\bibitem{qi2020attentional}
Q.~Qi, Y.~Xu, R.~Jin, W.~Yin, and T.~Yang, ``Attentional biased stochastic gradient for imbalanced classification,'' \emph{arXiv preprint arXiv:2012.06951}, 2020.

\bibitem{chen2018robust}
R.~Chen and I.~C. Paschalidis, ``A robust learning approach for regression models based on distributionally robust optimization,'' \emph{Journal of Machine Learning Research}, vol.~19, no.~13, 2018.

\bibitem{nemirovski2009robust}
A.~Nemirovski, A.~Juditsky, G.~Lan, and A.~Shapiro, ``Robust stochastic approximation approach to stochastic programming,'' \emph{SIAM Journal on optimization}, vol.~19, no.~4, pp. 1574--1609, 2009.

\bibitem{kairouz2021advances}
P.~Kairouz, H.~B. McMahan, B.~Avent, A.~Bellet, M.~Bennis, A.~N. Bhagoji, K.~Bonawitz, Z.~Charles, G.~Cormode, R.~Cummings \emph{et~al.}, ``Advances and open problems in federated learning,'' \emph{Foundations and Trends{\textregistered} in Machine Learning}, vol.~14, no. 1--2, pp. 1--210, 2021.

\bibitem{Mcmahan_PMLR_2017}
B.~McMahan, E.~Moore, D.~Ramage, S.~Hampson, and B.~A. y~Arcas, ``Communication-efficient learning of deep networks from decentralized data,'' in \emph{Artificial Intelligence and Statistics}.\hskip 1em plus 0.5em minus 0.4em\relax PMLR, 2017, pp. 1273--1282.

\bibitem{li2019convergence}
X.~Li, K.~Huang, W.~Yang, S.~Wang, and Z.~Zhang, ``On the convergence of {F}ed{A}vg on non-iid data,'' \emph{arXiv preprint arXiv:1907.02189}, 2019.

\bibitem{li2020federated}
T.~Li, A.~K. Sahu, M.~Zaheer, M.~Sanjabi, A.~Talwalkar, and V.~Smith, ``Federated optimization in heterogeneous networks,'' \emph{Proceedings of Machine learning and systems}, vol.~2, pp. 429--450, 2020.

\bibitem{praneeth2019scaffold}
S.~Karimireddy, S.~Kale, M.~Mohri, S.~J. Reddi, S.~U. Stich, and A.~Theertha~Suresh, ``{SCAFFOLD}: Stochastic controlled averaging for federated learning,'' \emph{arXiv e-prints}, pp. arXiv--1910, 2019.

\bibitem{Sharma_Arxiv_2019}
P.~Sharma, P.~Khanduri, S.~Bulusu, K.~Rajawat, and P.~K. Varshney, ``Parallel restarted {SPIDER} -- {C}ommunication efficient distributed nonconvex optimization with optimal computation complexity,'' \emph{arXiv preprint arXiv:1912.06036}, 2019.

\bibitem{zhang2021fedpd}
X.~Zhang, M.~Hong, S.~Dhople, W.~Yin, and Y.~Liu, ``Fed{PD}: A federated learning framework with adaptivity to non-iid data,'' \emph{IEEE Transactions on Signal Processing}, vol.~69, pp. 6055--6070, 2021.

\bibitem{khanduri2021stem}
P.~Khanduri, P.~Sharma, H.~Yang, M.~Hong, J.~Liu, K.~Rajawat, and P.~Varshney, ``{STEM}: A stochastic two-sided momentum algorithm achieving near-optimal sample and communication complexities for federated learning,'' \emph{Advances in Neural Information Processing Systems}, vol.~34, pp. 6050--6061, 2021.

\bibitem{karimireddy2020mime}
S.~P. Karimireddy, M.~Jaggi, S.~Kale, M.~Mohri, S.~J. Reddi, S.~U. Stich, and A.~T. Suresh, ``Mime: Mimicking centralized stochastic algorithms in federated learning,'' \emph{arXiv preprint arXiv:2008.03606}, 2020.

\bibitem{chen2021solving}
T.~Chen, Y.~Sun, and W.~Yin, ``Solving stochastic compositional optimization is nearly as easy as solving stochastic optimization,'' \emph{IEEE Transactions on Signal Processing}, vol.~69, pp. 4937--4948, 2021.

\bibitem{huang2021compositional}
F.~Huang, J.~Li, and H.~Huang, ``Compositional federated learning: Applications in distributionally robust averaging and meta learning,'' \emph{arXiv preprint arXiv:2106.11264}, 2021.

\bibitem{guo2022fedx}
Z.~Guo, R.~Jin, J.~Luo, and T.~Yang, ``Fed{X}: Federated learning for compositional pairwise risk optimization,'' \emph{arXiv preprint arXiv:2210.14396}, 2022.

\bibitem{gao2022convergence}
H.~Gao, J.~Li, and H.~Huang, ``On the convergence of local stochastic compositional gradient descent with momentum,'' in \emph{International Conference on Machine Learning}.\hskip 1em plus 0.5em minus 0.4em\relax PMLR, 2022, pp. 7017--7035.

\bibitem{tarzanagh2022fednest}
D.~A. Tarzanagh, M.~Li, C.~Thrampoulidis, and S.~Oymak, ``Fed{N}est: Federated bilevel, minimax, and compositional optimization,'' \emph{arXiv preprint arXiv:2205.02215}, 2022.

\bibitem{khanduri2023proximal}
P.~Khanduri, C.~Li, R.~I. Sultan, Y.~Qiang, J.~Kliewer, and D.~Zhu, ``Proximal compositional optimization for distributionally robust learning,'' in \emph{The Second Workshop on New Frontiers in Adversarial Machine Learning}, 2023.

\bibitem{Yu_Jin_Arxiv_2019linear}
H.~Yu, R.~Jin, and S.~Yang, ``On the linear speedup analysis of communication efficient momentum {SGD} for distributed non-convex optimization,'' in \emph{International Conference on Machine Learning}.\hskip 1em plus 0.5em minus 0.4em\relax PMLR, 2019, pp. 7184--7193.

\bibitem{Woodworth_Minibatch_Arxiv_2020}
B.~Woodworth, K.~K. Patel, and N.~Srebro, ``Minibatch vs local {SGD} for heterogeneous distributed learning,'' \emph{arXiv preprint arXiv:2006.04735}, 2020.

\bibitem{vaswani2017attention}
A.~Vaswani, N.~Shazeer, N.~Parmar, J.~Uszkoreit, L.~Jones, A.~N. Gomez, {\L}.~Kaiser, and I.~Polosukhin, ``Attention is all you need,'' \emph{Advances in neural information processing systems}, vol.~30, 2017.

\bibitem{Nguyen_ICML_2017_SARAH}
L.~M. Nguyen, J.~Liu, K.~Scheinberg, and M.~Tak{\'a}{\v{c}}, ``{SARAH}: A novel method for machine learning problems using stochastic recursive gradient,'' in \emph{Proceedings of the 34th International Conference on Machine Learning-Volume 70}.\hskip 1em plus 0.5em minus 0.4em\relax JMLR. org, 2017, pp. 2613--2621.

\bibitem{Ghadimi_Siam_2013_SGD}
S.~Ghadimi and G.~Lan, ``Stochastic first-and zeroth-order methods for nonconvex stochastic programming,'' \emph{SIAM Journal on Optimization}, vol.~23, no.~4, pp. 2341--2368, 2013.

\bibitem{Cutkosky_NIPS2019}
A.~Cutkosky and F.~Orabona, ``Momentum-based variance reduction in non-convex {SGD},'' in \emph{Advances in Neural Information Processing Systems 32}.\hskip 1em plus 0.5em minus 0.4em\relax Curran Associates, Inc., 2019, pp. 15\,236--15\,245.

\bibitem{Yu_Zhu_2018parallel}
H.~Yu, S.~Yang, and S.~Zhu, ``Parallel restarted {SGD} with faster convergence and less communication: Demystifying why model averaging works for deep learning,'' in \emph{Proceedings of the AAAI Conference on Artificial Intelligence}, vol.~33, 2019, pp. 5693--5700.

\bibitem{qi2020simple}
Q.~Qi, Y.~Yan, Z.~Wu, X.~Wang, and T.~Yang, ``A simple and effective framework for pairwise deep metric learning,'' in \emph{Computer Vision--ECCV 2020: 16th European Conference, Glasgow, UK, August 23--28, 2020, Proceedings, Part XXVII 16}.\hskip 1em plus 0.5em minus 0.4em\relax Springer, 2020, pp. 375--391.

\bibitem{krizhevsky2009learning}
A.~Krizhevsky, G.~Hinton \emph{et~al.}, ``Learning multiple layers of features from tiny images,'' 2009.

\bibitem{Dua:2019}
\BIBentryALTinterwordspacing
D.~Dua and C.~Graff, ``{UCI} machine learning repository,'' 2017. [Online]. Available: \url{http://archive.ics.uci.edu/ml}
\BIBentrySTDinterwordspacing

\bibitem{hardt2016equality}
M.~Hardt, E.~Price, and N.~Srebro, ``Equality of opportunity in supervised learning,'' \emph{Advances in neural information processing systems}, vol.~29, 2016.

\bibitem{levy2020large}
D.~Levy, Y.~Carmon, J.~C. Duchi, and A.~Sidford, ``Large-scale methods for distributionally robust optimization,'' \emph{Advances in Neural Information Processing Systems}, vol.~33, pp. 8847--8860, 2020.

\bibitem{xu2020primal}
Y.~Xu, ``Primal-dual stochastic gradient method for convex programs with many functional constraints,'' \emph{SIAM Journal on Optimization}, vol.~30, no.~2, pp. 1664--1692, 2020.

\bibitem{wang2016accelerating}
M.~Wang, J.~Liu, and E.~Fang, ``Accelerating stochastic composition optimization,'' \emph{Advances in Neural Information Processing Systems}, vol.~29, 2016.

\bibitem{ghadimi2020single}
S.~Ghadimi, A.~Ruszczynski, and M.~Wang, ``A single timescale stochastic approximation method for nested stochastic optimization,'' \emph{SIAM Journal on Optimization}, vol.~30, no.~1, pp. 960--979, 2020.

\bibitem{lian2017finite}
X.~Lian, M.~Wang, and J.~Liu, ``Finite-sum composition optimization via variance reduced gradient descent,'' in \emph{Artificial Intelligence and Statistics}.\hskip 1em plus 0.5em minus 0.4em\relax PMLR, 2017, pp. 1159--1167.

\bibitem{zhang2019stochastic}
J.~Zhang and L.~Xiao, ``A stochastic composite gradient method with incremental variance reduction,'' \emph{Advances in Neural Information Processing Systems}, vol.~32, 2019.

\bibitem{hu2019efficient}
W.~Hu, C.~J. Li, X.~Lian, J.~Liu, and H.~Yuan, ``Efficient smooth non-convex stochastic compositional optimization via stochastic recursive gradient descent,'' \emph{Advances in Neural Information Processing Systems}, vol.~32, 2019.

\bibitem{ben2013robust}
A.~Ben-Tal, D.~Den~Hertog, A.~De~Waegenaere, B.~Melenberg, and G.~Rennen, ``Robust solutions of optimization problems affected by uncertain probabilities,'' \emph{Management Science}, vol.~59, no.~2, pp. 341--357, 2013.

\bibitem{bertsimas2018data}
D.~Bertsimas, V.~Gupta, and N.~Kallus, ``Data-driven robust optimization,'' \emph{Mathematical Programming}, vol. 167, pp. 235--292, 2018.

\bibitem{duchi2021statistics}
J.~C. Duchi, P.~W. Glynn, and H.~Namkoong, ``Statistics of robust optimization: A generalized empirical likelihood approach,'' \emph{Mathematics of Operations Research}, vol.~46, no.~3, pp. 946--969, 2021.

\bibitem{namkoong2017variance}
H.~Namkoong and J.~C. Duchi, ``Variance-based regularization with convex objectives,'' \emph{Advances in neural information processing systems}, vol.~30, 2017.

\bibitem{staib2019distributionally}
M.~Staib and S.~Jegelka, ``Distributionally robust optimization and generalization in kernel methods,'' \emph{Advances in Neural Information Processing Systems}, vol.~32, 2019.

\bibitem{yan2019stochastic}
Y.~Yan, Y.~Xu, Q.~Lin, L.~Zhang, and T.~Yang, ``Stochastic primal-dual algorithms with faster convergence than {$O (1/\sqrt {T})$} for problems without bilinear structure,'' \emph{arXiv preprint arXiv:1904.10112}, 2019.

\bibitem{song2021variance}
C.~Song, S.~J. Wright, and J.~Diakonikolas, ``Variance reduction via primal-dual accelerated dual averaging for nonsmooth convex finite-sums,'' in \emph{International Conference on Machine Learning}.\hskip 1em plus 0.5em minus 0.4em\relax PMLR, 2021, pp. 9824--9834.

\bibitem{alacaoglu2022complexity}
A.~Alacaoglu, V.~Cevher, and S.~J. Wright, ``On the complexity of a practical primal-dual coordinate method,'' \emph{arXiv preprint arXiv:2201.07684}, 2022.

\bibitem{tran2020hybrid}
Q.~Tran~Dinh, D.~Liu, and L.~Nguyen, ``Hybrid variance-reduced {SGD} algorithms for minimax problems with nonconvex-linear function,'' \emph{Advances in Neural Information Processing Systems}, vol.~33, pp. 11\,096--11\,107, 2020.

\end{thebibliography}

 \newpage
\onecolumn	
\section*{Appendix}
 \appendix

\section{Related work}
\label{App: Related Work}
\paragraph{Centralized CO.}  
The first non-asymptotic analysis of stochastic CO problems was performed in \cite{wang2017stochastic} where the authors proposed SCGD a two-timescale algorithm for solving problem (\ref{Eq: Basic_CompositeOpt}). The convergence of SCGD was improved in \cite{wang2016accelerating} where the authors proposed an accelerated variant of SCGD. Both SCGD and its accelerated variant achieved convergence rates strictly worse than SGD for solving non-CO problems. 
Recently, \cite{ghadimi2020single} and \cite{chen2021solving} developed a single time-scale algorithm for solving the CO problem that achieves the same convergence as SGD for solving non-CO problems. Variance-reduced algorithms for solving the CO problems have also been considered in the literature, however, a major drawback of such approaches is the reliance of batch size on the desired solution accuracy \cite{lian2017finite,zhang2019stochastic,hu2019efficient}.  

 \paragraph{Distributed CO.} There have been only a few attempts to solve non-convex CO problems in the FL setting, partially, because of the challenges discussed in Section \ref{sec: Intro}. The first FL algorithm to solve the non-convex CO problem, Compositional Federated Learning (ComFedL), was developed in \cite{huang2021compositional}. ComFedL required accuracy dependent batch sizes that resulted in $\mathcal{O}(\epsilon^{-4})$ convergence which is significantly worse compared to FedAvg to solve standard non-compositional problems \cite{Yu_Zhu_2018parallel}. In \cite{gao2022convergence}, Local Stochastic Compositional Gradient Descent
with Momentum (Local-SCGDM) was proposed which removed the requirement of large batch sizes and achieved an $\mathcal{O}(\epsilon^{-2})$ convergence. However, Local-SCGDM utilized a non-standard momentum-based update from \cite{ghadimi2020single} that does not resemble a simple SGD-based update. Importantly, the CO problem solved by ComFedL \cite{huang2021compositional} and Local-SCGDM \cite{gao2022convergence} is non-standard as the problem is not distributed in the compositional objective (see Remark \ref{Rem: Comparison_Huang}). In contrast, we consider a general setting where the compositional objective is also distributed among multiple nodes. Recently,  \cite{tarzanagh2022fednest} proposed a nested optimization framework, FedNest, to solve bilevel problems in the FL setting. The proposed algorithm achieved SGD rates of $\mathcal{O}(\epsilon^{-2})$ \cite{Ghadimi_Siam_2013_SGD}. Different from the simple SGD-based update rule, FedNest adopted a multi-loop variance reduction-based update.
In \cite{haddadpour2022learning}, the authors proposed a Generalized Composite Incremental Variance Reduction (GCIVR) framework for solving problems of the form (\ref{Eq: Fl_Prob}) in a distributed setting. GICVR achieved a better convergence rate of $\mathcal{O}(\epsilon^{-1.5})$, however, it relied on a double-loop structure and accuracy dependent large batch sizes to achieve variance reduction. Importantly, none of the above works guarantee linear speedup with the number of clients. Moreover, the current algorithms utilize complicated momentum or VR-based update rules that require computation of accuracy-dependent batch sizes \cite{haddadpour2022learning}, and/or consider a simple setting where the compositional objective is not distributed among nodes \cite{huang2021compositional, gao2022convergence}. 

In contrast to all the above works, our work considers a general setting (\ref{Eq: Fl_Prob}), where the goal is to jointly minimize a compositional and a non compositional objective in the FL setting. To solve  (\ref{Eq: Fl_Prob}), we develop \aname~a FedAvg algorithm for CO problems that achieves (i). the same guarantees as FedAvg for minimizing non-CO problems, (ii). linear speed-up with the number of clients, (iii). improved communication complexity, (iv). performance guarantees where the batch sizes required are  independent of the desired solution accuracy, and (v). characterizes the performance as a function of local updates at each client and the data heterogeneity in the inner and outer non-compositional objectives.

 \paragraph{DRO.} DRO has been extensively studied in optimization, machine learning, and statistics literature \cite{ben2013robust,bertsimas2018data, duchi2021statistics,namkoong2017variance,staib2019distributionally}
 Broadly, DRO problem formulation can be divided into two classes, one is a constrained formulation and the other is the regularized formulation (see (\ref{Eq: GeneralDRO})) \cite{levy2020large, duchi2021statistics}. A popular approach to solve the constrained DRO formulation is via primal-dual formulation where algorithms developed for min-max problems can directly be applied to solve constrained DRO \cite{yan2019stochastic,namkoong2017variance,song2021variance,alacaoglu2022complexity,tran2020hybrid}. Many algorithms under different settings, e.g., convex, non-convex losses, and stochastic settings have been considered in the past to address such problems. However, primal-dual algorithms suffer from computational bottlenecks, since they require maintaining and updating the set of dual variables equal to the size of the dataset which can become particularly challenging, especially for large-scale machine learning tasks. Recently, \cite{levy2020large}
\cite{qi2022stochastic}
\cite{haddadpour2022learning} have developed algorithms that are applicable to large-scale stochastic settings. Works  \cite{levy2020large} and
\cite{qi2022stochastic} consider specific formulations of the DRO problem while \cite{haddadpour2022learning} considers a general formulation, however, as pointed out earlier the algorithms developed in \cite{haddadpour2022learning} are double loop and require accuracy dependent batch sizes to guarantee convergence (see Table \ref{tab:table1}). In contrast, in this work, we develop algorithms that solve general instants of CO problems that often arise in
DRO formulation. Importantly, the developed algorithms are amenable to large-scale distributed implementation with algorithmic guarantees independent of accuracy dependent batch sizes.

\subsection{Detailed Comparison with \cite{huang2021compositional,gao2022convergence,tarzanagh2022fednest}}
\label{App: Comparison}

\paragraph{Comparison with \cite{huang2021compositional,gao2022convergence}.} We note that the problem setting in \cite{huang2021compositional} and \cite{gao2022convergence} is significantly different from the one considered in our work. We also would like to point out that the problem formulation considered in our work is more challenging than \cite{huang2021compositional,gao2022convergence} and the algorithms developed for solving the problem in \cite{huang2021compositional,gao2022convergence} cannot solve the problem considered in our work. In the following, we elaborate on the differences between our work and that of \cite{huang2021compositional,gao2022convergence}.

In \cite{huang2021compositional,gao2022convergence}, the authors consider the objective function
  \begin{align}
      \frac{1}{k}\sum_{k = 1}^K f_k(g_k(\cdot)).
      \label{eq: non-dist_CO}
  \end{align}
  Please observe that in this setting the local nodes have access to local composite functions $f_k(g_k(\cdot))$. In contrast, we consider a setting with objective function defined in (\ref{Eq: Fl_Prob}) where the local nodes have access to only $h_k(\cdot)$ and $g_k(\cdot)$\footnote{We would also like to note that the setting considered in the paper can be easily extended to the case where $f (\cdot) = 1/K \sum_{k=1}^K f_k(\cdot)$ without changing the current results.}. Note that the major difference in the two settings in (\ref{eq: non-dist_CO}) and (\ref{Eq: Fl_Prob}) comes from the fact that in (\ref{eq: non-dist_CO}) the inner function $g_k(\cdot)$ is fully available at each node, whereas in (\ref{Eq: Fl_Prob}) the inner function $1/K \sum_{k = 1}^K g_k(\cdot)$ is not available (since each node can only access $g_k(\cdot)$) at the local nodes. Below, we discuss two major consequences of this:
  
    \begin{itemize}[leftmargin=*]
        \item {\bf Practicality:}  We point out that the setting in (\ref{Eq: Fl_Prob}) is more practical as can be seen from the examples presented in Section \ref{Subsec: Examples} wherein the DRO problems take the form of (\ref{Eq: Fl_Prob}) rather than (\ref{eq: non-dist_CO}) in a distributed setting. For illustration, let us consider a simple setting where we have a total of $m$ samples with each node having access to $m_k = m/K$ samples. Then the DRO problem with KL-Divergence problem becomes
  $$\min_{x \in \mathbb{R}^d}   f\bigg(\frac{1}{k} \sum_{k=1}^K g_k(\cdot)\bigg) : =  \log \bigg( \frac{1}{m} \sum_{i = 1}^m \exp\bigg( \frac{\ell_i(x)}{\lambda} \bigg) \bigg), $$ where $f(\cdot)= \log(\cdot)$, $g_k(x) = 1/m_k \sum_{i=1}^{m_k} \exp\big(  {\ell_i(x)}/{\lambda} \big)$, and $g(\cdot) = 1/K \sum_{k = 1}^K g_k(\cdot)$. Note that the above formulation is same as (\ref{Eq: Fl_Prob}) and cannot be formulated using (\ref{eq: non-dist_CO}). To demonstrate this fact we have used the notation in Table \ref{tab:table1} as CO-ND for formulation of (\ref{eq: non-dist_CO} where the inner function $g_k(\cdot)$ can be fully locally accessed by each node whereas our setting is more general with each node having only partial access to the inner-function $g(\cdot)$.  Next, we show why the algorithms developed for \cite{huang2021compositional,gao2022convergence} cannot be utilized to solve the problem considered in our work.  
  
 \item {\bf Challenges in solving (\ref{Eq: Fl_Prob}):} A major contribution of our work is in establishing the fact that the algorithms that are developed for solving \ref{Eq: Fl_Prob}), i.e., the algorithms developed in \cite{huang2021compositional,gao2022convergence}, cannot be utilized to solve the problem considered in our work. 
 
   To demonstrate this consider the simple deterministic setting with $f_k = f$, then the local gradient computed for the objective function in (\ref{Eq: Fl_Prob}) will be $\nabla g_k(x) \nabla f(g_k(x))$ (please see (\ref{eq: Deterministic_Grad}) in the manuscript). Note that this is an unbiased local gradient for objective in (\ref{eq: non-dist_CO}) which further implies that simple FedAVG-based implementations can be developed for solving this problem as done in \cite{huang2021compositional,gao2022convergence}. In contrast, note that the local gradient $\nabla g_k(x) \nabla f(g_k(x))$ will be a biased local gradient for our problem in (\ref{Eq: Fl_Prob}) and will lead to divergence of FedAvg-based algorithms \cite{huang2021compositional,gao2022convergence} as shown in Section \ref{subsec: VanillaFedAvg}. Moreover, note that we establish that even if we share the local functions $g_k(\cdot)$ intermitteltly among nodes we may not be able to mitigate the bias of local gradient and the developed algorithms will again diverge to incorrect solutions. Please see Section \ref{subsec: VanillaFedAvg} for more details. 
    \end{itemize}

\paragraph{Comparison with \cite{tarzanagh2022fednest}.} Next, we note that the algorithm deveoped in \cite{tarzanagh2022fednest} is a bilevel algorithm with multi-loop structure with many tunable (hyper) parameters. Such algorithms are not preferred in practical implementations. In contrast our algorithm is a single-loop algorithm with simple FedAvg-type SGD updates. In addition to being practical, our work also significantly improves upon the theoretical guarantees achieved in \cite{tarzanagh2022fednest} by achieving linear speed-up with the number of clients as well as improved communication complexity which any of the works including \cite{huang2021compositional,gao2022convergence,tarzanagh2022fednest} are unable to achieve.

 \section{Detailed experiment setup and additional experiments}
 \label{App: Add_Exp}

\textbf{Experiment setup.} The models are trained on an NVIDIA GeForce RTX 3090 GPU with 24 GB of memory. All experiments are conducted using the PyTorch framework, specifically Python 3.9.16 and PyTorch 1.8

\textbf{Datasets.} To evaluate the performance of \aname~, the first section of the experiments is conducted on CIFAR10-ST and CIFAR-100-ST datasets for image classification. The second section of the experiments focuses on the Adult dataset, utilizing tabular data classification and emphasizing DRO for fairness constraints. The CIFAR10-ST and CIFAR-100-ST datasets are modified versions of the original CIFAR10 and CIFAR-100 datasets. The modification involves intentionally creating imbalanced training data. Specifically, only the last 100 images are retained for each class in the first half of the classes, while the other classes and the test data remain unchanged. This creates an imbalanced distribution, posing a challenge for machine learning models to effectively handle imbalanced class scenarios. In the Adult dataset, we consider the race groups ``white," ``black," and ``other" as protected groups. We assign the value of $\epsilon$ as 0.05 and set the noise level to 0.3 during training across all the algorithms.

\textbf{Evaluation metrics.} We present the Top-1 accuracies for the training and testing segments of the CIFAR10-ST and CIFAR-100-ST datasets (please see Figures \ref{Fig: CIFAR} and \ref{Fig: CIFAR_I} in Section \ref{Sec: Experiments}). Furthermore, in addition to training and testing performance, we also include the maximum violation values for both the training and testing sections of the Adult dataset. Specifically, the maximum group violation is evaluated following \cite{haddadpour2022learning}. To ensure equal opportunities among different groups, even when group membership is uncertain and fluctuating during training, the objective is to develop a solution that is robust across various protected groups in the problem. We assume that we have access to the probability distribution of the actual group memberships ($P(gi = j | g^i = k)$ where $g^i$ represents the true group membership and $g^i$ represents the noisy group membership). With this information, we aim to enforce fairness constraints by considering all potential proxy groups based on this probability distribution, which can significantly increase the number of constraints. In the case of equal opportunity, our goal is to ensure that the true positive rate $(TPR)$ for each group closely aligns with the $TPR$ of the overall dataset, within a certain threshold $\epsilon$. In other words, we want to achieve $tpr(g = j) \geq tpr(ALL) - \epsilon$ for every proxy group we define.

\begin{figure}[t]
\centering
\includegraphics[width=\textwidth]{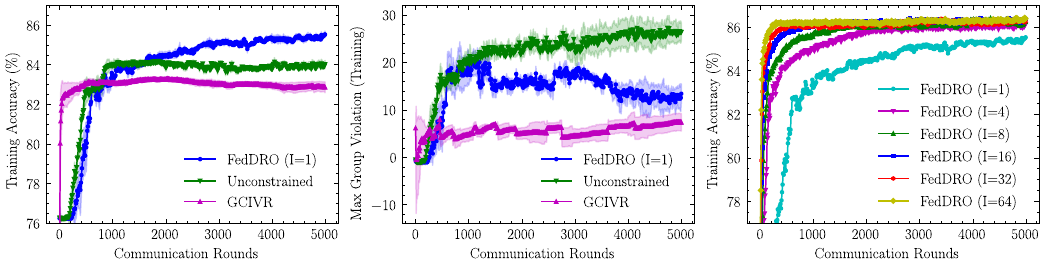}
\caption{Comparison of training accuracies of \aname, GCIVR, and the unconstrained baseline (first two figures). Training performance of \aname~with different $I$ (rightmost figure).}
\label{Fig: Train}
\end{figure}

\begin{figure}[ht]
\centering
\includegraphics[width=0.8\textwidth]{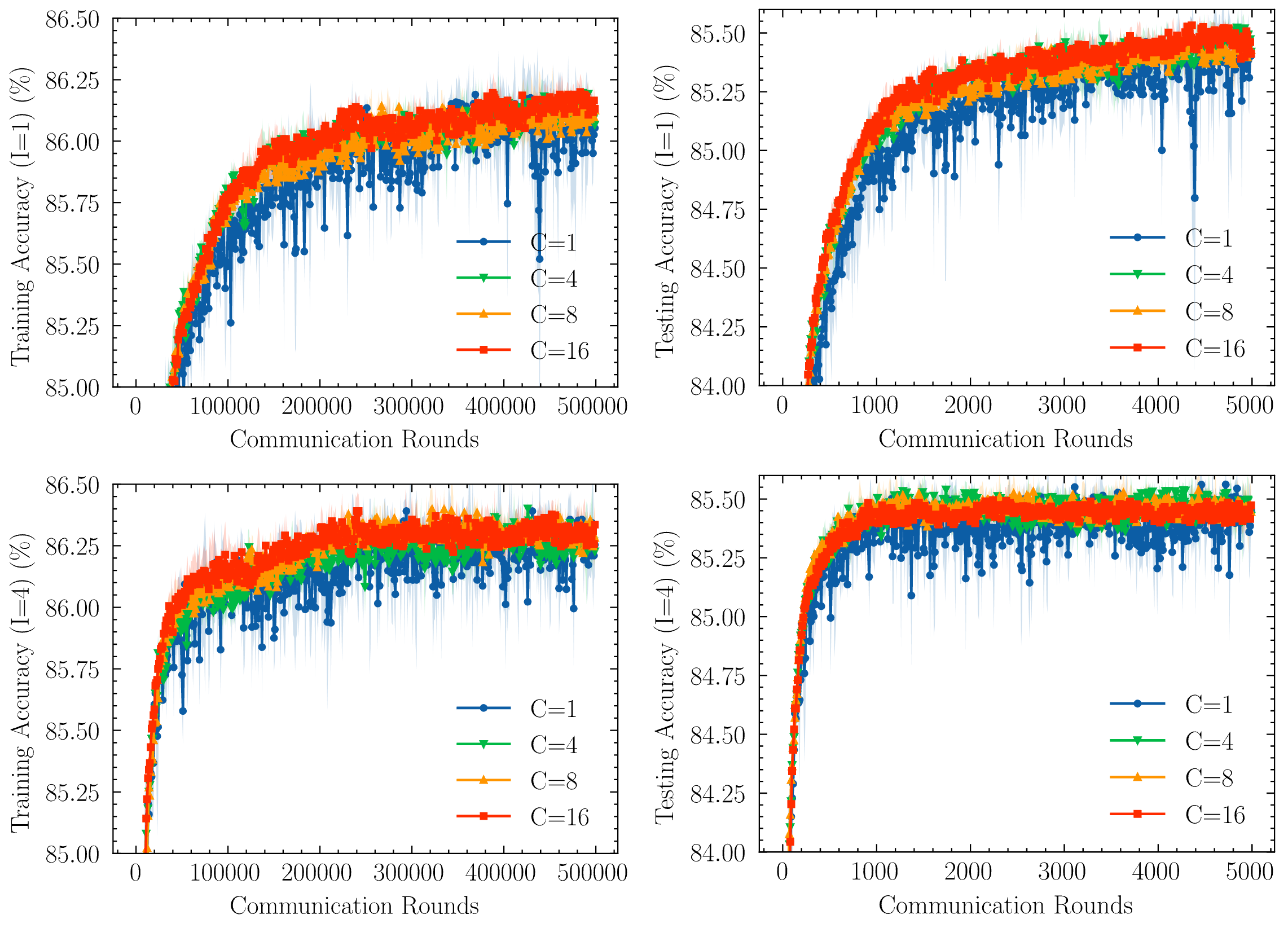}
\caption[short]{Training and testing performance of \aname~with the number of  clients (denoted as $C = 1,2,3$ and $4$ in the figure) and number of local updates, $I = 1$ and $4$.}
\label{Fig: different k}
\end{figure}

{\bf Discussion.} In Figure \ref{Fig: Train}, we evaluate the training performance on the adult dataset under the same conditions as mentioned earlier for testing in Section \ref{Sec: Experiments}. Similar to the previous findings, in the leftmost image, we observe that \aname~ outperforms both the constrained version of GCIVR and unconstrained baseline formulation. 
 Evaluating the maximum group violation, we see the unconstrained optimization demonstrates the poorest performance, while our technique performs comparably to GCIVR, and improves in performance as the communication rounds increase. The right-most plot, confirms that increasing the local updates, i.e., $I$ results in improved performance, aligning with the theoretical guarantees presented in the paper.

In Figure \ref{Fig: different k}, we evaluate the performance of \aname with the number of clients. Specifically, the accuracy demonstrates an upward trend as the value of $C$ (representing the number of clients) increases in the experiments conducted on the adult dataset. The top two plots depict the training and testing performance for $I=1$, while the bottom two demonstrate the training and testing performance with $I=4$.

\section{Useful lemmas}
\label{App: Useful Lemmas}
\begin{lem}
\label{Lem: Sum_vectors}
For vectors $a_1, a_2, \ldots, a_n \in \R^d$, we have
\begin{align*}
    \| a_1 + a_2 + \ldots, + a_n \|^2 \leq n \big[ \|a_1 \|^2 + \|a_2 \|^2 + \ldots, + \|a_n \|^2 \big].
\end{align*}
\end{lem}

\begin{lem}
    \label{Lem: Emp_Var}
    For a sequence of vectors $a_1, a_2, \ldots, a_K \in \R^d$, defining $\bar{a} \coloneqq \frac{1}{K} \sum_{k = 1}^K a_k$, we then have
    \begin{align*}
        \sum_{k = 1}^K \| a_k - \bar{a} \|^2 \leq \sum_{k = 1}^K \| a_k\|^2.
    \end{align*}
\end{lem}

\section{Proof of Theorem \ref{Thm: FedAvg_No}}
\label{App: FedAvg}

We restate Theorem \ref{Thm: FedAvg_No} for convenience. 
\begin{theorem}[Vanilla FedAvg: Non-Convergence for CO]
    There exist functions $f(\cdot)$ and $g_k(\cdot)$ for $k \in [K]$ satisfying Assumptions \ref{Ass: Lip}, \ref{Ass: BoundedVar}, and \ref{Ass: BoundedHetero}, and an initialization strategy such that for a fixed number of local updates $I > 1$, and for any $0< \eta^t  < C_\eta$ for $t \in \{0, 1, \ldots, T - 1\}$ where $C_\eta > 0$ is a constant, the iterates generated by Algorithm \ref{Algo: FedAvg} under both Cases I and II do not converge to the stationary point of $\Phi(\cdot)$, where $\Phi(\cdot)$ is defined in (\ref{Eq: Fl_Prob}) with $h(x) = 0$.
    \label{Thm: App_FedAvg_No}
\end{theorem}
\begin{proof}
    We consider a setting where we have $K = 2$ nodes in the network. Also, let us consider a single-dimensional setting where the local functions $g_k: \mathbb{R} \to \mathbb{R}$ for $k = \{1,2\}$ at each node are 
    \begin{align*}
        g_1 (x) \coloneqq 4x - 4 \quad \text{and} \quad g_2 (x) \coloneqq - 2 x + 4.
    \end{align*}
    Moreover, assume $f: \mathbb{R} \to \mathbb{R}$ as $f(y) \coloneqq \sqrt{{y^2} + 4}$. Therefore, the CO problem becomes   
    \begin{align}
   \min_{x \in \mathbb{R}} \Bigg\{ \Phi(x) \coloneqq     f\bigg( \frac{1}{2} \Big(  g_1(x) + g_2(x)\Big)  \bigg) \coloneqq \sqrt{   \Bigg[ \frac{1}{2} \Big(  g_1(x) + g_2(x)\Big) \Bigg]^2 + 4 } = \sqrt{ {x^2}  + 4} \Bigg\}.
   \label{Eq: App_Example}
    \end{align}
First, we establish that the functions $f(\cdot)$ and $g_k(\cdot)$ for $k \in [K]$ satisfy Assumptions \ref{Ass: Lip}, \ref{Ass: BoundedVar}, and \ref{Ass: BoundedHetero}.

{\bf Claim:} Functions $f$, $g_1$ and $g_2$ satisfy Assumptions \ref{Ass: Lip}, \ref{Ass: BoundedVar}, and \ref{Ass: BoundedHetero}. 

The above claim is straightforward to verify. Specifically, we have
\begin{itemize}[leftmargin=*]
    \item[--] The functions  $f$, $g_1$ and $g_2$ are differentiable and Lipschitz smooth. 
    \item[--] The function $f(\cdot)$ is Lipschitz. Moreover, $g_k(\cdot)$'s are deterministic functions implying mean-squared Lipschitzness. 
    \item[--] Assumption \ref{Ass: BoundedVar} is automatically satisfied since $g_k(\cdot)$'s are deterministic functions. 
    \item[--] Bounded heterogeneity of $g_k(\cdot)$'s is satisfied.
\end{itemize}
Note that it is clear from (\ref{Eq: App_Example}) that the minimizer of $\Phi(\cdot)$ is $x^\ast = 0$. In the following, we will show that Algorithm \ref{Algo: FedAvg} is not suitable to solve such problems by establishing that there exists an initialization strategy and choice of step-sizes in the range $0 < \eta < C_\eta$ where $C_\eta > 0$ is a constant, the iterates generated by Algorithm \ref{Algo: FedAvg} under both Cases I and II fail to converge to $x^\ast$. Next, we prove the statement of the theorem in two parts. In the first part, we tackle Case I of Algorithm \ref{Algo: FedAvg} while in the second part, we prove Case II of Algorithm \ref{Algo: FedAvg}. Next, we consider Case I.

{\bf Case I:} Let us first compute the local gradients at each agent. We have 
\begin{align*}
    \nabla \Phi_1(x)    & = \nabla g_1(x) \nabla f(y_1) = 4 \frac{ y_1}{\sqrt{y_1^2 + 4}} \\
      \nabla \Phi_2(x)   & = \nabla g_2(x) \nabla f(y_2) = - 2 \frac{ y_2}{\sqrt{y_2^2 + 4}}
\end{align*}
To prove the results, we consider a simple setting with $I = 2$, i.e., each node conducts $2$ local updates and shares the model parameters with the server. Moreover, we initialize the local iterates to be $x_k^0 = \bar{x}^0 = 0.5$ for $k = \{ 1,2\}$ at both nodes. For this setting, let us write the update rule for Algorithm 1 in Case I. 

\begin{enumerate}[leftmargin=*]
    \item Note that for every $t$ such that $t ~\text{mod} ~2 = 0$, the local update at each node will be: 
    \begin{align*}
       x^{t+1}_1 & = \bar{x}^{t} - 4 \eta \frac{4 \bar{x}^{t} - 4}{\sqrt{(4 \bar{x}^{t} - 4)^2 + 4}}  \\
 x^{t+1}_2 & = \bar{x}^{t} + 2 \eta \frac{-2 \bar{x}^{t} + 4}{\sqrt{(-2 \bar{x}^{t} + 4)^2 + 4}},
    \end{align*}
    \item Moreover, the next immediate update at each node will be 
      \begin{align*}
       x^{t+2}_1 & = x^{t+1}_1 - 4 \eta \frac{4 x^{t+1}_1 - 4}{\sqrt{(4 x^{t+1}_1 - 4)^2 + 4}} \\
 x^{t+2}_2 & = x^{t+1}_2  + 2 \eta \frac{-2 x^{t+1}_2 + 4}{\sqrt{(-2 x^{t+1}_2 + 4)^2 + 4}},
    \end{align*}
    \item This process keeps repeating for $T$ iterations. 
\end{enumerate}

Let us focus on the local functions $f(g_1(x))$ and $f(g_2(x))$. Note from the definition of $g_1(\cdot)$, $g_2(\cdot)$ and $f(\cdot)$ that the local optimum of these functions will be $x_1^\ast = 1$ and $x_2^\ast = 2$, respectively. Consequently, for appropriately chosen step-size $\eta$ in each iteration $x_1^{t+1}$ and $x_1^{t+2}$ at node 1 will converge towards $x^\ast_1 = 1$ and similarly, $x_2^{t+1}$ and $x_2^{t+2}$ at node 2 will converge towards $x^\ast_1 = 2$. This implies that we can expect the sequence $\bar{x}^{t}$ for each $t \in [T]$ to not converge to $x^\ast = 0$, the minimizer of the CO problem defined in (\ref{Eq: App_Example}). Let us present this argument formally.

{\bf Claim:} For $C_\eta = 1/8$ such that we have $0 < \eta < C_\eta$, and utilizing the initialization $\bar{x}^0 = 0.5$, we have $\bar{x}^t \geq 0.5$ for every $t > 0$ with $t ~\text{mod}~ 2 = 0$. 

This above Claim directly proves the statement of Theorem \ref{Thm: FedAvg_No} for Case I. Let us now prove the claim formally. We utilize induction to prove the claim.

{\em Proof of claim:} First, note that the claim is automatically satisfied for $t = 0$ as a consequence of the initialization strategy. Assuming the claim holds for some $t \in [T]$ with $t ~\text{mod}~ 2 = 0$, i.e., we have $\bar{x}_t \geq 0.5$ for some $t \in [T]$ with $t ~\text{mod}~ 2 = 0$, we need to show that $\bar{x}_{t+2} \geq 0.5$. 

In the following, we consider the following three cases: (1) $0.5 \leq \bar{x}_t < 1$, (2) $1 \leq \bar{x}_t < 2$, and (3) $\bar{x}_t \geq 2$. Here, we present the proof for case (1), the rest of the cases follow in a similar manner. 

\begin{itemize}[leftmargin=*]
    \item  Note from Step 1 above that since $0.5 \leq \bar{x}^t < 1$, we have $4 \bar{x}^t - 4 < 0$ and $-2 \bar{x}^t + 4 > 0$, which further implies that the locally updated iterates $x^{t+1}_1 > \bar{x}^t \geq 0.5$ and $x^{t+1}_2 > \bar{x}^t \geq 0.5$. Next, let us analyze the iterates at $t+2$.
    \item   At node 1, we further consider two cases, when $x_1^{t+1} < 1$ and the other when  $x_1^{t+1} \geq 1$.
    \begin{itemize}[leftmargin=*]
        \item First, note that if $x_1^{t+1} < 1$ we will have $4 x^{t+1}_1 - 4 < 0$ in Step 2 above implying $x^{t+2}_1 > x^{t+1}_1 > \bar{x}^t \geq 0.5$.
        \item Otherwise, if $x_1^{t+1} \geq 1$, we have $4 x^{t+1}_1 - 4 \geq  0$ however in this case we have 

$$\Bigg|  4 \eta \frac{4 x^{t+1}_1 - 4}{\sqrt{(4 x^{t+1}_1 - 4)^2 + 4}} \Bigg| \leq 1/2 ~~\text{for}~~\eta \leq \frac{1}{8},$$
again implying from the update rule in Step 2 that 
$$x_1^{t+2} \geq x_1^{t+1} - \frac{1}{2} \geq 0.5,$$
where the last step follows from the fact that $x_1^{t+1} \geq 1$.
Therefore, we have established that $x_1^{t+2} \geq 0.5$.
    \end{itemize}
    \item At node 2, it is easy to establish that for case (1) with $0.5 \leq \bar{x_t} < 1$, we will have $0.5 \leq  x_2^{t+1} \leq 1.5$. Note from the update rule in Step 2 that for this $x_2^{t+1}$, we have $-2 x^{t+1}_2 + 4 > 0$ which further implies that $x_2^{t+2} > x_2^{t+1} \geq 0.5$.
    \item Finally, we have established that both $x_1^{t + 2} \geq 0.5$ and $x_2^{t + 2} \geq 0.5$, implying $\bar{x}_{t+2} \geq 0.5$. This completes the proof of Case (1). Note that the proof for the other cases follows in a very similar straightforward manner. 
\end{itemize}
Therefore, we have the proof of Case I in Algorithm \ref{Algo: FedAvg}. Next, we consider Case II where in addition to the model parameters, the local embeddings $g_k(\cdot)$ for $k \in [K]$ are also shared intermittently among nodes. Please see Case II in Algorithm \ref{Algo: FedAvg}.

{\bf Case II:} Let us consider the same setting as in Case I. Specifically, we consider a simple setting with $I = 2$, i.e., each node conducts $2$ local updates and shares the model parameters with the server. Moreover, we initialize the model parameters $x_k^0 = \bar{x}^0 = 0.5$ for $k = \{ 1,2\}$ at both nodes. Note that this implies from the definition of $g_1(\cdot)$ and $g_2(\cdot)$ that $y_k^0 = \bar{y}^0 = 0.5$ for $k = \{ 1,2\}$. For this setting, let us write the update rule for Algorithm 1. 

\begin{enumerate}[leftmargin=*]
    \item Note that for every $t$ such that $t ~\text{mod} ~2 = 0$, the local update at each node will be: 
    \begin{align*}
       x^{t+1}_1 & = \bar{x}^{t} - 4 \eta \frac{\bar{x}^{t}}{\sqrt{(\bar{x}^{t})^2 + 4}} \\
 x^{t+1}_2 & = \bar{x}^{t} + 2 \eta \frac{\bar{x}^{t}}{\sqrt{(\bar{x}^{t})^2 + 4}},
    \end{align*}
    \item Moreover, the next immediate update at each node will be 
      \begin{align*}
       x^{t+2}_1 & = x^{t+1}_1 - 4 \eta \frac{4 x^{t+1}_1 - 4}{\sqrt{(4 x^{t+1}_1 - 4)^2 + 4}} \\
 x^{t+2}_2 & = x^{t+1}_2  + 2 \eta \frac{-2 x^{t+1}_2 + 4}{\sqrt{(-2 x^{t+1}_2 + 4)^2 + 4}},
    \end{align*}
       \item This process keeps repeating for $T$ iterations. 
\end{enumerate}
We point out that this setting is considerably challenging compared to Case I since a cursory look at the algorithm may suggest that sharing the embeddings $g_k(\cdot)$ for $k \in [K]$ intermittently may help mitigate the bias in the gradient estimates. However, this is not the case as we show next. 

{\bf Claim:} For $C_\eta = {1/22}$ such that we have $0 < \eta < C_\eta$, and utilizing the initialization $\bar{x}^0 = 0.5$, we have $\bar{x}^t \geq 0.5$ for every $t > 0$ with $t ~\text{mod}~ 2 = 0$.

We note that for this case the intuition is not as straightforward as in the previous case. We again prove the claim by induction. 

{\em Proof of claim:} First, note that the claim is automatically satisfied for $t = 0$ as a consequence of the initialization strategy. Assuming the claim holds for some $t \in [T]$ with $t ~\text{mod}~ 2 = 0$, i.e., we have $\bar{x}_t \geq 0.5$ for some $t \in [T]$ with $t ~\text{mod}~ 2 = 0$, we need to show that $\bar{x}_{t+2} \geq 0.5$. 

Let us first construct $x_1^{t+2}$ and $x_2^{t+2}$ as a function of $\bar{x}^t$. To this end, we have from the update rule in Steps 1 and 2 that
 \begin{align*}
       x^{t+2}_1 & = \bar{x}^{t} \big( 1 -  \epsilon_1^t \big) - 4 \eta \frac{4 \bar{x}^{t} \big( 1 -  \epsilon_1^t \big) - 4}{\sqrt{(4 \bar{x}^{t} \big( 1 -  \epsilon_1^t \big) - 4)^2 + 4}} \\
 x^{t+2}_2 & = \bar{x}^{t} \big( 1 +  \epsilon_2^t \big)  + 2 \eta \frac{-2  \bar{x}^{t} \big( 1 +  \epsilon_2^t \big) + 4}{\sqrt{(-2  \bar{x}^{t} \big( 1 +  \epsilon_2^t \big) + 4)^2 + 4}},
    \end{align*}
 where we have defined $\epsilon_1^t \coloneqq \frac{4 \eta }{\sqrt{(\bar{x}^t)^2 + 4}}$ and   $\epsilon_2^t \coloneqq \frac{2 \eta }{\sqrt{(\bar{x}^t)^2 + 4}}$, therefore, we have $\epsilon_1^t = 2 \epsilon_2^t$. Using the above we can evaluate $\bar{x}^{t+2}$ as
 \begin{align*}
  \bar{x}^{t+2} & = \frac{1}{2} \big( x_1^{t+2} +  x_2^{t+2}\big) \\
  &  =  \bigg( \frac{2 - \epsilon_1^t + \epsilon_2^t}{2} \bigg)  \bar{x}^t + 2 \eta \frac{4 - 4 \bar{x}^{t} \big( 1 -  \epsilon_1^t \big)  }{\sqrt{(4 \bar{x}^{t} \big( 1 -  \epsilon_1^t \big) - 4)^2 + 4}} + \eta \frac{4 -2  \bar{x}^{t} \big( 1 +  \epsilon_2^t \big)  }{\sqrt{(-2  \bar{x}^{t} \big( 1 +  \epsilon_2^t \big) + 4)^2 + 4}} \\
  & = \bigg( 1 - \frac{\epsilon_2^t}{2} \bigg) \bar{x}^t + 2 \eta \frac{4 - 4 \bar{x}^{t} \big( 1 -  \epsilon_1^t \big)  }{\sqrt{(4 \bar{x}^{t} \big( 1 -  \epsilon_1^t \big) - 4)^2 + 4}} + \eta \frac{4 -2  \bar{x}^{t} \big( 1 +  \epsilon_2^t \big)  }{\sqrt{(-2  \bar{x}^{t} \big( 1 +  \epsilon_2^t \big) + 4)^2 + 4}},
 \end{align*}
where in the first term of the last equality, we have used the fact that $\epsilon_1^t = 2 \epsilon_2^t$. Recall from the induction hypothesis that we have $\bar{x}^t \geq 0.5$, and we need to show that $\bar{x}^{t+2} \geq 0.5$. Note from above that to establish $\bar{x}^{t+2} \geq 0.5$, it suffices to show that 
 \begin{align}
 \label{Eq: Sufficient_Cond_Case2_Prelim}
 \bar{x}^t - 0.5 + 2 \eta \frac{4 - 4 \bar{x}^{t} \big( 1 -  \epsilon_1^t \big)  }{\sqrt{(4 \bar{x}^{t} \big( 1 -  \epsilon_1^t \big) - 4)^2 + 4}} + \eta \frac{4 -2  \bar{x}^{t} \big( 1 +  \epsilon_2^t \big)  }{\sqrt{(-2  \bar{x}^{t} \big( 1 +  \epsilon_2^t \big) + 4)^2 + 4}} \geq   \frac{\epsilon_2^t}{2}   \bar{x}^t.
 \end{align}
  From the definition of $\epsilon_2^t \coloneqq \frac{2 \eta }{\sqrt{(\bar{x}^t)^2 + 4}}$, we note that the r.h.s. term can be further upper bounded as $$\frac{\epsilon_2^t}{2}   ~\bar{x}^t = \eta 
 ~\frac{ \bar{x}^t}{\sqrt{(\bar{x}^t)^2 + 4}} \leq \eta.$$
 Therefore, to establish to establish $\bar{x}^{t+2} \geq 0.5$, it suffices to show that 
 \begin{align}
 \bar{x}^t - 0.5 + 2 \eta \frac{4 - 4 \bar{x}^{t} \big( 1 -  2 \epsilon_2^t \big)  }{\sqrt{(4 \bar{x}^{t} \big( 1 -  2 \epsilon_2^t \big) - 4)^2 + 4}} + \eta \frac{4 -2  \bar{x}^{t} \big( 1 +  \epsilon_2^t \big)  }{\sqrt{(-2  \bar{x}^{t} \big( 1 +  \epsilon_2^t \big) + 4)^2 + 4}} \geq   \eta,
 \label{Eq: Sufficient_Cond_Case2}
 \end{align}
 where we have replaced $\epsilon_1^t = 2 \epsilon_2^t$. Similar to the previous proof here we again consider three cases as listed below
 \begin{itemize}[leftmargin=*]
     \item Case (1): $\frac{4 - 4 \bar{x}^{t}  ( 1 -  2 \epsilon_2^t  )  }{\sqrt{(4 \bar{x}^{t}  ( 1 -  2 \epsilon_2^t  ) - 4)^2 + 4}} < 0$ and $\frac{4 -2  \bar{x}^{t}  ( 1 +  \epsilon_2^t  )  }{\sqrt{(-2  \bar{x}^{t}  ( 1 +  \epsilon_2^t  ) + 4)^2 + 4}} < 0$
      \item Case (2): $\frac{4 - 4 \bar{x}^{t}  ( 1 -  2 \epsilon_2^t  )  }{\sqrt{(4 \bar{x}^{t}  ( 1 -  2 \epsilon_2^t  ) - 4)^2 + 4}} < 0$ and $\frac{4 -2  \bar{x}^{t} ( 1 +  \epsilon_2^t  )  }{\sqrt{(-2  \bar{x}^{t} ( 1 +  \epsilon_2^t  ) + 4)^2 + 4}} > 0$
       \item Case (3): $\frac{4 - 4 \bar{x}^{t}  ( 1 -  2 \epsilon_2^t )  }{\sqrt{(4 \bar{x}^{t}  ( 1 -  2 \epsilon_2^t  ) - 4)^2 + 4}} \geq 0$ and $\frac{4 -2  \bar{x}^{t}  ( 1 +  \epsilon_2^t  )  }{\sqrt{(-2  \bar{x}^{t}  ( 1 +  \epsilon_2^t ) + 4)^2 + 4}} \geq 0$
 \end{itemize}
 We first consider Case (1). Note that Case (1) implies that $\bar{x}^t > 1$, and using the fact that $\frac{4 - 4 \bar{x}^{t}  ( 1 -  2 \epsilon_2^t  )  }{\sqrt{(4 \bar{x}^{t}  ( 1 -  2 \epsilon_2^t  ) - 4)^2 + 4}}\geq -1$ and $\frac{4 -2  \bar{x}^{t}  ( 1 +  \epsilon_2^t  )  }{\sqrt{(-2  \bar{x}^{t}  ( 1 +  \epsilon_2^t  ) + 4)^2 + 4}} \geq - 1$, we get
 \begin{align*}
     &  \bar{x}^t - 0.5 + 2 \eta \frac{4 - 4 \bar{x}^{t} \big( 1 -  2 \epsilon_2^t \big)  }{\sqrt{(4 \bar{x}^{t} \big( 1 -  2 \epsilon_2^t \big) - 4)^2 + 4}} + \eta \frac{4 -2  \bar{x}^{t} \big( 1 +  \epsilon_2^t \big)  }{\sqrt{(-2  \bar{x}^{t} \big( 1 +  \epsilon_2^t \big) + 4)^2 + 4}}   \geq 0.5 - 3 \eta  
 \end{align*}
 Note that by choosing $\eta \leq 1/8$, the sufficient condition in (\ref{Eq: Sufficient_Cond_Case2}) is satisfied, which further implies that under Case (1), we have $\bar{x}^{t+2} \geq 0.5$. Next, we consider Case (2).

 Note that for Case (2) we have $2/(1 + \epsilon_2^t) > \bar{x}^t > 1$, next using the fact that $\frac{4 - 4 \bar{x}^{t}  ( 1 -  2 \epsilon_2^t  )  }{\sqrt{(4 \bar{x}^{t}  ( 1 -  2 \epsilon_2^t  ) - 4)^2 + 4}}\geq -1$ and $\frac{4 -2  \bar{x}^{t}  ( 1 +  \epsilon_2^t  )  }{\sqrt{(-2  \bar{x}^{t}  ( 1 +  \epsilon_2^t  ) + 4)^2 + 4}} \geq 0$, we get

\begin{align*}
     &  \bar{x}^t - 0.5 + 2 \eta \frac{4 - 4 \bar{x}^{t} \big( 1 -  2 \epsilon_2^t \big)  }{\sqrt{(4 \bar{x}^{t} \big( 1 -  2 \epsilon_2^t \big) - 4)^2 + 4}} + \eta \frac{4 -2  \bar{x}^{t} \big( 1 +  \epsilon_2^t \big)  }{\sqrt{(-2  \bar{x}^{t} \big( 1 +  \epsilon_2^t \big) + 4)^2 + 4}}   \geq 0.5 - 2 \eta  
 \end{align*}
 Again choosing $\eta \leq 1/8$, the sufficient condition in (\ref{Eq: Sufficient_Cond_Case2}) is satisfied, which further implies that under Case (2), we have $\bar{x}^{t+2} \geq 0.5$.

 Finally, we consider the most challenging Case (3). Note that in Case (3) we have $0.5 \leq \bar{x}^t \leq 1/(1 - 2 \epsilon_2^t)$. For this case, we revisit the sufficient condition in (\ref{Eq: Sufficient_Cond_Case2_Prelim}) and make it tight. Recall that we had from (\ref{Eq: Sufficient_Cond_Case2_Prelim}) that
 \begin{align*}
 \bar{x}^t - 0.5 + 2 \eta \frac{4 - 4 \bar{x}^{t} \big( 1 -  2 \epsilon_2^t \big)  }{\sqrt{(4 \bar{x}^{t} \big( 1 -  2\epsilon_2^t \big) - 4)^2 + 4}} + \eta \frac{4 -2  \bar{x}^{t} \big( 1 +  \epsilon_2^t \big)  }{\sqrt{(-2  \bar{x}^{t} \big( 1 +  \epsilon_2^t \big) + 4)^2 + 4}} & \geq   \eta \frac{\bar{x_t}}{\sqrt{(\bar{x}_t)^2 + 4}},
 \end{align*} 
now using the fact that for Case (3), we have $0.5 \leq \bar{x}^t \leq 1/(1 - 2 \epsilon_2^t)$, we can restate the sufficient condition as
\begin{align}
\label{Eq: Sufficient_Cond_Mod}
 \bar{x}^t - 0.5 + 2 \eta \frac{4 - 4 \bar{x}^{t} \big( 1 -  2 \epsilon_2^t \big)  }{\sqrt{(4 \bar{x}^{t} \big( 1 -  2 \epsilon_2^t \big) - 4)^2 + 4}} + \eta \frac{4 -2  \bar{x}^{t} \big( 1 +  \epsilon_2^t \big)  }{\sqrt{(-2  \bar{x}^{t} \big( 1 +  \epsilon_2^t \big) + 4)^2 + 4}} & \geq  \frac{\eta}{2}  ,
 \end{align}
where we have used the fact that $0.5 \leq \bar{x}^t \leq 1.1$ for $\eta < 1/22$ and the fact that the term $\eta \frac{\bar{x_t}}{\sqrt{(\bar{x}_t)^2 + 4}}  > \frac{\eta}{2}$ for $0.5 \leq \bar{x}^t \leq 1.1$.  Moreover, $\eta < 1/22$ ensures that $1 + \epsilon_2^t \leq 23/22$. Next, using the fact that $\frac{4 - 4 \bar{x}^{t} \big( 1 -  2 \epsilon_2^t \big)  }{\sqrt{(4 \bar{x}^{t} \big( 1 -  2 \epsilon_2^t \big) - 4)^2 + 4}} > 0$ and
 $$\frac{4 -2  \bar{x}^{t} \big( 1 +  \epsilon_2^t \big)  }{\sqrt{(-2  \bar{x}^{t} \big( 1 +  \epsilon_2^t \big) + 4)^2 + 4}} \geq \frac{4 -2  \bar{x}^{t} \big( 23/22\big)  }{\sqrt{(-2  \bar{x}^{t} (1 + \epsilon_2^t) + 4)^2 + 4}} \geq \frac{6}{10},$$
 Substituting in the l.h.s. of the sufficient condition stated in (\ref{Eq: Sufficient_Cond_Mod}), we get 
 \begin{align*}
     \bar{x}^t - 0.5 + 2 \eta \frac{4 - 4 \bar{x}^{t} \big( 1 -  2 \epsilon_2^t \big)  }{\sqrt{(4 \bar{x}^{t} \big( 1 -  2 \epsilon_2^t \big) - 4)^2 + 4}} + \eta \frac{4 -2  \bar{x}^{t} \big( 1 +  \epsilon_2^t \big)  }{\sqrt{(-2  \bar{x}^{t} \big( 1 +  \epsilon_2^t \big) + 4)^2 + 4}}   \geq \frac{6 \eta}{10},  
 \end{align*}
where we used that fact that $\bar{x}^t \geq 0.5$. Note that $\frac{6 \eta}{10} > \frac{\eta}{2}$, therefore, the sufficient condition stated in (\ref{Eq: Sufficient_Cond_Mod}) is satisfied. This further implies that the $\bar{x}^{t+2} \geq 0.5$ during the execution of the algorithm.

 Recall that the optimal solution for solving the CO problem is $x^\ast = 0$. This means Algorithm \ref{Algo: FedAvg} under both Case I and II fails to converge to the stationary solution. 

 Hence, the theorem is proved.
\end{proof}
Finally, we corroborate the result presented in Theorem \ref{Thm: App_FedAvg_No} via numerical experiment for solving (\ref{Eq: App_Example}) using Case II of Algorithm \ref{Algo: FedAvg}. In Figure \ref{Fig: LB}, we plot the evolution of $\bar{x}^t$ in each communication round. We note that $\bar{x}^t$ is lower bounded by $0.5$ as established in the proof of Theorem \ref{Thm: FedAvg_Yes} above. In fact, note that for all the settings as the communication rounds increase, $\bar{x}^t$ eventually converges to a quantity that is greater than $1$. However, as discussed for the example considered to establish the proof of Theorem \ref{Thm: FedAvg_No}, we know that the true optimizer of the CO problem (\ref{Eq: App_Example}) is $x^\ast = 0$. 

\begin{figure}[t]
\centering
\includegraphics[width=0.5\textwidth]{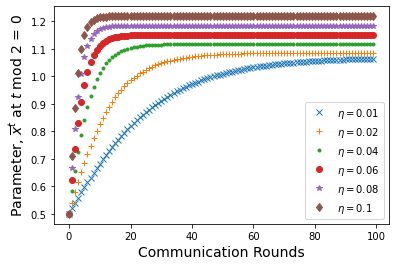}
\caption[short]{The evolution of parameter $\bar{x}^t$ at each communication round for different choices of step-sizes $\eta$. }
\label{Fig: LB}
\end{figure}

\section{Proof of Theorem \ref{Thm: FedAvg_Yes}}

\begin{theorem}[Modified FedAvg: Convergence for CO]
   Suppose we modify Algorithm \ref{Algo: FedAvg} such that $y_k^t = \bar{y}^t$ is updated at each iteration $t \in \{0, 1, \ldots, T - 1 \}$ instead of $[t+1~ \text{\em mod}~ I]$ iterations as in current version of Algorithm \ref{Algo: FedAvg}. Then if functions $f(\cdot)$ and $g_k(x)$ for $k \in [K]$ satisfy Assumptions \ref{Ass: Lip}, \ref{Ass: BoundedVar}, and \ref{Ass: BoundedHetero} such that for a fixed number of local updates $1 \leq I \leq \mathcal{O}(T^{1/4})$, there exists a choice of $\eta^t > 0$ for $t \in \{0, 1, \ldots, T - 1\}$ such that the iterates generated by (modified) Algorithm \ref{Algo: FedAvg} converge to the stationary point of $\Phi(\cdot)$, where $\Phi(\cdot)$ is defined in (\ref{Eq: Fl_Prob}) with $h(x) = 0$.
    \label{Thm: FedAvg_Yes_App}
\end{theorem}

\begin{proof}
     Theorem \ref{Thm: FedAvg_Yes_App} is a direct consequence of Theorem \ref{Thm: FL}. Therefore, we next prove the main result of the paper in Theorem \ref{Thm: FL}.
\end{proof}

\section{Proof of Theorem \ref{Thm: FL}}
\label{App: FL}
For the purpose of this proof, we define the filtration $\mathcal{F}^t$ as the sigma-algebra generated by the iterates $x_k^1, x_k^1, \ldots, x_k^t$ as
\begin{align*}
   \mathcal{F}^t \coloneqq \sigma(x_k^1, x_k^1, \ldots, x_k^t,~ \text{for all}~k \in [K]). 
\end{align*}
Moreover, we define the following. Assuming the total training rounds,
$T -1$, to be a multiple of $I$, i.e., $T - 1 = S \times I$ for some $S \in \mathbb{N}$, we define $t_s \coloneqq s \times I$ with $s \in \{ 0, 1, \ldots, S \}$ as the training rounds where the potentially high-dimensional model parameters, $x_k^t$, are shared among the clients. 
Next, we state Theorem \ref{Thm: FL} again and present the detailed proof of the result. 

\begin{theorem}
\label{Thm: FL_App}
  Under Assumptions \ref{Ass: Lip}, \ref{Ass: BoundedVar}, and \ref{Ass: BoundedHetero} and with the choice of step-size $\eta^t = \eta = \sqrt{\frac{|b|K}{T}}$ for all $t \in \{0, 1, \ldots, T - 1 \}$. Moreover, choosing the momentum parameter $\beta^t = \beta = c_\beta \eta $ where $c_\beta = 4 B_g^4 L_f^2$. Then for

\begin{align*}
T \geq T_{\text{th}} \coloneqq & \max \bigg\{ \frac{4 (L_\Phi |b| K + 8 B_g^2)^2}{|b|K} , ~\frac{B_g^4(96 L_h^2 + 96 B_f^2 L_g^2)^2}{|b|K (L_h^2 + 2 B_f^2 L_g^2 + 4 B_g^4L_f^2)^2}, \\
& \qquad \qquad \qquad \qquad \qquad \qquad \qquad \qquad \qquad \big( 216 L_h^2 + 216 B_f^2 L_g^2 \big) I^2 |b| K \bigg\}
\end{align*}
The iterates generated by Algorithm \ref{Algo: FL} satisfy
    \begin{align*}
      \Ebb \big\|\nabla \Phi(\bar{x}^{a(T)}) \big\|^2   & \leq \frac{2 \big[\Phi(\bar{x}^0) - \Phi(x^\ast) + \big\| \bar{y}^{0} - g(\bar{x}^0) \big\|^2 \big]}{\sqrt{|b| K T}} +  \frac{K (I - 1)^2}{T}  \Big[ {2 \bar{L}_{f,g}} \sigma_h^2 + {2 B_f^2 \bar{L}_{f,g}}  \sigma_g^2 \Big] \nonumber\\
&   \qquad   \qquad        +  \frac{1}{\sqrt{|b| K T}} \bigg[ \big( {4 L_{\Phi} + 8B_g^2} \big)   \sigma_h^2 +  \big({4 L_\Phi B_f^2 + 4c_\beta^2+ 8 B_f^2 B_g^2} \big)   \sigma_g^2 \bigg] \nonumber \\
 &     +  \frac{|b|K (I - 1)^2}{T} \Big[  {6\bar{L}_{f,g}}  \Delta_h^2 + {6B_f^2 \bar{L}_{f,g}}   \Delta_g^2 \Big]  +  \frac{1}{\sqrt{|b| K T}} \bigg[  {96 B_g^2} ~ \Delta_h^2 +  {96 B_f^2 B_g^2}  ~\Delta_g^2 \bigg].
    \end{align*}
\end{theorem}

 \begin{cor}
 Under the same setting as Theorem \ref{Thm: FL}, for the choice of local updates $I = T^{1/4} / (|b| K)^{3/4}$, the iterates generated by Algorithm \ref{Algo: FL} satisfy
  \begin{align}
\Ebb \big\|\nabla\Phi(\bar{x}^{a(T)}) \big\|^2   & \leq \frac{2 \big[\Phi(\bar{x}^0) - \Phi(x^\ast) + \big\| \bar{y}^{0} - g(\bar{x}^0) \big\|^2 \big]}{\sqrt{|b| K T}} +   \frac{C_{\sigma_h}}{\sqrt{|b| K T}} \sigma_h^2  +  \frac{C_{\sigma_g}}{\sqrt{|b| K T}} \sigma_g^2 \nonumber \\
& \qquad \qquad \qquad \qquad \qquad \qquad  \qquad + \frac{C_{\Delta_h}}{\sqrt{|b| K T}} \Delta_h^2  +  \frac{C_{\Delta_g}}{\sqrt{|b| K T}} \Delta_g^2 .    \end{align}
where the constants $C_{\sigma_h}$, $C_{\sigma_g}$, $C_{\Delta_h}$, and $C_{\Delta_g}$ are constants dependent on $L_g$, $L_h$, $L_f$, $B_g$, and $B_f$.
 \end{cor}
We prove the Theorem in multiple steps with the help of several intermediate Lemmas. 

\begin{lem}[{\bf Descent in Function Value}] 
\label{Lem: Descent_Phi}
Under Assumptions \ref{Ass: Lip}-\ref{Ass: BoundedHetero}, the iterates generated by Algorithm \ref{Algo: FL} satisfy
    \begin{align*}
   &  \Ebb \big[  \Phi(\bar{x}^{t + 1}) - \Phi(\bar{x}^{t })  \big]  \leq  - \frac{\eta^t}{2} \Ebb \big\|\nabla \Phi(\bar{x}^{t }) \big\|^2  - \bigg( \frac{\eta^t}{2} - (\eta^t)^2 L_\Phi \bigg) \Ebb \bigg\| \frac{1}{K} \sum_{k = 1}^K \mathbb{E} \big[ \nabla \Phi_k(x_k^t ; \bar{\xi}_k^t) \big| \mathcal{F}^t \big] 
 \bigg\|^2 \nonumber \\
 &   \qquad   +  \eta^t   \big(  L_h^2 + 2B_f^2 L_g^2 + 4 B_g^4 L_F^2 \big)   \frac{1}{K} \sum_{k = 1}^K   \Ebb  \|     x_k^t -   \bar{x}^t \|^2    + 4 B_g^4 L_f^2  \eta^t  ~  \Ebb \bigg\| \bar{y}^t - \frac{1}{K} \sum_{k = 1}^K g_k(x_k^t) \bigg\|^2  \nonumber \\
 & \qquad \qquad \qquad \qquad \qquad \qquad \qquad \qquad \qquad +   \frac{2  (\eta^t)^2 L_{\Phi}}{K |b_{h}|} \sigma_h^2 + \frac{2  (\eta^t)^2 L_{\Phi}  B_f^2}{K |b_{g}|} \sigma_g^2.
 \end{align*}
 for all $t \in \{ 0, 1, \ldots, T - 1\}$.
\end{lem}

\begin{proof}
Using the fact that the loss function $\Phi(x)$ is $L_{\Phi}$-Lipschitz smooth, we get

{\allowdisplaybreaks
\begin{align}
& \Ebb \big[  \Phi(\bar{x}^{t + 1}) - \Phi(\bar{x}^{t })  \big] \nonumber \\
& \leq \Ebb \Big[\langle \nabla \Phi(\bar{x}^{t }), \bar{x}^{t +1 } - \bar{x}^{t } 
 \rangle  + \frac{L_{\Phi}}{2} \| \bar{x}^{t +1 } - \bar{x}^{t } \|^2 \Big] \nonumber \\
 & \overset{(a)}{\leq} \Ebb \bigg[ -\eta^t \bigg\langle \nabla \Phi(\bar{x}^{t }), \frac{1}{K} \sum_{k = 1}^K \nabla \Phi_k(x_k^t ; \bar{\xi}_k^t)
 \bigg\rangle  + \frac{(\eta^t)^2 L_{\Phi}}{2} \bigg\| \frac{1}{K} \sum_{k = 1}^K \nabla \Phi_k(x_k^t ; \bar{\xi}_k^t) \bigg\|^2 \bigg]\nonumber \\
 & \overset{(b)}{\leq} \Ebb \bigg[ -\eta^t \bigg\langle \nabla \Phi(\bar{x}^{t }), \frac{1}{K} \sum_{k = 1}^K \mathbb{E} \big[ \nabla \Phi_k(x_k^t ; \bar{\xi}_k^t) \big| \mathcal{F}^t \big]
 \bigg\rangle  + \frac{(\eta^t)^2 L_{\Phi}}{2} \bigg\| \frac{1}{K} \sum_{k = 1}^K \nabla \Phi_k(x_k^t ; \bar{\xi}_k^t) \bigg\|^2 \bigg] \nonumber \\ \nonumber\\ \nonumber\\
 & \overset{(c)}{\leq}  - \frac{\eta^t}{2} \Ebb \big\|\nabla \Phi(\bar{x}^{t }) \big\|^2  - \bigg( \frac{\eta^t}{2} - (\eta^t)^2 L_\Phi \bigg) \Ebb \bigg\| \frac{1}{K} \sum_{k = 1}^K \mathbb{E} \big[ \nabla \Phi_k(x_k^t ; \bar{\xi}_k^t) \big| \mathcal{F}^t \big] 
 \bigg\|^2 \nonumber \\
 & \qquad \qquad \qquad \qquad \qquad   + \frac{\eta^t}{2} \underbrace{\Ebb \bigg\| \nabla \Phi(\bar{x}^{t }) - \frac{1}{K} \sum_{k = 1}^K \mathbb{E} \big[ \nabla \Phi_k(x_k^t ; \bar{\xi}_k^t) \big| \mathcal{F}^t \big]  \bigg\|^2}_{\text{Term I}} \label{Eq: Function_Descent_FL}   \\
 & \qquad \qquad \qquad \qquad \qquad \qquad 
    +  (\eta^t)^2 L_{\Phi} \underbrace{ \Ebb \bigg\| \frac{1}{K} \sum_{k = 1}^K \nabla \Phi_k(x_k^t ; \bar{\xi}_k^t) -\frac{1}{K} \sum_{k = 1}^K \mathbb{E} \big[ \nabla \Phi_k(x_k^t ; \bar{\xi}_k^t) \big| \mathcal{F}^t \big] \bigg\|^2}_{\text{Term II}}, \nonumber
\end{align}}where $(a)$ follows from the update step in Algorithm \ref{Algo: FL}; $(b)$ results from moving the conditional expectation w.r.t. the filtration $\mathcal{F}^t$ inside the inner-product; finally, $(c)$ uses the equality $2 \langle a, b \rangle  = \|a\|^2 + \| b \|^2 - \|a - b\|^2$ for $a, b \in \R^d$ and Lemma \ref{Lem: Sum_vectors} to split the last term. 

Next, we consider Terms I and II separately. First, note that from the definition of $\nabla \Phi_k(x^t_k; \bar{\xi}_k^t)$ for all $k \in [K]$, we have
\begin{align}
\label{Eq: Grad_Local_Exp}
   \mathbb{E} \big[ \nabla \Phi_k(x^t_k;\bar{\xi}_k^t) \big| \mathcal{F}^t \big] & = \mathbb{E} \bigg[ \frac{1}{|b^t_{h_k}|} \sum_{i \in b^t_{h_k}} \nabla h_k(x^t_k; \xi_{k,i}^{t}) \nonumber  +   \frac{1}{|b^t_{g_k}|} \sum_{j \in b^t_{g_k}} \nabla g_k(x^t_k;\zeta_{k,j}^{t} ) \nabla f(\bar{y}^t) \bigg| \mathcal{F}^t \bigg] \nonumber \\
   & \overset{(a)}{=} \nabla h_k (x^t_k) + \nabla g_k(x_k^t) \nabla f(\bar{y}^t)
\end{align}
where $(a)$ follows from Assumption \ref{Ass: BoundedVar}. 
Moreover, from the definition of $ \Phi(\bar{x}^t)$, we have
\begin{align}
\label{Eq: Grad_Phi}
   \nabla \Phi(\bar{x}^t)  = \frac{1}{K} \sum_{k = 1}^K \Big[ \nabla h_k(\bar{x}^t) +  \nabla g_k(\bar{x}^t) \nabla f(g(\bar{x}^t)) \Big],
\end{align}
where $g(\bar{x}^t) = \frac{1}{K} \sum_{k = 1}^K g_k(\bar{x}^t)$. Next, utilizing the expressions obtained in (\ref{Eq: Grad_Local_Exp}) and (\ref{Eq: Grad_Phi}) we bound Term I as
{
\allowdisplaybreaks
\begin{align*}
    \text{Term I} & \coloneqq \Ebb \bigg\| \nabla \Phi(\bar{x}^{t }) - \frac{1}{K} \sum_{k = 1}^K \mathbb{E} \big[ \nabla \Phi_k(x_k^t ; \bar{\xi}_k^t) \big| \mathcal{F}^t \big]  \bigg\|^2 \\
    & = \Ebb \bigg\|  \frac{1}{K} \sum_{k=1}^K \Big[ 
    \nabla h_k(\bar{x}^t) +  \nabla g_k(\bar{x}^t) \nabla f(g(\bar{x}^t)) -
    \big[ \nabla h_k (x^t_k) + \nabla g_k(x_k^t) \nabla f(\bar{y}^t) \big]
    \Big]  \bigg\|^2 \\
   &  \overset{(a)}{\leq} \frac{2}{K}  \sum_{k = 1}^K \Big[ \Ebb \| \nabla h_k(x_k^t) - \nabla h_k(\bar{x}^t) \|^2 + \|\nabla g_k(x_k^t) \nabla f(\bar{y}^t) - \nabla g_k(\bar{x}^t) \nabla f(g(\bar{x}^t))  \|^2 \Big] \\
   &  \overset{(b)}{\leq} \frac{2 L_h^2}{K}  \sum_{k = 1}^K   \Ebb  \|     x_k^t -   \bar{x}^t \|^2 + \frac{4}{K}  \sum_{k = 1}^K \Ebb \big\|\nabla g_k(x_k^t) \big[ \nabla f(\bar{y}^t) - \nabla f(g(\bar{x}^t))  \big] \big\|^2 \\
  & \qquad \qquad \qquad \qquad \qquad \qquad \qquad + \frac{4}{K}  \sum_{k = 1}^K \Ebb\| \big[  \nabla g_k(x_k^t)   - \nabla g_k(\bar{x}^t) \big] \nabla f(g(\bar{x}^t))    \|^2 \\
  &  \overset{(c)}{\leq} \frac{2 L_h^2}{K}  \sum_{k = 1}^K   \Ebb  \|     x_k^t -   \bar{x}^t \|^2 
    + \frac{4 B_g^2}{K}  \sum_{k = 1}^K \Ebb \big\|  \nabla f(\bar{y}^t) - \nabla f(g(\bar{x}^t))    \big\|^2 \\
    & \qquad \qquad \qquad \qquad \qquad \qquad \qquad \qquad \qquad \qquad + \frac{4 B_f^2 }{K}  \sum_{k = 1}^K \Ebb\|    \nabla g_k(x_k^t)   - \nabla g_k(\bar{x}^t)     \|^2 \\
  &  \overset{(d)}{\leq} \Bigg( \frac{2 L_h^2}{K} 
 + \frac{4B_f^2 L_g^2}{K}\Bigg)  \sum_{k = 1}^K   \Ebb  \|     x_k^t -   \bar{x}^t \|^2 + 4 B_g^2 L_f^2~  \underbrace{ \Ebb \big\|  \bar{y}^t - g(\bar{x}^t)   \big\|^2}_{\text{Term III}} .
\end{align*}}
Next, let us consider Term III above. 
{
\begin{align*}
    \text{Term III} & \coloneqq  \Ebb \big\|  \bar{y}^t - g(\bar{x}^t)   \big\|^2 \\
    & \overset{(a)}{\leq} 2 \Ebb \bigg\|  \bar{y}^t - \frac{1}{K} \sum_{k = 1}^K  g_k(x_k^t)   \bigg\|^2  +  2 \Ebb \bigg\| \frac{1}{K} \sum_{k = 1}^K  g_k(x_k^t) - g(\bar{x}^t)   \bigg\|^2  \\
    & \overset{(b)}{\leq} 2 \Ebb \bigg\|  \bar{y}^t - \frac{1}{K} \sum_{k = 1}^K  g_k(x_k^t)   \bigg\|^2  + \frac{2}{K} \sum_{k = 1}^K  \Ebb \big\|     g_k(x_k^t) - g_k(\bar{x}^t)    \big\|^2\\
    & \overset{(c)}{\leq} 2 \Ebb \bigg\|  \bar{y}^t - \frac{1}{K} \sum_{k = 1}^K  g_k(x_k^t)   \bigg\|^2  + \frac{2 B_g^2}{K} \sum_{k = 1}^K  \Ebb  \|      x_k^t -  \bar{x}^t     \|^2,
\end{align*}}
where $(a)$ follows from the application of Lemma \ref{Lem: Sum_vectors}; $(b)$ results from the definition of $g(x) = \frac{1}{K} \sum_{k = 1}^K g_k(x)$ and the use of Lemma \ref{Lem: Sum_vectors}; finally $(c)$ results from the Lipschitz-ness of $g_k(\cdot)$ for all $k \in [K]$.

Next, we consider Term II below
{
\allowdisplaybreaks
\begin{align*}
  \text{Term II} & \coloneqq \Ebb \bigg\| \frac{1}{K} \sum_{k = 1}^K \nabla \Phi_k(x_k^t ; \bar{\xi}_k^t) -\frac{1}{K} \sum_{k = 1}^K \mathbb{E} \big[ \nabla \Phi_k(x_k^t ; \bar{\xi}_k^t) \big| \mathcal{F}^t \big] \bigg\|^2 \\
  & \overset{(a)}{ = } \frac{1}{K^2} \sum_{k = 1}^K \Ebb \big\|  \nabla \Phi_k(x_k^t ; \bar{\xi}_k^t) - \Ebb \big[ \nabla \Phi_k(x_k^t ; \bar{\xi}_k^t) \big| \mathcal{F}^t \big] \big\|^2 \\
  & \overset{(b)}{ = } \frac{1}{K^2} \sum_{k = 1}^K \Ebb \bigg\|  \frac{1}{|b^t_{h_k}|} \sum_{i \in b^t_{h_k}} \nabla h_k(x^t_k; \xi_{k,i}^{t}) \nonumber  +   \frac{1}{|b^t_{g_k}|} \sum_{j \in b^t_{g_k}} \nabla g_k(x^t_k;\zeta_{k,j}^{t} ) \nabla f(\bar{y}^t) \\
  & \qquad \qquad \qquad \qquad \qquad \qquad \qquad \qquad \qquad \qquad \qquad - \Big[  \nabla h_k (x^t_k) + \nabla g_k(x_k^t) \nabla f(\bar{y}^t) \Big]   \bigg\|^2 \\
  & \overset{(c)}{ = } \frac{2}{K^2} \sum_{k = 1}^K \Ebb \bigg\|  \frac{1}{|b^t_{h_k}|} \sum_{i \in b^t_{h_k}} \nabla h_k(x^t_k; \xi_{k,i}^{t})  -    \nabla h_k (x^t_k)     \bigg\|^2 \\
  & \qquad \qquad \qquad \qquad  + \frac{2}{K^2} \sum_{k = 1}^K \Ebb \bigg\|      \frac{1}{|b^t_{g_k}|} \sum_{j \in b^t_{g_k}} \nabla g_k(x^t_k;\zeta_{k,j}^{t} ) \nabla f(\bar{y}^t) -   \nabla g_k(x_k^t) \nabla f(\bar{y}^t)     \bigg\|^2 \\
  & \overset{(d)}{ \leq } \frac{2  \sigma_h^2}{K |b_{h}|} + \frac{2 \sigma_g^2 B_f^2}{K |b_{g}|},
\end{align*}}where $(a)$ follows from the application of Lemma \ref{Lem: Sum_vectors}; $(b)$ follows from the definition of the stochastic gradient in (\ref{Eq: SG_FL}) and its expectation in (\ref{Eq: Grad_Local_Exp}); $(c)$ again uses Lemma \ref{Lem: Sum_vectors}; Finally, $(d)$ uses Cauchy-Schwartz inequality, Lipschitzness of $f(\bar{y}^t)$ and Assumption \ref{Ass: BoundedVar} and using $|b_{h_k}| = |b_{h}|$ and $|b_{g_k}| = |b_{g}|$ for all $k \in [K]$.

Next, substituting the upper bounds obtained for  Terms I, II, and III into (\ref{Eq: Function_Descent_FL}), we get 
 \begin{align}
    & \Ebb \big[  \Phi(\bar{x}^{t + 1}) - \Phi(\bar{x}^{t })  \big]  \leq  - \frac{\eta^t}{2} \Ebb \big\|\nabla \Phi(\bar{x}^{t }) \big\|^2  - \bigg( \frac{\eta^t}{2} - (\eta^t)^2 L_\Phi \bigg) \Ebb \bigg\| \frac{1}{K} \sum_{k = 1}^K \mathbb{E} \big[ \nabla \Phi_k(x_k^t ; \bar{\xi}_k^t) \big| \mathcal{F}^t \big] 
 \bigg\|^2 \nonumber \\
 & \qquad     +  \eta^t   \big(  L_h^2 + 2B_f^2 L_g^2 + 4 B_g^4 L_F^2 \big) \underbrace{ \frac{1}{K} \sum_{k = 1}^K   \Ebb  \|     x_k^t -   \bar{x}^t \|^2}_{\text{Term IV}}   + 4 B_g^4 L_f^2  \eta^t  ~ \underbrace{\Ebb \bigg\| \bar{y}^t - \frac{1}{K} \sum_{k = 1}^K g_k(x_k^t) \bigg\|^2}_{\text{Term V}} \nonumber \\
 & \qquad \qquad \qquad \qquad \qquad \qquad \qquad \qquad \qquad +   \frac{2  (\eta^t)^2 L_{\Phi}}{K |b_{h}|} \sigma_h^2 + \frac{2  (\eta^t)^2 L_{\Phi}  B_f^2}{K |b_{g}|} \sigma_g^2.
  \label{Eq: Function_Descent_FL1}
 \end{align}
 Therefore, we have the proof of the Lemma. 
 \end{proof}
Next, we bound Terms IV and V in (\ref{Eq: Function_Descent_FL1}) in the next Lemmas. Let us first consider Term IV.

\begin{lem}[{\bf Client Drift}]
\label{Lem: Client_Drift}
    Under Assumptions \ref{Ass: Lip}-\ref{Ass: BoundedHetero}, the iterates generated by Algorithm \ref{Algo: FL} satisfy
    \begin{align*}
    & \frac{1}{K} \sum_{k = 1}^K \Ebb \|x_k^t - \bar{x}^t \|   \leq (I - 1) \Big( 24 L_h^2 + 24 B_f^2 L_g^2 \Big) \sum_{\ell = t_s}^{t-1}   \frac{(\eta^{\ell})^2}{K} \sum_{k = 1}^K 
 \Ebb \big\| x^\ell_k  -     \bar{x}^\ell \big\|^2 \\
 & \qquad \qquad  +  (I - 1) \bigg(   \frac{4 }{|b^t_{h}|} \sigma_h^2  +  \frac{4 B_f^2 }{|b^t_{g}|} \sigma_g^2 \bigg) \sum_{\ell = t_s}^{t-1}   (\eta^{\ell})^2   + 
 (I - 1) \Big(  12 \Delta_h^2 +  12 B_f^2 \Delta_g^2 \Big) \sum_{\ell = t_s}^{t-1}   (\eta^{\ell})^2  .  
    \end{align*}
\end{lem}
\begin{proof}
Recall from the definition of $t_s$ that we have $x_k^{t_s} = \bar{x}^{t_s}$ for all $s \in \{0, 1, \ldots, S \}$. Next, we have from the update rule in Algorithm \ref{Algo: FL} that for all $t \in [t_s + 1, t_{s+1} - 1]$
\begin{align}
\label{Eq: Update_Unroll}
    x^{t}_k =  x^{t-1}_k - \eta^{t-1} \nabla \Phi_k(x^{t-1}_k ; \xio^{t-1}_k) \overset{(a)}{=} x^{t_s}_k - \sum_{\ell = t_s}^{t-1}  \eta^{\ell} \nabla \Phi_k(x^{\ell}_k ; \xio^{\ell}_k).
\end{align}
where $(a)$ results from unrolling the updates from Algorithm \ref{Algo: FL}.
Similarly, we have
\begin{align}
\label{Eq: Update_Unroll_Avg}
    \bar{x}^{t} =  \bar{x}^{t-1} - \eta^{t-1} \frac{1}{K} \sum_{k = 1}^K \nabla \Phi_k(x^{t-1}_k ; \xio^{t-1}_k) = \bar{x}^{t_s} - \frac{1}{K} \sum_{k = 1}^K \sum_{\ell = t_s}^{t-1}  \eta^{\ell} \nabla \Phi_k(x^{\ell}_k ; \xio^{\ell}_k)
\end{align}
Bounding Term IV, we have
\begin{align*}
    \text{Term IV} & \coloneqq \frac{1}{K} \sum_{k = 1}^K   \Ebb  \|     x_k^t -   \bar{x}^t \|^2 \\
    & \overset{(a)}{=} \frac{1}{K} \sum_{k = 1}^K \Ebb  \bigg\|   \sum_{\ell = t_s}^{t-1}  \eta^{\ell} \nabla \Phi_k(x^{\ell}_k ; \xio^{\ell}_k)  -  \frac{1}{K} \sum_{k = 1}^K \sum_{\ell = t_s}^{t-1}  \eta^{\ell} \nabla \Phi_k(x^{\ell}_k ; \xio^{\ell}_k)  \bigg\|^2 \\
     & \overset{(b)}{=} (I - 1) \sum_{\ell = t_s}^{t-1}    \frac{(\eta^{\ell})^2}{K} \sum_{k = 1}^K \underbrace{\Ebb  \bigg\|       \nabla \Phi_k(x^{\ell}_k ; \xio^{\ell}_k)  -  \frac{1}{K} \sum_{k = 1}^K   \nabla \Phi_k(x^{\ell}_k ; \xio^{\ell}_k)    \bigg\|^2}_{\text{Term VI}} \\
\end{align*}
 where $(a)$ follows from (\ref{Eq: Update_Unroll}) and (\ref{Eq: Update_Unroll_Avg}) and $(b)$ follows from the application of Lemma \ref{Lem: Sum_vectors}. 
 
 Next, we bound Term VI in the above expression. 
 {\allowdisplaybreaks
 \begin{align*}
  \text{Term VI} & \coloneqq    \Ebb  \bigg\|     \nabla \Phi_k(x^{\ell}_k ; \xio^{\ell}_k)  -  \frac{1}{K} \sum_{k = 1}^K  \nabla \Phi_k(x^{\ell}_k ; \xio^{\ell}_k)     \bigg\|^2 \\
  & \overset{(a)}{=}  \Ebb  \bigg\|   \frac{1}{|b^\ell_{h_k}|} \sum_{i \in b^\ell_{h_k}} \nabla h_k(x^\ell_k; \xi_{k,i}^{\ell})  +   \frac{1}{|b^\ell_{g_k}|} \sum_{j \in b^\ell_{g_k}} \nabla g_k(x^\ell_k;\zeta_{k,j}^{\ell} ) \nabla f(\bar{y}^\ell)  \\
  & \qquad  \qquad - \frac{1}{K} \sum_{k = 1}^K \bigg[ \frac{1}{|b^\ell_{h_k}|} \sum_{i \in b^\ell_{h_k}} \nabla h_k(x^\ell_k; \xi_{k,i}^{\ell})   +   \frac{1}{|b^\ell_{g_k}|} \sum_{j \in b^\ell_{g_k}} \nabla g_k(x^\ell_k;\zeta_{k,j}^{\ell} ) \nabla f(\bar{y}^\ell) \bigg]    \bigg\|^2 \\
  & \overset{(b)}{\leq} 2 \Ebb  \bigg\|   \frac{1}{|b^\ell_{h_k}|} \sum_{i \in b^\ell_{h_k}} \nabla h_k(x^\ell_k; \xi_{k,i}^{\ell})
  -
  \frac{1}{K} \sum_{k = 1}^K  \frac{1}{|b^\ell_{h_k}|} \sum_{i \in b^\ell_{h_k}} \nabla h_k(x^\ell_k; \xi_{k,i}^{\ell})   \bigg\|^2  \\
  & \qquad   + 2 \Ebb \bigg\|    \frac{1}{|b^\ell_{g_k}|} \sum_{j \in b^\ell_{g_k}} \nabla g_k(x^\ell_k;\zeta_{k,j}^{\ell} ) \nabla f(\bar{y}^\ell) - \frac{1}{K} \sum_{k = 1}^K    \frac{1}{|b^\ell_{g_k}|} \sum_{j \in b^\ell_{g_k}} \nabla g_k(x^\ell_k;\zeta_{k,j}^{t} ) \nabla f(\bar{y}^\ell)     \bigg\|^2 \\ \nonumber\\ \nonumber\\
  & \overset{(c)}{\leq} 2 \underbrace{\Ebb  \bigg\| \frac{1}{|b^\ell_{h_k}|} \sum_{i \in b^\ell_{h_k}} \nabla h_k(x^\ell_k; \xi_{k,i}^{\ell})
  -
  \frac{1}{K} \sum_{k = 1}^K  \frac{1}{|b^\ell_{h_k}|} \sum_{i \in b^\ell_{h_k}} \nabla h_k(x^\ell_k; \xi_{k,i}^{\ell}) \bigg\|^2}_{\text{Term VII}}  \\
  & \qquad \qquad  \qquad + 2 B_f^2~ \underbrace{ \Ebb \bigg\|  \frac{1}{|b^\ell_{g_k}|} \sum_{j \in b^\ell_{g_k}} \nabla g_k(x^\ell_k;\zeta_{k,j}^{\ell} )   - \frac{1}{K} \sum_{k = 1}^K    \frac{1}{|b^\ell_{g_k}|} \sum_{j \in b^\ell_{g_k}} \nabla g_k(x^\ell_k;\zeta_{k,j}^{\ell} )     \bigg\|^2}_{\text{Term VIII}},
 \end{align*}}
 where $(a)$ results from the definition of the stochastic gradient evaluated in (\ref{Eq: SG_FL}); $(b)$ uses Lemma \ref{Lem: Sum_vectors}; and $(c)$ utilizes the Cauchy-Schwartz inequality combined with the Lipschitzness of $f(\cdot)$. Next, in order to upper bound Term VI, we bound Terms VII and VIII separately. First, let us consider Term VII above
 \begin{align*}
  &   \text{Term VII} \coloneqq \Ebb  \bigg\| \frac{1}{|b^\ell_{h_k}|} \sum_{i \in b^\ell_{h_k}} \nabla h_k(x^\ell_k; \xi_{k,i}^{\ell})
  -
  \frac{1}{K} \sum_{k = 1}^K  \frac{1}{|b^\ell_{h_k}|} \sum_{i \in b^\ell_{h_k}} \nabla h_k(x^\ell_k; \xi_{k,i}^{\ell}) \bigg\|^2 \\
 & \overset{(a)}{\leq} 2 \Ebb  \bigg\| \bigg[ \frac{1}{|b^\ell_{h_k}|} \sum_{i \in b^\ell_{h_k}} \nabla h_k(x^\ell_k; \xi_{k,i}^{\ell}) - \nabla h_k(x^\ell_k) \bigg]  
  -
  \frac{1}{K}  \sum_{k = 1}^K \bigg[ \frac{1}{|b^\ell_{h_k}|}  \sum_{i \in b^\ell_{h_k}} \nabla h_k(x^\ell_k; \xi_{k,i}^{\ell}) - \nabla h_k(x^\ell_k) \bigg] \bigg\|^2 \\
 & \qquad \qquad \qquad \qquad \qquad \qquad \qquad \qquad \qquad \qquad + 2 \Ebb \bigg\| \nabla h_k(x^\ell_k)  - \frac{1}{K} \sum_{k = 1}^K \nabla h_k(x^\ell_k)  \bigg\|^2 \\
 & \overset{(b)}{\leq} 2 \Ebb  \bigg\|  \frac{1}{|b^\ell_{h_k}|} \sum_{i \in b^\ell_{h_k}} \nabla h_k(x^\ell_k; \xi_{k,i}^{\ell}) - \nabla h_k(x^\ell_k) 
   \bigg\|^2  + 2 \Ebb \bigg\| \nabla h_k(x^\ell_k)  - \frac{1}{K} \sum_{k = 1}^K \nabla h_k(x^\ell_k)  \bigg\|^2 \\
 &  \overset{(c)}{\leq}  \frac{2 \sigma_h^2}{|b^\ell_{h_k}|}  + 2 \underbrace{\Ebb \bigg\| \nabla h_k(x^\ell_k)  - \frac{1}{K} \sum_{k = 1}^K \nabla h_k(x^\ell_k)  \bigg\|^2}_{\text{Term IX}},
 \end{align*}
where $(a)$ utilizes Lemma \ref{Lem: Sum_vectors}; $(b)$ results from the application of Lemma \ref{Lem: Emp_Var}; and $(c)$ results from Assumption \ref{Ass: BoundedVar}. 

Next, we bound Term IX below
{\allowdisplaybreaks
\begin{align*}
  \text{Term IX} & \coloneqq  \Ebb \bigg\| \nabla h_k(x^\ell_k)  - \frac{1}{K} \sum_{k = 1}^K \nabla h_k(x^\ell_k)  \bigg\|^2 \\
  & \overset{(a)}{\leq}  3 \Ebb \big\| \nabla h_k(x^\ell_k)  -     \nabla h_k(\bar{x}^\ell)  \big\|^2 + 3  \Ebb \bigg\| \frac{1}{K} \sum_{k = 1}^K  \Big[ \nabla h_k(\bar{x}^\ell)  -    \nabla h_k(x^\ell_k) \Big] \bigg\|^2 \\
  & \qquad \qquad \qquad \qquad \qquad \qquad \qquad \qquad \qquad + 3 \Ebb \bigg\| \nabla h_k(\bar{x}^\ell)  -   \frac{1}{K} \sum_{k = 1}^K \nabla h_k(\bar{x}^\ell)  \bigg\|^2 \\
  & \overset{(b)}{\leq}  3 L_h^2 \Ebb \big\| x^\ell_k  -     \bar{x}^\ell \big\|^2 +\frac{3 L_h^2}{K} \sum_{k = 1}^K \Ebb \big\| x^\ell_k  -     \bar{x}^\ell \big\|^2  + 3 \Ebb \big\| \nabla h_k(\bar{x}^\ell)  -     \nabla h(\bar{x}^\ell)  \big\|^2 \\
  & \overset{(c)}{\leq}  3 L_h^2 \Ebb \big\| x^\ell_k  -     \bar{x}^\ell \big\|^2 +\frac{3 L_h^2}{K} \sum_{k = 1}^K \Ebb \big\| x^\ell_k  -     \bar{x}^\ell \big\|^2  + 3 \Delta_h^2,
\end{align*}}
where $(a)$ results from the application of Lemma \ref{Lem: Sum_vectors}; $(b)$ utilizes Lipschitz smoothness of $h(\cdot)$ and the definition of $h(x) = \frac{1}{K} \sum_{k = 1}^K h_k(x)$; finally, $(c)$ results from the bounded heterogeneity assumption Assumption \ref{Ass: BoundedHetero}. Substituting the bound on Term IX in the bound of Term VII, we get
\begin{align*}
     \text{Term VII}\leq \frac{2 \sigma_h^2}{|b^t_{h_k}|}  +  6 L_h^2 \Ebb \big\| x^\ell_k  -     \bar{x}^\ell \big\|^2 +\frac{6 L_h^2}{K} \sum_{k = 1}^K \Ebb \big\| x^\ell_k  -     \bar{x}^\ell \big\|^2  + 6 \Delta_h^2.
\end{align*}

Similarly, we bound Term VIII as
\begin{align*}
   & \text{Term VIII}  \coloneqq \Ebb \bigg\|  \frac{1}{|b^\ell_{g_k}|} \sum_{j \in b^\ell_{g_k}} \nabla g_k(x^\ell_k;\zeta_{k,j}^{\ell} )   - \frac{1}{K} \sum_{k = 1}^K    \frac{1}{|b^\ell_{g_k}|} \sum_{j \in b^\ell_{g_k}} \nabla g_k(x^\ell_k;\zeta_{k,j}^{\ell} )     \bigg\|^2 \\
    & \overset{(a)}{\leq} 2 \Ebb  \bigg\| \bigg[ \frac{1}{|b^\ell_{g_k}|} \sum_{i \in b^\ell_{g_k}} \nabla g_k(x^\ell_k; \zeta_{k,i}^{\ell}) - \nabla g_k(x^\ell_k) \bigg]
  -
  \frac{1}{K}  \sum_{k = 1}^K \bigg[ \frac{1}{|b^\ell_{g_k}|}  \sum_{i \in b^\ell_{g_k}} \nabla g_k(x^\ell_k; \zeta_{k,i}^{\ell}) - \nabla g_k(x^\ell_k) \bigg] \bigg\|^2 \\
 & \qquad \qquad \qquad \qquad \qquad \qquad \qquad \qquad \qquad \qquad + 2 \Ebb \bigg\| \nabla g_k(x^\ell_k)  - \frac{1}{K} \sum_{k = 1}^K \nabla g_k(x^\ell_k)  \bigg\|^2 \\
 & \overset{(b)}{\leq} 2 \Ebb  \bigg\|  \frac{1}{|b^\ell_{g_k}|} \sum_{i \in b^\ell_{g_k}} \nabla g_k(x^\ell_k; \zeta_{k,i}^{\ell}) - \nabla g_k(x^\ell_k) 
   \bigg\|^2  + 2 \Ebb \bigg\| \nabla g_k(x^\ell_k)  - \frac{1}{K} \sum_{k = 1}^K \nabla g_k(x^\ell_k)  \bigg\|^2 \\
 &  \overset{(c)}{\leq}  \frac{2 \sigma_g^2}{|b^\ell_{g_k}|}  + 2 \underbrace{\Ebb \bigg\| \nabla g_k(x^\ell_k)  - \frac{1}{K} \sum_{k = 1}^K \nabla g_k(x^\ell_k)  \bigg\|^2}_{\text{Term X}},
\end{align*}
where $(a)$ utilizes Lemma \ref{Lem: Sum_vectors}; $(b)$ results from the application of Lemma \ref{Lem: Emp_Var}; and $(c)$ results from Assumption \ref{Ass: BoundedVar}. Next, we bound Term X below
\begin{align*}
  \text{Term X} & \coloneqq  \Ebb \bigg\| \nabla g_k(x^\ell_k)  - \frac{1}{K} \sum_{k = 1}^K \nabla g_k(x^\ell_k)  \bigg\|^2 \\
  & \overset{(a)}{\leq}  3 \Ebb \big\| \nabla g_k(x^\ell_k)  -     \nabla g_k(\bar{x}^\ell)  \big\|^2 + 3  \Ebb \bigg\| \frac{1}{K} \sum_{k = 1}^K  \Big[ \nabla g_k(\bar{x}^\ell)  -    \nabla g_k(x^\ell_k) \Big] \bigg\|^2 \\
  & \qquad \qquad \qquad \qquad \qquad \qquad \qquad \qquad \qquad + 3 \Ebb \bigg\| \nabla g_k(\bar{x}^\ell)  -   \frac{1}{K} \sum_{k = 1}^K \nabla g_k(\bar{x}^\ell)  \bigg\|^2 \\
  & \overset{(b)}{\leq}  3 L_g^2 \Ebb \big\| x^\ell_k  -     \bar{x}^\ell \big\|^2 +\frac{3 L_g^2}{K} \sum_{k = 1}^K \Ebb \big\| x^\ell_k  -     \bar{x}^\ell \big\|^2  + 3 \Ebb \big\| \nabla g_k(\bar{x}^\ell)  -     \nabla g(\bar{x}^\ell)  \big\|^2 \\
  & \overset{(c)}{\leq}  3 L_g^2 \Ebb \big\| x^\ell_k  -     \bar{x}^\ell \big\|^2 +\frac{3 L_g^2}{K} \sum_{k = 1}^K \Ebb \big\| x^\ell_k  -     \bar{x}^\ell \big\|^2  + 3 \Delta_g^2,
\end{align*}
where $(a)$ results from the application of Lemma \ref{Lem: Sum_vectors}; $(b)$ utilizes Lipschitz smoothness of $g(\cdot)$ and the definition of $g(x) = \frac{1}{K} \sum_{k = 1}^K g_k(x)$; finally, $(c)$ results from the bounded heterogeneity assumption Assumption \ref{Ass: BoundedHetero}. Substituting the bound on Term X in the bound of Term VIII, we get
\begin{align*}
   \text{Term VIII} \leq \frac{2 \sigma_g^2}{|b^\ell_{g_k}|}  +  6 L_g^2 \Ebb \big\| x^\ell_k  -     \bar{x}^\ell \big\|^2 +\frac{6 L_g^2}{K} \sum_{k = 1}^K \Ebb \big\| x^\ell_k  -     \bar{x}^\ell \big\|^2  + 6 \Delta_g^2. 
\end{align*}
Next, we substitute the upper bounds on Terms VII and VIII in the expression of Term VI, we get
\begin{align*}
    \text{Term VI} & \leq    \frac{4 }{|b^\ell_{h_k}|} \sigma_h^2  +  12 L_h^2 \Ebb \big\| x^\ell_k  -     \bar{x}^\ell \big\|^2 + \frac{12 L_h^2}{K} \sum_{k = 1}^K \Ebb \big\| x^\ell_k  -     \bar{x}^\ell \big\|^2  + 12 \Delta_h^2 \\
    & \qquad   +   \frac{4 B_f^2 }{|b^\ell_{g_k}|} \sigma_g^2 +  12 B_f^2 L_g^2 \Ebb \big\| x^\ell_k  -     \bar{x}^\ell \big\|^2 +\frac{12 B_f^2 L_g^2}{K} \sum_{k = 1}^K \Ebb \big\| x^\ell_k  -     \bar{x}^\ell \big\|^2  + 12 B_f^2 \Delta_g^2 \\
    &  = \Big( 12 L_h^2 + 12 B_f^2 L_g^2 \Big) \Ebb \big\| x^\ell_k  -     \bar{x}^\ell \big\|^2 + \bigg(\frac{12 L_h^2  +  12 B_f^2 L_g^2}{K} \bigg) \sum_{k = 1}^K \Ebb \big\| x^\ell_k  -     \bar{x}^\ell \big\|^2 \\
    &  \qquad\qquad\qquad\qquad \qquad\qquad\qquad \qquad + \frac{4 }{|b^\ell_{h_k}|} \sigma_h^2  +  \frac{4 B_f^2 }{|b^\ell_{g_k}|} \sigma_g^2 + 12 \Delta_h^2 +  12 B_f^2 \Delta_g^2.
\end{align*}
Therefore, we finally have the bound on Term IV as
\begin{align*}
   \text{Term IV}  & \leq (I - 1) \Big( 24 L_h^2 + 24 B_f^2 L_g^2 \Big) \sum_{\ell = t_s}^{t-1}   \frac{(\eta^{\ell})^2}{K} \sum_{k = 1}^K 
 \Ebb \big\| x^\ell_k  -     \bar{x}^\ell \big\|^2 \\
 &     +  (I - 1) \bigg(   \frac{4 }{|b^t_{h}|} \sigma_h^2  +  \frac{4 B_f^2 }{|b^t_{g}|} \sigma_g^2 \bigg) \sum_{\ell = t_s}^{t-1}   (\eta^{\ell})^2     + 
 (I - 1) \Big(  12 \Delta_h^2 +  12 B_f^2 \Delta_g^2 \Big) \sum_{\ell = t_s}^{t-1}   (\eta^{\ell})^2  .
\end{align*}
where we have chosen $|b_{h_k}^\ell| = |b_{h}^t|$ and $|b_{g_k}^\ell| = |b_{g}^t|$ for all $k \in [K]$ and $\ell \in \{ 0, \ldots, T -1 \}$.

Therefore, we have proof of the Lemma. 
\end{proof}

Next, we bound Term V from (\ref{Eq: Function_Descent_FL1}), we have

\begin{lem}[{\bf Descent in the estimate of $g(x)$}]
\label{Lem: Descent_Y}
Under Assumptions \ref{Ass: Lip}-\ref{Ass: BoundedHetero},
the iterates generated by Algorithm \ref{Algo: FL} satisfy:
\begin{align*}
  & \Ebb \bigg\| \bar{y}^t - \frac{1}{K} \sum_{k = 1}^K g_k(x_k^t) \bigg\|^2 \\
  &  {\leq} (1 - \beta^t)^2  \mathbb{E} \bigg\|   \bar{y}^{t - 1}  - \frac{1}{K} \sum_{k = 1}^K g_k(x_k^{t-1}) \bigg\|^2 +  \frac{8 (\eta^t)^2 (1 - \beta^t)^2 B_g^2}{|b_g| K} \Ebb \bigg\|    \frac{1}{K} \sum_{k = 1}^K \Ebb [ \Phi_k(x_k^{t}; \bar{\xi}_k^t)| \mathcal{F}^t] \bigg\|^2 \\
  & + \frac{ (\eta^t)^2 (1 - \beta^t)^2 B_g^2 (96 L_h^2 + 96 B_f^2 L_g^2) }{|b_g| K^2}   \sum_{k = 1}^K \Ebb \big\| x^t_k  -     \bar{x}^t \big\|^2 +  \frac{4 (\eta^t)^2 (1 - \beta^t)^2 B_g^2}{|b_h|  K} \sigma_h^2  \\
  &  +  \frac{2 (\beta^t)^2 + 4 (\eta^t)^2 (1 - \beta^t)^2 B_g^2 B_f^2}{|b_g| K} \sigma_g^2  +  \frac{48 (\eta^t)^2 (1 - \beta^t)^2 B_g^2}{|b_g| K} \Delta_h^2 +   \frac{48 (\eta^t)^2 (1 - \beta^t)^2 B_f^2 B_g^2}{|b_g| K} \Delta_g^2.
\end{align*}
where we have chosen $|b_h^t| = |b_h|$ and $|b_{g_k}^t| = |b_g|$ for all $k \in [K]$ and $t \in [T]$.
\end{lem}
\begin{proof}
From the definition of Term V, we have
{\allowdisplaybreaks
\begin{align*}
  &  \text{Term V}  \coloneqq \Ebb \bigg\| \bar{y}^{t+1} - \frac{1}{K} \sum_{k = 1}^K g_k(x_k^{t+1}) \bigg\|^2 \\
    & \overset{(a)}{=}  \mathbb{E} \bigg\|\frac{1}{K} \sum_{k = 1}^K \Big[  y_k^{t+1} - g_k(x_k^{t+1})\Big] \bigg\|^2 \\
   & \overset{(b)}{=} \mathbb{E} \bigg\| \frac{1}{K} \sum_{k = 1}^K \Big[ (1 - \beta^{t+1}) \Big( y_k^{t} + \frac{1}{|b^{t+1}_{g_k}|} \sum_{i\in b^{t+1}_{g_k}} g_k(x_k^{t+1}; \zeta_{k,i}^{t+1}) - \frac{1}{|b^{t+1}_{g_k}|} \sum_{i\in b^{t+1}_{g_k}} g_k(x_k^{t}; \zeta_{k,i}^{t+1})  \Big)  \\
   & \qquad \qquad \qquad \qquad \qquad \qquad \qquad \qquad \qquad      + \frac{\beta^{t+1}}{|b^{t+1}_{g_k}|} \sum_{i\in b^{t+1}_{g_k}} g_k(x^{t+1}_k, \zeta_{k,i}^{t+1})  - g_k(x_k^{t+1}) \Big]  \bigg\|^2 \\
   & \overset{(c)}{=} (1 - \beta^{t+1})^2 ~\mathbb{E} \bigg\| \frac{1}{K} \sum_{k = 1}^K \Big[  y_k^{t }  - g_k(x_k^{t }) \Big] \bigg\|^2 \\
   &  + 
  \Ebb \bigg\| \frac{1}{K} \sum_{k = 1}^K \bigg[  (1 - \beta^{t+1}) \Big[ (g_k(x_k^{t}) - g_k(x_k^{t+1})) -  \frac{1}{|b^{t+1}_{g_k}|} \sum_{i \in b^{t+1}_{g_k}} \big(g_k(x_k^{t}; \zeta_{k,i}^{t+1}) 
 - g_k(x_k^{t+1} ; \zeta_{k,i}^{t+1})   \big)\Big] 
 \\
 &   \qquad    \qquad \qquad   \qquad \qquad \qquad \qquad  \qquad  + \beta^{t+1} \bigg( \frac{1}{|b^{t+1}_{g_k}|} \sum_{i \in b^{t+1}_{g_k}} g_k(x_k^{t+1}; \zeta_{k,i}^{t+1}) - g_k(x_k^{t+1}) \bigg)  \bigg] \bigg\|^2 \\ \nonumber\\\nonumber\\
   & \overset{(d)}{\leq} (1 - \beta^{t+1})^2~ \mathbb{E} \bigg\|   \bar{y}^{t}  - \frac{1}{K} \sum_{k = 1}^K g_k(x_k^{t}) \bigg\|^2 + \frac{2 (\beta^{t+1})^2 }{|b_g| K } \sigma_g^2  \\
   & \quad    + \frac{2 (1 - \beta^{t+1})^2 }{K^2 } \sum_{k = 1}^K \frac{1}{|b_{g}|^2} \sum_{i \in b^{t+1}_{g_k}}
   \Ebb \big\|          (g_k(x_k^{t}) - g(x_k^{t+1})) -     \big(g_k(x_k^{t}; \zeta_{k,i}^{t+1}) 
 - g_k(x_k^{t+1} ; \zeta_{k,i}^{t+1})   \big)   \big\|^2 \\
   & \overset{(e)}{\leq} (1 - \beta^{t+1})^2 ~\mathbb{E} \bigg\|   \bar{y}^{t}  - \frac{1}{K} \sum_{k = 1}^K g_k(x_k^{t}) \bigg\|^2 + \frac{2 (\beta^{t+1})^2 }{|b_g| K } \sigma_g^2  \\
   & \qquad \qquad \qquad \qquad   +   \frac{2 (1 - \beta^{t+1})^2}{K^2 } \sum_{k = 1}^K \frac{1}{|b_{g}|^2} \sum_{i \in b^{t+1}_{g_k}}
   \Ebb \big\|               g_k(x_k^{t}; \zeta_{k,i}^{t+1}) 
 - g_k(x_k^{t+1} ; \zeta_{k,i}^{t+1})      \big\|^2  \\
   & \overset{(f)}{\leq} (1 - \beta^{t+1})^2~ \mathbb{E} \bigg\|   \bar{y}^{t}  - \frac{1}{K} \sum_{k = 1}^K g_k(x_k^{t}) \bigg\|^2 + \frac{2 (\beta^{t+1})^2 }{|b_g| K } \sigma_g^2   + 
  \frac{2 (1 - \beta^{t+1})^2 B_g^2}{|b_g| K^2} \sum_{k = 1}^K   \Ebb \big\|         x_k^{t+1} -  x_k^{t}   \big\|^2 \\
   & \overset{(g)}{\leq} (1 - \beta^{t+1})^2  \mathbb{E} \bigg\|   \bar{y}^{t}  - \frac{1}{K} \sum_{k = 1}^K g_k(x_k^{t}) \bigg\|^2 + \frac{2 (\beta^{t+1})^2 }{|b_g| K } \sigma_g^2  \\
   & \qquad \qquad \qquad \qquad \qquad \qquad \qquad  + 
  \frac{2 (\eta^{t})^2 (1 - \beta^{t+1})^2 B_g^2}{|b_g| K^2} \sum_{k = 1}^K \underbrace{\Ebb \big\|  \nabla \Phi_k(x_k^{t}; \bar{\xi}_k^{t}) \big\|^2}_{\text{Term XI}},
\end{align*}}
where $(a)$ follows from the definition of $\bar{y}^{t+1}$; $(b)$ uses the update rule (\ref{Eq: Update_Y_FL}) for $y_k^{t+1}$; $(c)$ results from adding and subtracting $(1 - \beta^{t + 1}) g_k(x_k^{t})$ and utilizing the fact that the second term in the expression has zero-mean which follows from Assumption \ref{Ass: BoundedVar}; $(d)$ uses Young's inequality, Assumption \ref{Ass: BoundedVar} and by choosing $|b_h^t| = |b_h|$ and $|b_{g_k}^t| = |b_g|$ for all $k \in [K]$ and $t \in [T]$; $(e)$ results from the fact that for a random variable $X$, we have $\Ebb \| X - \Ebb[X] \|^2 \leq \Ebb \| X \|^2$; $(f)$ uses the mean-squared Lipschitzness of $g_k(\cdot)$ in Assumption \ref{Ass: Lip}; finally $(g)$ results from the update rule of Algorithm \ref{Algo: FL}. 

Next, we bound Term XI below
\begin{align*}
   \text{Term XI} & \coloneqq \Ebb \big\|  \nabla \Phi_k(x_k^{t}; \bar{\xi}_k^{t}) \big\|^2 \\
   & \overset{(a)}{\leq} 2 \Ebb \big\|  \nabla \Phi_k(x_k^{t}; \bar{\xi}_k^t) - \Ebb [\nabla \Phi_k(x_k^{t}; \bar{\xi}_k^t)| \mathcal{F}^t] \big\|^2  + 2 \Ebb \big\|    \Ebb [\nabla \Phi_k(x_k^{t}; \bar{\xi}_k^t)| \mathcal{F}^t] \big\|^2 \\
  & \overset{(b)}{\leq} \frac{2  \sigma_h^2}{  |b_{h}|} + \frac{2 \sigma_g^2 B_f^2}{  |b_{g}|} + 4 \underbrace{ \Ebb \bigg\|    \Ebb [\nabla \Phi_k(x_k^{t}; \bar{\xi}_k^t)| \mathcal{F}^t] 
 -  \frac{1}{K} \sum_{k = 1}^K \Ebb [\nabla \Phi_k(x_k^{t}; \bar{\xi}_k^t)| \mathcal{F}^t]  \bigg\|^2}_{\text{Term XII}} \\
 &  \qquad \qquad \qquad \qquad  \qquad \qquad \qquad \qquad + 4 \Ebb \bigg\|    \frac{1}{K} \sum_{k = 1}^K \Ebb [\nabla \Phi_k(x_k^{t}; \bar{\xi}_k^t)| \mathcal{F}^t] \bigg\|^2 ,
\end{align*}
where $(a)$ results from the application of Young's inequality and $(b)$ results from Assumptions \ref{Ass: Lip} and \ref{Ass: BoundedVar} along with the application of Young's inequality.

Next, we bound Term XII in the above expression. 
{\allowdisplaybreaks
\begin{align*}
 & \text{Term XII}  \coloneqq   \Ebb \bigg\|    \Ebb [\nabla \Phi_k(x_k^{t}; \bar{\xi}_k^t)| \mathcal{F}^t] 
 -  \frac{1}{K} \sum_{k = 1}^K \Ebb [ \nabla \Phi_k(x_k^{t}; \bar{\xi}_k^t)| \mathcal{F}^t]  \bigg\|^2 \\
 & \overset{(a)}{=}   \Ebb \bigg\| \nabla h_k(x_k^t) +  \nabla g_k(x_k^t) \nabla f(\bar{y}^t)    
 - \bigg[ \frac{1}{K} \sum_{k = 1}^K \big( \nabla h_k(x_k^t) +  \nabla g_k(x_k^t) \nabla f(\bar{y}^t) \big) \bigg]  \bigg\|^2 \\
 & \overset{(b)}{\leq} 2 \Ebb \bigg\| \nabla h_k(x_k^t)  - \frac{1}{K} \sum_{k = 1}^K   \nabla h_k(x_k^t)    \bigg\|^2 + 2 \Ebb \bigg\|    \nabla g_k(x_k^t) \nabla f(\bar{y}^t)    
 -   \frac{1}{K} \sum_{k = 1}^K   \nabla g_k(x_k^t) \nabla f(\bar{y}^t)   \bigg]  \bigg\|^2 \\
 & \overset{(c)}{\leq} 2 \underbrace{\Ebb \bigg\| \nabla h_k(x_k^t)  - \frac{1}{K} \sum_{k = 1}^K   \nabla h_k(x_k^t)    \bigg\|^2}_{\text{Term IX}} + 2 B_f^2 \underbrace{\Ebb \bigg\|    \nabla g_k(x_k^t)      
 -   \frac{1}{K} \sum_{k = 1}^K   \nabla g_k(x_k^t)     \bigg] \bigg\|^2}_{\text{Term X}}\\
 & \overset{(d)}{\leq} (6 L_h^2 + 6 B_f^2 L_g^2 ) \Ebb \big\| x^t_k  -     \bar{x}^t \big\|^2 +\frac{6 L_h^2 + 6 B_f^2 L_g^2}{K} \sum_{k = 1}^K \Ebb \big\| x^t_k  -     \bar{x}^t \big\|^2  + 6 \Delta_h^2 + 6 B_f^2 \Delta_g^2  
\end{align*}}
where $(a)$ above uses the definition of $\nabla \Phi_k (x_k^t; \bar{\xi}_k^t)$ in (\ref{Eq: SG_FL}) and Assumption \ref{Ass: BoundedVar}; $(b)$ results from the application of Young's inequality; $(c)$ utilized Assumtion \ref{Ass: Lip}; finally, $(d)$ results from the application of Assumptions \ref{Ass: Lip} and \ref{Ass: BoundedHetero}.

Replacing in the upper bound for Term XI, we get
\begin{align*}
     \text{Term XI} & \leq   4 \Ebb \bigg\|    \frac{1}{K} \sum_{k = 1}^K \Ebb [ \Phi_k(x_k^{t}; \bar{\xi}_k^t)| \mathcal{F}^t] \bigg\|^2 +  (24 L_h^2 + 24 B_f^2 L_g^2 ) \Ebb \big\| x^t_k  -     \bar{x}^t \big\|^2 
    \\
   & \qquad       + \frac{24 L_h^2 + 24 B_f^2 L_g^2}{K} \sum_{k = 1}^K \Ebb \big\| x^t_k  -     \bar{x}^t \big\|^2  + \frac{2  \sigma_h^2}{  |b_{h}|} + \frac{2 \sigma_g^2 B_f^2}{  |b_{g}|} + 24 \Delta_h^2 + 24 B_f^2 \Delta_g^2 .
\end{align*}
Substituting the bound on Term XI in the bound of Term V, we get
\begin{align*}
   & \Ebb \bigg\| \bar{y}^{t+1} - \frac{1}{K} \sum_{k = 1}^K g_k(x_k^{t+1}) \bigg\|^2 \\ & {\leq} (1 - \beta^{t+1})^2~  \mathbb{E} \bigg\|   \bar{y}^{t}  - \frac{1}{K} \sum_{k = 1}^K g_k(x_k^{t}) \bigg\|^2 +  \frac{8 (\eta^t)^2 (1 - \beta^{t+1})^2 B_g^2}{|b_g| K} \Ebb \bigg\|    \frac{1}{K} \sum_{k = 1}^K \Ebb [ \Phi_k(x_k^{t}; \bar{\xi}_k^t)| \mathcal{F}^t] \bigg\|^2 \\
  & + \frac{ (\eta^t)^2 (1 - \beta^{t+1})^2 B_g^2 (96 L_h^2 + 96 B_f^2 L_g^2) }{|b_g| K^2}   \sum_{k = 1}^K \Ebb \big\| x^t_k  -     \bar{x}^t \big\|^2 +  \frac{4 (\eta^t)^2 (1 - \beta^{t+1})^2 B_g^2}{|b_h|  K} \sigma_h^2  \\
  &  +  \frac{2 (\beta^{t+1})^2 + 4 (\eta^t)^2 (1 - \beta^{t+1})^2 B_g^2 B_f^2}{|b_g| K} \sigma_g^2  +  \frac{48 (\eta^t)^2 (1 - \beta^{t+1})^2 B_g^2}{|b_g| K} \Delta_h^2 +   \frac{48 (\eta^t)^2 (1 - \beta^{t+1})^2 B_f^2 B_g^2}{|b_g| K} \Delta_g^2.
\end{align*}
Therefore, we have proof of Lemma. 
\end{proof}
Next, we show descent in the potential function specially designed to show convergence of Algorithm \ref{Algo: FL}. For this purpose, we define the potential function as
\begin{align}
    \label{Eq: Potential_Fn}
    V^t = \Ebb [\Phi(\bar{x}^t)] + \Ebb \bigg\| \bar{y}^t - \frac{1}{K} \sum_{k = 1}^K g_k(x_k^t) \bigg\|^2  .
\end{align}
Next, we derive the descent in the potential function. 
\begin{lem}[{\bf Descent in Potential Function}]
\label{lem: Descent_Potential}
    Under Assumptions \ref{Ass: Lip}-\ref{Ass: BoundedHetero} with the choice of momentum-parameter $\beta^{t+1} =   c_\beta \eta^t$ with $c_\beta = 4 B_g^4 L_f^2$ where step-size $\eta^t$ is chosen such that 
    \begin{align*}
        \eta^t \leq \bigg\{\frac{|b_g| K}{2 (L_\Phi |b_g| K + 8 B_g^2)}, \frac{|b_g| K \big( L_h^2 + 2B_f^2 L_g^2 + 4B_g^4 L_f^2 \big)}{B_g^2 \big(96 L_h^2 + 96 B_f^2 L_g^2 \big)} \bigg\}
    \end{align*}
    the iterates generated by Algorithm \ref{Algo: FL} satisfy
    \begin{align*}
        V^{t + 1} - V^t &  {\leq} - \frac{\eta^t}{2} \Ebb \big\|\nabla \Phi(\bar{x}^{t }) \big\|^2   
    +  \eta^t   \big(  2 L_h^2 + 4 B_f^2 L_g^2 + 8 B_g^4 L_F^2 \big)  \frac{1}{K} \sum_{k = 1}^K   \Ebb  \|     x_k^t -   \bar{x}^t \|^2 \\
 &\qquad   +  \frac{2  (\eta^t)^2 L_{\Phi}}{K |b_{h}|} \sigma_h^2 +   \frac{4 (\eta^t)^2  B_g^2 }{|b_h|  K} \sigma_h^2   +  \frac{2  (\eta^t)^2 L_{\Phi}  B_f^2}{ |b_{g}| K} \sigma_g^2   +  \frac{(\eta^t)^2 (2 c_\beta^2 + 4  B_g^2 B_f^2 )}{|b_g| K} \sigma_g^2\\
 & \qquad \qquad \qquad \qquad \qquad \qquad \qquad \qquad \qquad +  \frac{48 (\eta^t)^2 B_g^2}{|b_g| K} \Delta_h^2 +   \frac{48 (\eta^t)^2   B_f^2 B_g^2}{|b_g| K} \Delta_g^2.
    \end{align*}
\end{lem}
\begin{proof}
    From the definition of $V^t$ in (\ref{Eq: Potential_Fn}) and using Lemmas \ref{Lem: Descent_Phi} and \ref{Lem: Descent_Y}, we get
    \begin{align*}
        V^{t + 1} - V^t & = \Ebb [\Phi(\bar{x}^{t+1}) - \Phi(\bar{x}^t)] + \Ebb \bigg\| \bar{y}^{t+1} - \frac{1}{K} \sum_{k = 1}^K g_k(x_k^{t+1}) \bigg\|^2 - \Ebb \bigg\| \bar{y}^{t} - \frac{1}{K} \sum_{k = 1}^K g_k(x_k^{t}) \bigg\|^2 \\
        & \leq - \frac{\eta^t}{2} \Ebb \big\|\nabla \Phi(\bar{x}^{t }) \big\|^2  - \bigg( \frac{\eta^t}{2} - (\eta^t)^2 L_\Phi -\frac{8(\eta^t)^2 B_g^2}{|b_g|K} \bigg) \Ebb \bigg\| \frac{1}{K} \sum_{k = 1}^K \mathbb{E} \big[ \nabla \Phi_k(x_k^t ; \bar{\xi}_k^t) \big| \mathcal{F}^t \big] 
 \bigg\|^2 \nonumber \\
 &       + \bigg( \eta^t   \big(  L_h^2 + 2B_f^2 L_g^2 + 4 B_g^4 L_F^2 \big) + \frac{ (\eta^t)^2  B_g^2 (96 L_h^2 + 96 B_f^2 L_g^2) }{|b_g| K} \bigg)  \frac{1}{K} \sum_{k = 1}^K   \Ebb  \|     x_k^t -   \bar{x}^t \|^2 \\
 &    + \big( 4 B_g^4 L_f^2  \eta^t - \beta^{t+1} \big)  ~  \Ebb \bigg\| \bar{y}^t - \frac{1}{K} \sum_{k = 1}^K g_k(x_k^t) \bigg\|^2    +   \frac{2  (\eta^t)^2 L_{\Phi}}{K |b_{h}|} \sigma_h^2 +   \frac{4 (\eta^t)^2 B_g^2  }{|b_h|  K} \sigma_h^2 \\
  &          +  \frac{2  (\eta^t)^2 L_{\Phi}  B_f^2 }{ |b_{g}| K} \sigma_g^2   +  \frac{2 (\beta^{t+1})^2 + 4 (\eta^t)^2  B_g^2 B_f^2 }{|b_g| K} \sigma_g^2  +  \frac{48 (\eta^t)^2 B_g^2}{|b_g| K} \Delta_h^2 +   \frac{48 (\eta^t)^2   B_f^2 B_g^2}{|b_g| K} \Delta_g^2 \\
  & \overset{(a)}{\leq} - \frac{\eta^t}{2} \Ebb \big\|\nabla \Phi(\bar{x}^{t }) \big\|^2   
    +  \eta^t   \big(  2 L_h^2 + 4 B_f^2 L_g^2 + 8 B_g^4 L_F^2 \big)  \frac{1}{K} \sum_{k = 1}^K   \Ebb  \|     x_k^t -   \bar{x}^t \|^2 \\
 &\qquad   +  \frac{2  (\eta^t)^2 L_{\Phi}}{K |b_{h}|} \sigma_h^2 +   \frac{4 (\eta^t)^2  B_g^2 }{|b_h|  K} \sigma_h^2   +  \frac{2  (\eta^t)^2 L_{\Phi}  B_f^2}{ |b_{g}| K} \sigma_g^2   +  \frac{(\eta^t)^2 (2 c_\beta^2 + 4  B_g^2 B_f^2 )}{|b_g| K} \sigma_g^2\\
 & \qquad \qquad \qquad \qquad \qquad \qquad \qquad \qquad \qquad +  \frac{48 (\eta^t)^2 B_g^2}{|b_g| K} \Delta_h^2 +   \frac{48 (\eta^t)^2   B_f^2 B_g^2}{|b_g| K} \Delta_g^2.
    \end{align*}
    where $(a)$ results from the choice of $\beta^{t}$ and $\eta_t$ given in the statement of the Lemma. 

    Therefore, we have the proof. 
\end{proof}

\begin{theorem}[{\bf Potential Function}] 
\label{Thm: Potential_Telescope}
Under Assumptions \ref{Ass: Lip}-\ref{Ass: BoundedHetero} and the choice of step-size $\eta^t = \eta$ such that we have
\begin{align*}
    \eta \leq \frac{1}{3 I  {\big( 24 L_h^2 + 24 B_f^2 L_g^2 \big)^{1/2}}}
\end{align*}
the iterates generated by Algorithm \ref{Algo: FL} satisfy
   \begin{align*} 
    & V^{T} - V^0  \leq - \frac{\eta}{2} \sum_{t = 0}^{T - 1} \Ebb \big\|\nabla \Phi(\bar{x}^{t }) \big\|^2  
    +  \eta^3 (I - 1)^2 \frac{ \big(  10 L_h^2 + 20 B_f^2 L_g^2 + 40 B_g^4 L_F^2 \big)}{|b_h|} \sigma_h^2 ~T \\
    & +  \frac{2 \eta^2 L_{\Phi}}{K |b_{h}|} \sigma_h^2 ~T +   \frac{4 \eta^2  B_g^2 }{|b_h|  K} \sigma_h^2~ T    +  \eta^3 (I - 1)^2 \frac{ \big(  10 B_f^2 L_h^2 + 20 B_f^4 L_g^2 + 40 B_f^2 B_g^4 L_F^2 \big)}{|b_g|} \sigma_g^2 ~T  \\
    & +  \frac{2 \eta^2 L_{\Phi}  B_f^2}{ |b_{g}| K} \sigma_g^2 ~T  +  \frac{\eta^2 (2 c_\beta^2 + 4   B_f^2 B_g^2)}{|b_g| K} \sigma_g^2 ~T    + \eta^3 (I - 1)^2   \big(  30   L_h^2 + 60 B_f^2 L_g^2 + 120  B_g^4 L_F^2 \big)  \Delta_h^2 ~T  \\
  &  +  \frac{48 \eta^2 B_g^2 }{|b_g| K} \Delta_h^2~T    + \eta^3 (I - 1)^2   \big(  30 B_f^2 L_h^2 + 60 B_f^4 L_g^2 + 120 B_f^2 B_g^4 L_F^2 \big)  \Delta_g^2 ~T  +   \frac{48 \eta^2   B_f^2 B_g^2}{|b_g| K} \Delta_g^2 ~T.   
   \end{align*}
   \end{theorem}
   \begin{proof}
   Telescoping the sum of  Lemma \ref{lem: Descent_Potential} for $t = \{0, 1, \ldots, T - 1 \}$, we get
   \begin{align}
       V^{T} - V^0 & \leq - \frac{\eta}{2} \sum_{t = 0}^{T - 1} \Ebb \big\|\nabla \Phi(\bar{x}^{t }) \big\|^2  
    +  \eta   \big(  2 L_h^2 + 4 B_f^2 L_g^2 + 8 B_g^4 L_F^2 \big) \underbrace{\sum_{t = 0}^{T - 1}  \frac{1}{K} \sum_{k = 1}^K   \Ebb  \|     x_k^t -   \bar{x}^t \|^2}_{\text{Term XIII}}  \nonumber\\
 &\qquad    +  \frac{2 \eta^2 L_{\Phi}}{K |b_{h}|} \sigma_h^2 ~T +   \frac{4 \eta^2  B_g^2 }{|b_h|  K} \sigma_h^2~ T   +  \frac{2 \eta^2 L_{\Phi}  B_f^2}{ |b_{g}| K} \sigma_g^2 ~T  +  \frac{\eta^2 (2 c_\beta^2 + 4   B_g^2 B_f^2)}{|b_g| K} \sigma_g^2 ~T \nonumber\\
 & \qquad \qquad \qquad \qquad \qquad \qquad \qquad   +  \frac{48 \eta^2 B_g^2}{|b_g| K} \Delta_h^2~T +   \frac{48 \eta^2   B_f^2 B_g^2}{|b_g| K} \Delta_g^2 ~T.  
 \label{Eq: Telescope_Potential}
   \end{align}
   We bound Term XIII in (\ref{Eq: Telescope_Potential}) using Lemma (\ref{Lem: Client_Drift}). Note that we have from Lemma (\ref{Lem: Client_Drift}) 
    \begin{align*}
      &  \frac{1}{K} \sum_{k = 1}^K   \Ebb  \|     x_k^t -   \bar{x}^t \|^2
    \leq (I - 1) \Big( 24 L_h^2 + 24 B_f^2 L_g^2 \Big) \sum_{\ell = t_s}^{t-1}   \frac{(\eta^{\ell})^2}{K} \sum_{k = 1}^K 
 \Ebb \big\| x^\ell_k  -     \bar{x}^\ell \big\|^2 \\
 & \qquad  +  (I - 1) \bigg(   \frac{4 }{|b^t_{h}|} \sigma_h^2  +  \frac{4 B_f^2 }{|b^t_{g}|} \sigma_g^2 \bigg) \sum_{\ell = t_s}^{t-1}   (\eta^{\ell})^2   + 
 (I - 1) \Big(  12 \Delta_h^2 +  12 B_f^2 \Delta_g^2 \Big) \sum_{\ell = t_s}^{t-1}   (\eta^{\ell})^2 
   \end{align*}
   Summing the above from $t = t_s$ to $t_{s+1} - 1$, we get
    \begin{align*}
      & \sum_{t = t_s}^{t_{s+1} - 1} \frac{1}{K} \sum_{k = 1}^K   \Ebb  \|     x_k^t -   \bar{x}^t \|^2
    \overset{(a)}{\leq} \eta^2 (I - 1) \Big( 24 L_h^2 + 24 B_f^2 L_g^2 \Big) \sum_{t = t_s}^{t_{s+1} - 1} \sum_{\ell = t_s}^{t-1}   \frac{1}{K} \sum_{k = 1}^K 
 \Ebb \big\| x^\ell_k  -     \bar{x}^\ell \big\|^2 \\
 & \qquad\qquad\qquad + \eta^2  (I - 1)^2 I \bigg(   \frac{4 }{|b^t_{h}|} \sigma_h^2  +  \frac{4 B_f^2 }{|b^t_{g}|} \sigma_g^2 \bigg)     + 
 \eta^2 (I - 1)^2 I \Big(  12 \Delta_h^2 +  12 B_f^2 \Delta_g^2 \Big)  \\
 & \qquad\qquad\qquad \overset{(b)}{\leq} \eta^2 (I - 1) \Big( 24 L_h^2 + 24 B_f^2 L_g^2 \Big) \sum_{t = t_s}^{t_{s+1} - 1} \sum_{\ell = t_s}^{t_{s  + 1}-1}   \frac{1}{K} \sum_{k = 1}^K 
 \Ebb \big\| x^\ell_k  -     \bar{x}^\ell \big\|^2 \\
 & \qquad\qquad\qquad\qquad + \eta^2  (I - 1)^2 I \bigg(   \frac{4 }{|b^t_{h}|} \sigma_h^2  +  \frac{4 B_f^2 }{|b^t_{g}|} \sigma_g^2 \bigg)     + 
 \eta^2 (I - 1)^2 I \Big(  12 \Delta_h^2 +  12 B_f^2 \Delta_g^2 \Big) \\
 & \qquad\qquad\qquad \overset{(c)}{\leq} \eta^2 (I - 1) I \Big( 24 L_h^2 + 24 B_f^2 L_g^2 \Big) \sum_{t = t_s}^{t_{s+1} - 1}     \frac{1}{K} \sum_{k = 1}^K 
 \Ebb \big\| x^t_k  -     \bar{x}^t \big\|^2 \\
 & \qquad\qquad\qquad \qquad + \eta^2  (I - 1)^2 I \bigg(   \frac{4 }{|b^t_{h}|} \sigma_h^2  +  \frac{4 B_f^2 }{|b^t_{g}|} \sigma_g^2 \bigg)     + 
 \eta^2 (I - 1)^2 I \Big(  12 \Delta_h^2 +  12 B_f^2 \Delta_g^2 \Big)
   \end{align*}
where in $(a)$ we have used the fact that $\eta^t = \eta$ for all $t \in [T]$ and $(t - 1) - t_s \leq I - 1$ for $t \in [t_s, t_{s+1}-1]$; $(b)$ results from the fact that $t \leq t_{s + 1}$; finally, $(c)$ again uses the fact that $(t - 1) - t_s \leq I - 1$ for $t \in [t_s, t_{s+1}-1]$.

   Summing the above from $s = \{0,1, \ldots, S \}$ and using the fact that $S \times I = T - 1$, we get
   \begin{align*}
     &   \sum_{t = 0}^{T - 1} \frac{1}{K} \sum_{k = 1}^K   \Ebb  \|     x_k^t -   \bar{x}^t \|^2  \leq \eta^2  I^2 \Big( 24 L_h^2 + 24 B_f^2 L_g^2 \Big) \sum_{t = 0}^{T - 1}     \frac{1}{K} \sum_{k = 1}^K 
 \Ebb \big\| x^t_k  -     \bar{x}^t \big\|^2 \\
 & \qquad \quad  + \eta^2  (I - 1)^2  \bigg(   \frac{4 }{|b^t_{h}|} \sigma_h^2  +  \frac{4 B_f^2 }{|b^t_{g}|} \sigma_g^2 \bigg)~T     + 
 \eta^2 (I - 1)^2  \Big(  12 \Delta_h^2 +  12 B_f^2 \Delta_g^2 \Big) ~T.
   \end{align*}
   Rearranging the terms, we get
   \begin{align*}
     \Big(1 - \eta^2 I^2 \big( 24 L_h^2 + 24 B_f^2 L_g^2 \big) \Big)   \sum_{t = 0}^{T - 1} \frac{1}{K} \sum_{k = 1}^K   \Ebb  \|     x_k^t -   \bar{x}^t \|^2  & \leq \eta^2  (I - 1)^2  \bigg(   \frac{4 }{|b^t_{h}|} \sigma_h^2  +  \frac{4 B_f^2 }{|b^t_{g}|} \sigma_g^2 \bigg)~T    \\
  &      \qquad   + 
 \eta^2 (I - 1)^2  \Big(  12 \Delta_h^2 +  12 B_f^2 \Delta_g^2 \Big) ~T.
   \end{align*}
   Finally, choosing $\eta \leq \frac{1}{3 I  {\big( 24 L_h^2 + 24 B_f^2 L_g^2 \big)^{1/2}}}$, such that we have $ 1 - \eta^2 I^2 \big( 24 L_h^2 + 24 B_f^2 L_g^2 \big)   \geq 8/9$, utilizing this we get
 \begin{align*}
    \text{Term XIII} & \coloneqq   \sum_{t = 0}^{T - 1} \frac{1}{K} \sum_{k = 1}^K   \Ebb  \|     x_k^t -   \bar{x}^t \|^2 \\
    & \leq \eta^2  (I - 1)^2  \bigg(   \frac{5 }{ |b^t_{h}|} \sigma_h^2  +  \frac{5 B_f^2 }{  |b^t_{g}|} \sigma_g^2 \bigg)~T    
    + 
 \eta^2 (I - 1)^2  \Big(  15 \Delta_h^2 +  15 B_f^2 \Delta_g^2 \Big) ~T.
   \end{align*}
   Finally, substituting the bound on Term XIII in (\ref{Eq: Telescope_Potential}), we get
       \begin{align}
       & V^{T} - V^0   \leq - \frac{\eta}{2} \sum_{t = 0}^{T - 1} \Ebb \big\|\nabla \Phi(\bar{x}^{t }) \big\|^2  
    +  \eta^3 (I - 1)^2 \frac{ \big(  10 L_h^2 + 20 B_f^2 L_g^2 + 40 B_g^4 L_F^2 \big)}{|b_h|} \sigma_h^2 ~T \nonumber\\
    & +  \frac{2 \eta^2 L_{\Phi}}{K |b_{h}|} \sigma_h^2 ~T +   \frac{4 \eta^2  B_g^2 }{|b_h|  K} \sigma_h^2~ T    +  \eta^3 (I - 1)^2 \frac{ \big(  10 B_f^2 L_h^2 + 20 B_f^4 L_g^2 + 40 B_f^2 B_g^4 L_F^2 \big)}{|b_g|} \sigma_g^2 ~T  \nonumber \\
   & +  \frac{2 \eta^2 L_{\Phi}  B_f^2}{ |b_{g}| K} \sigma_g^2 ~T  +  \frac{\eta^2 (2 c_\beta^2 + 4   B_f^2 B_g^2)}{|b_g| K} \sigma_g^2 ~T     + \eta^3 (I - 1)^2   \big(  30   L_h^2 + 60 B_f^2 L_g^2 + 120  B_g^4 L_F^2 \big)  \Delta_h^2 ~T \nonumber \\
   & +  \frac{48 \eta^2 B_g^2 }{|b_g| K} \Delta_h^2~T   + \eta^3 (I - 1)^2   \big(  30 B_f^2 L_h^2 + 60 B_f^4 L_g^2 + 120 B_f^2 B_g^4 L_F^2 \big)  \Delta_g^2 ~T  +   \frac{48 \eta^2   B_f^2 B_g^2}{|b_g| K} \Delta_g^2 ~T. \nonumber 
   \end{align}
   Therefore, we have the proof. 
   \end{proof}
Now, we are finally ready to prove Theorem \ref{Thm: FL}.
   \begin{proof}
       Assuming $|b_h| = |b_g| = |b|$ and defining $\bar{L}_{f,g} \coloneqq 10 L_h^2 + B_f^2 L_g^2 + 40 B_g^4 L_f^2 $. Rearranging the terms in the expression of Theorem \ref{Thm: Potential_Telescope} and multiplying both sides by $2/\eta T$ we get
       \begin{align}
   &  \frac{1}{T} \sum_{t = 0}^{T - 1} \Ebb \big\|\nabla \Phi(\bar{x}^{t }) \big\|^2    \leq \frac{2 \big[\Phi(\bar{x}^0) - \Phi(x^\ast) + \big\| \bar{y}^{0} - g(\bar{x}^0) \big\|^2 \big]}{\eta T} +  \eta^2 (I - 1)^2 \bigg[ \frac{2 \bar{L}_{f,g}}{|b|} \sigma_h^2 + \frac{2 B_f^2 \bar{L}_{f,g}}{|b|} \sigma_g^2 \bigg] \nonumber\\
&    \qquad    +  \eta^2 (I - 1)^2 \Big[  {6\bar{L}_{f,g}}  \Delta_h^2 + {6B_f^2 \bar{L}_{f,g}}   \Delta_g^2 \Big]     +  \eta \bigg[ \frac{4 L_{\Phi} + 8 B_g^2}{|b|K}  \sigma_h^2 + \frac{4 L_\Phi B_f^2 + 4c_\beta^2+ 8 B_f^2 B_g^2} {|b|K} \sigma_g^2 \bigg] \nonumber \\
 & \qquad \qquad \qquad \qquad \qquad \qquad \qquad \qquad \qquad \qquad \qquad +  \eta \bigg[ \frac{96 B_g^2}{|b|K} \Delta_h^2 + \frac{96 B_f^2 B_g^2}{|b|K} \Delta_g^2 \bigg],\nonumber  
   \end{align}
   where the first term on the right follows from the fact that $\Phi(\bar{x}^T) \geq \Phi(x^\ast)$ and $\| \bar{y}^T - {1}/{K} \sum_{k = 1}^K g_k(x_k^T) \|^2 \geq 0$.

   Next, choosing $\eta = \sqrt{\frac{|b| K}{T}}$ then for $T \geq \big( 216 L_h^2 + 216 B_f^2 L_g^2 \big) I^2 |b| K$ such that $\eta \leq \frac{1}{3 I  {\big( 24 L_h^2 + 24 B_f^2 L_g^2 \big)^{1/2}}}$ in Theorem \ref{Thm: Potential_Telescope} is satisfied, we get the following 
    \begin{align}
   &  \frac{1}{T} \sum_{t = 0}^{T - 1} \Ebb \big\|\nabla \Phi(\bar{x}^{t }) \big\|^2    \leq \frac{2 \big[\Phi(\bar{x}^0) - \Phi(x^\ast) + \big\| \bar{y}^{0} - g(\bar{x}^0) \big\|^2 \big]}{\sqrt{|b| K T}} +  \frac{K (I - 1)^2}{T}  \Big[ {2 \bar{L}_{f,g}} \sigma_h^2 + {2 B_f^2 \bar{L}_{f,g}}  \sigma_g^2 \Big] \nonumber\\
&       +  \frac{|b|K (I - 1)^2}{T} \Big[  {6\bar{L}_{f,g}}  \Delta_h^2 + {6B_f^2 \bar{L}_{f,g}}   \Delta_g^2 \Big]     +  \frac{1}{\sqrt{|b| K T}} \bigg[ \big( {4 L_{\Phi} + 8B_g^2} \big)   \sigma_h^2 +  \big({4 L_\Phi B_f^2 + 4c_\beta^2+ 8 B_f^2 B_g^2} \big)   \sigma_g^2 \bigg] \nonumber \\
 & \qquad   \qquad \qquad \qquad \qquad \qquad \qquad \qquad \qquad +  \frac{1}{\sqrt{|b| K T}} \bigg[  {96 B_g^2} ~ \Delta_h^2 +  {96 B_f^2 B_g^2}  ~\Delta_g^2 \bigg],  \nonumber
   \end{align}
Explicitly choosing $I = T^{1/4} / (|b| K)^{3/4}$, we get
 \begin{align*}
         \Ebb \big\|\nabla \Phi(\bar{x}^{a(T)}) \big\|^2   & \leq \frac{2 \big[\Phi(\bar{x}^0) - \Phi(x^\ast) + \big\| \bar{y}^{0} - g(\bar{x}^0) \big\|^2 \big]}{\sqrt{|b| K T}} +   \frac{C_{\sigma_h}}{\sqrt{|b| K T}} \sigma_h^2  +  \frac{C_{\sigma_g}}{\sqrt{|b| K T}} \sigma_g^2 \\
         & \qquad \qquad \qquad \qquad \qquad \qquad \qquad \qquad  + \frac{C_{\Delta_h}}{\sqrt{|b| K T}} \Delta_h^2  +  \frac{C_{\Delta_g}}{\sqrt{|b| K T}} \Delta_g^2 .  \end{align*}
         where the constants $C_{\sigma_h}$, $C_{\sigma_g}$, $C_{\Delta_f}$, and $C_{\Delta_g}$ are defined as:
         \begin{align*}
         C_{\sigma_h} & =  2 \bar{L}_{f,g} + 4 L_{\Phi} + 8 B_g^2 \\
C_{\sigma_g} & = 2 B_f^2 \bar{L}_{f,g} + 4 L_\Phi B_f^2 + 4 c_\beta^2 + 8 B_f^2 B_g^2  \\
C_{\Delta_f} &  = 6 \bar{L}_{f,g} + 96 B_g^2
\\
C_{\Delta_g} & = 6 B_f^2 \bar{L}_{f,g} + 96 B_f^2 B_g^2 .
         \end{align*}
  The constant $c_\beta$ is defined in the statement of Lemma \ref{lem: Descent_Potential}.

  Hence, Theorem \ref{Thm: FL} is proved. 
   \end{proof}

\end{document}